\documentclass[default,iicol]{sn-jnl}

\usepackage{graphicx}%
\usepackage{multirow}%
\usepackage{amsmath,amssymb,amsfonts}%
\usepackage{amsthm}%
\usepackage{mathrsfs}%
\usepackage[title]{appendix}%
\usepackage{xcolor}%
\usepackage{textcomp}%
\usepackage{manyfoot}%
\usepackage{booktabs}%
\usepackage{algorithm}%
\usepackage{algorithmicx}%
\usepackage{algpseudocode}%
\usepackage{listings}%
\usepackage{setspace}
\usepackage[colorinlistoftodos]{todonotes}
\usepackage[font=small,labelfont=bf]{caption}
\usepackage{enumerate}
\usepackage[autostyle]{csquotes}
\usepackage[table]{colortbl}
\usepackage{ulem}
\usepackage{hhline}
\usepackage{placeins}
\usepackage{float}

\newcommand{\R}{\mathbb{R}}
\newcommand{\N}{\mathbb{N}}

\DeclareMathOperator{\sign}{sign}

\newtheorem{lemma}{Lemma}

\begin{document}

\title[Riesz networks]{Riesz networks: scale invariant neural networks\linebreak 
in a single forward pass}


\author*[1,2]{\fnm{Tin} \sur{ Barisin}}\email{barisin@rptu.de}

\author[2]{\fnm{Katja} \sur{ Schladitz}}\email{katja.schladitz@itwm.fraunhofer.de}

\author[1]{\fnm{Claudia} \sur{ Redenbach}}\email{claudia.redenbach@rptu.de}

\affil[1]{\orgdiv{Mathematics Department}, \orgname{RPTU Kaiserslautern-Landau}, \orgaddress{\street{Gottlieb-Daimler-Straße 47}, \city{Kaiserslautern}, \postcode{67663}, 
\country{Germany}}}

\affil[2]{\orgdiv{Department Image Processing}, \orgname{Fraunhofer-Institut für Techno- und Wirtschaftsmathematik (ITWM)}, \orgaddress{\street{Fraunhofer-Platz 1}, \city{Kaiserslautern}, \postcode{67663}, 
\country{Germany}}}


\abstract{
Scale invariance of an algorithm refers to its ability to treat objects equally independently of their size. For neural networks, scale invariance is typically achieved by data augmentation. However, when presented with a scale far outside the range covered by the training set, neural networks may fail to generalize. 
    
    Here, we introduce the Riesz network, a novel scale invariant neural network. Instead of standard 2d or 3d convolutions for combining spatial information, the Riesz network is based on the Riesz transform which is a scale equivariant operation. As a consequence, this network naturally generalizes to unseen or even arbitrary scales in a single forward pass.
    
    As an application example, we consider detecting and segmenting cracks in tomographic images of concrete. In this context, 'scale' refers to the crack thickness which may vary strongly even within the same sample. 
    To prove its scale invariance, the Riesz network is trained on one fixed crack width. We then validate its performance in segmenting simulated and real tomographic images featuring a wide range of crack widths. An additional experiment is carried out on the MNIST Large Scale data set.
    
}




\maketitle

\section{Introduction}\label{sec1}

In image data, similar objects may occur at highly varying scales. Examples are cars or pedestrians at different distances from the camera, cracks in concrete of varying thickness or imaged at different resolution, or blood vessels in biomedical applications (see Fig. \ref{fig:scale_var}). It is natural to assume that the same object or structure at different scales should be treated equally i.e. should have equal or at least similar features. This property is called scale or dilation invariance and has been investigated in detail in classical image processing~\cite{lindeberg98, lowe99, lindeberg15}. 

Neural networks have proven to segment and classify robustly and well in many computer vision tasks. Nowadays, the most popular and successful neural networks are Convolutional Neural Networks (CNNs). 
It would be desirable that neural networks share typical properties of human vision such as translation, rotation, or scale invariance.
While this is true for translation invariance, CNNs are not scale or rotation invariant by default. This is due to the excessive use of convolutions which are local operators. Moreover, training sets often contain a very limited number of scales. 
To overcome this problem, CNNs are often trained with rescaled images through data augmentation. However, when a CNN is given input whose scale is outside the range covered by the training set, it will not be able to generalize \cite{kanazawa14, jansson22}. 
To overcome this problem, a CNN trained at a fixed scale can be applied to several rescaled versions of the input image and the results can be combined. This, however, requires multiple runs of the network.

One application example, where the just described challenges naturally occur is the task of segmenting cracks in 2d or 3d gray scale images of concrete. Crack segmentation in 2d has been a vividly researched topic in civil engineering, see \cite{barisin22} for an overview. 
Cracks are naturally multiscale structures (Fig. \ref{fig:scale_var}, left) and hence require multiscale treatment. Nevertheless, adaption to scale (crack thickness\footnote{Crack scale, thickness, and width refer to the same characteristic and will be interchangeably used throughout the paper.}) has not been treated explicitly so far.  

Recently, crack segmentation in 3d  images obtained by computed tomography (CT) has become a subject of interest~\cite{paetsch11,barisin22}. Here, the effect of varying scales is even more pronounced \cite{jung22}: crack thicknesses can vary from a single pixel to more than $100$ pixels. Hence, the aim is to design and evaluate crack segmentation methods that work equally well on all possible crack widths without complicated adjustment by the user. 

\begin{figure*}
    \centering
    \includegraphics[width = 0.732\textwidth]{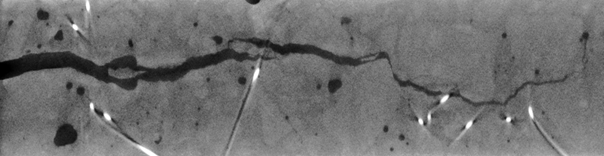}
    \includegraphics[width = 0.25325\textwidth]{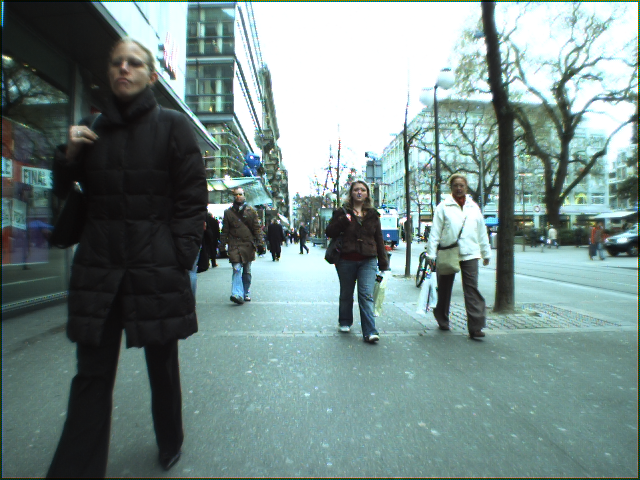}
    \caption{Examples of similar objects appearing on different scales: section of a CT image of concrete showing a crack of locally varying thickness (left) and pedestrians at difference distances from the camera (right, taken from~\cite{ess08}).}
    \label{fig:scale_var}
\end{figure*}

In this work, we focus on 2d multiscale crack segmentation in images of concrete samples. We design the Riesz network which replaces the popular 2d convolutions by first and second order Riesz transforms to allow for a scale invariant spatial operation. The resulting neural network is provably scale invariant in only one forward pass. It is sufficient to train the Riesz network on one scale or crack thickness, only. The network then generalizes automatically without any adjustments or rescaling to completely unseen scales. 
We validate the network performance using images with simulated cracks of constant and varying widths generated as described in \cite{barisin22, jung22a}.  
Our network is compared with competing methods for multiscale segmentation and finally applied to real multiscale cracks observed in 2d slices of tomographic images.

There is just one publicly available dataset which allows for testing scale equivariance -- MNIST Large Scale \cite{jansson22}. Additional experiments with the Riesz network on this dataset are reported in Appendix \ref{secA1}.

\subsection{Related work}

\subsubsection{The Riesz transform} 
The Riesz transform is a generalization of the Hilbert transform to higher dimensional spaces, see e.g. \cite{bernstein13}. First practical applications of the Riesz transform arise in signal processing through the definition of the monogenic signal \cite{felsberg01} which enables a decomposition 
of higher dimensional signals into local phase and local amplitude. 
First, a bandpass filter is applied to the signal to separate the band of frequencies. Using the Riesz transform, the local phase and amplitude can be calculated for a selected range of frequencies. For more details we refer to \cite{felsberg01, felsbergPhD02}. 

As images are 2d or 3d signals, applications of the Riesz transform naturally extend to the fields of image processing and computer vision through the Poisson scale space \cite{felsberg02, felsberg04} which is an alternative to the well-known Gaussian scale space.
K{\"o}the~\cite{koethe05} compared the Riesz transform with the structure tensor from a signal processing perspective. Unser and van de Ville \cite{unser10} related higher order Riesz transforms and derivatives. Furthermore, they give a reason for preferring the Riesz transform over the standard derivative operator: The Riesz transform does not amplify high frequencies. 

Higher order Riesz transforms were also used for analysis of local image structures using ideas from differential geometry \cite{wietzke08, dobrovolskij19}. Benefits of using the first and second order Riesz transforms as low level features have also been shown in measuring similarity \cite{zhang10}, analyzing and classification of textures \cite{depeursinge12, bernstein13}, and orientation estimation \cite{langley10, reinhardt20}. The Riesz transform can be used to create steerable wavelet frames, so-called \textit{Riesz-Laplace wavelets}~\cite{unser09, unser10}, which are the first ones utilizing the scale equivariance property of the Riesz transform and have inspired the design of \textit{quasi monogenic shearlets} \cite{hauser14}.

Interestingly, in early works on the Riesz transform in signal processing or image processing \cite{felsberg01, felsbergPhD02, felsberg04}, scale equivariance has not been noticed as a feature of the Riesz transform and hence remained sidelined. Benefits of the scale equivariance have been shown later in \cite{unser10, dobrovolskij19}. 

Recently, the Riesz transform found its way into the field of deep learning: Riesz transform features are used as supplementary features in classical CNNs to improve robustness \cite{joyseeree19}.
In our work, we will use the Riesz transforms for extracting low-level features from images and use them as basis functions which replace trainable convolutional filters in CNNs or Gaussian derivatives in~\cite{lindeberg21}. 




\subsubsection{Scale invariant deep learning methods}  Deep learning methods which have mechanisms to handle variations in scale effectively can be split in two groups based on their scale generalization ability.


\paragraph{Scale invariant deep learning methods for a limited range of scales} 
The first group can handle multiscale data but is limited to the scales represented either by the training set or by the neural network architecture. 
The simplest approach to learn multiscale features is to apply the convolutions to several rescaled versions of the images or feature maps in every layer and to combine the results by maximum pooling \cite{kanazawa14} or by keeping the scale with maximal activation \cite{marcos18} before passing it to the next layer. 
In \cite{cai16,lin17}, several approaches based on downscaling images or feature maps with the goal of designing robust multiscale object detectors are summarized. However, scaling is included in the network architecture such that scales have to be selected a priori. Therefore, this approach only yields local scale invariance, i.e. an adaption to the scale observed in a given input image is not possible after training. 

Another intuitive approach is to rescale trainable filters, i.e. convolutions, by interpolation \cite{xu14}. 
In \cite{cai16}, a new multiscale strategy was derived which uses convolution blocks of varying sizes sequenced in several downscaling layers creating a pyramidal shape. 
The pyramidal structure is utilized for learning scale dependent features and making predictions in every 
downsampling layer. Layers can be trained according to object size. That is, only the part of the network relevant for the object size is optimized. This guarantees robustness to a large range of object scales.
Similarly, in \cite{lin17}, a network consisting of a downsampling pyramid followed by an upsampling pyramid is proposed. 
Here, connections between pyramid levels are devised for combining low and high resolution features and predictions are also made independently on every pyramid level. 
However, in both cases, scale generalization properties of the networks are restricted by their architecture, i.e. by the depth of the network (number of levels in the image pyramid), the size of convolutions as spatial operators as well as the size of the input image.

Spatial transformer networks \cite{jaderberg15} focus on invariance to affine transformations including scale. 
This is achieved by using a so-called \textit{localisation network} which learns transformation parameters. Finally, using these transformation parameters, a new sampling grid can be created and feature maps are resampled to it. These parts form a trainable module which is able to handle and correct the effect of the affine transformations.
However, spatial transformer networks do not necessarily achieve invariant recognition \cite{finnveden20}. Also, it is not clear how this type of network would generalize to scales not represented in the training set.

In \cite{jacobsen16}, so-called \textit{structured receptive fields} are introduced. Linear combinations ($1\times1$ convolutions) of basis functions (in this case Gaussian derivatives up to $4$th order) are used to learn complex features and to replace convolutions 
(e.g. of size $3\times3$ or $5\times5$). As a consequence, the number of parameters is reduced, while the expressiveness of the neural network is preserved. 
This type of network works better than classical CNNs in the case where little training data is available. However, the standard deviation parameters of the Gaussian kernels are manually selected and kept fixed. Hence, the scale generalization ability remains limited. 

Making use of the semi-group property of scale spaces,
\textit{scale equivariant neural networks} motivate the use of \textit{dilated convolutions} \cite{yu15} to define scale equivariant convolutions on the Gaussian scale space \cite{worral19} or morphological scale spaces \cite{sangalli21}. Unfortunately, these
neural networks are unable to generalize to scales outside those determined by their architecture and are only suitable for 
downscale factors which are powers of 2, i.e. $\{2,4,8,16,\cdots\}$.
Furthermore, scale equivariant steerable networks \cite{sosnovik20} show how to design scale invariant networks on the scale-translation group without using standard or dilated convolutions. Following  an idea from \cite{jacobsen16}, convolutions are replaced by linear combinations of basis filters (Hermite polynomials with Gaussian envelope). While this allows for non-integer scales, scales are still limited to powers of a positive scaling factor $a$. Scale space is again discretized and sampled. Hence, a generalization to arbitrary scales is not guaranteed.

\paragraph{Scale invariant deep learning methods for arbitrary scales}

The second group of methods can generalize to arbitrary scales, i.e. any scales that are in range bounded from below by image resolution and from above by image size, but not necessarily contained in the training set. 
Our Riesz network also belongs to this second group of methods.

An intuitive approach is to train standard CNNs on a fixed range of scales and enhance their scale generalization ability by the following three step procedure based on image pyramids: downsample by a factor $a>1$, forward pass of the CNN, upsample the result by $\frac{1}{a}$ to 
the original image size \cite{jung22,jansson22}. Finally, forward passes of the CNN from several downsampling factors $\{a_1,~\cdots~,a_n~>~0\quad| \quad n\in\N\}$ are aggregated by using the maximum or average operator across the scale dimension. This approach indeed guarantees generalization to unseen scales as scales can be adapted to the input image and share the weights of the network \cite{jansson22}. However, it requires multiple forward passes of the CNN and the downsampling factors have to be selected by the user.

Inspired by Scattering Networks \cite{bruna13, sifre13}, normalized differential operators based on first and second order Gaussian derivatives stacked in layers or a cascade of a network can be used to extract more complex features \cite{lindeberg20}. Subsequently, these features serve as an input for a classifier such as a support vector machine. Varying the standard deviation parameter $\sigma$ of the Gaussian kernel, generalization to new scales can be achieved. 
However, this type of network is useful for creating \textit{handcrafted} complex scale invariant features, only, and hence is not trainable. 

Its expansion to trainable networks by creating so-called Gaussian derivative networks \cite{lindeberg21} is one of the main inspirations for our work. 
For combining spatial information, $\gamma$-normalized Gaussian derivatives are used as scale equivariant operators ($\gamma = 1$). Similarly as in \cite{jacobsen16}, linear combinations of normalized derivatives are used to learn more complex features in the spirit of deep learning.
However, the Gaussian derivative network is based on ideas from scale space theory and requires the specification of a wide enough range of scales that cover all the scales present in the training and testing set.  Hence, the scale dimension needs to be discretized and sampled densely. 
These networks have a separate channel for every scale and the network weights are shared between channels. Scale invariance is achieved by maximum pooling over the multiple scale channels.



\section{The Riesz transform}

Let $L_2(\R^d) = \{ f :\R^d\to\R \text{  } | \text{  } \int_{\R^d}{|f(x)|^2dx < \infty}\}$ be the set of square integrable functions. 

Formally, for a $d$-dimensional signal $f\in L_2(\R^d)$ (i.e. an image or a feature map), the Riesz transform of first order $\mathcal{R}=(\mathcal{R}_1,\cdots,\mathcal{R}_d)$ is defined in the spatial domain as $\mathcal{R}_{j}: L_2(\R^d) \to L_2(\R^d)$
 \begin{equation}
     \mathcal{R}_{j}(f)(x) = C_d \lim_{\epsilon \to 0}{\int_{\R^d \setminus B_{\varepsilon}}{\frac{y_jf(x-y)}{|y|^{d+1}}dy}}, 
 \end{equation}
where $C_d = \Gamma((d+1)/2)/\pi^{(d+1)/2}$ is a normalizing constant and $B_{\varepsilon}$ is ball of radius $\varepsilon$ centered at the origin. Alternatively, the Riesz transform can be defined in the frequency domain via the Fourier transform $\mathcal{F}$ 
 	\begin{equation}
  \label{eq:FT:RT}
     \mathcal{F}(\mathcal{R}_j(f))(u) = -i\frac{u_j}{|u|}\mathcal{F}(f)(u) =  \frac{1}{|u|}\mathcal{F}(\partial_j f)(u),
    \end{equation}
for $j \in \{1,\cdots,d\}$. Higher order Riesz transforms are defined by applying a sequence of first order Riesz transforms. 
That is, for $k_1,k_2,...,k_d\in \N \cup \{0\}$ 
we set
\begin{equation} \mathcal{R}^{(k_1,k_2,...,k_d)} (f)(x) := \mathcal{R}^{k_1}_1( \mathcal{R}^{k_2}_2( \cdots (\mathcal{R}^{k_d}_d(f(x)))),
\end{equation}
where $\mathcal{R}_j^{k_j}$ refers to applying the Riesz transform $\mathcal{R}_j$   $k_j$ times in a sequence.

\begin{figure*}[h]
    \centering
    \includegraphics[width = 0.25\textwidth]{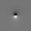}
    \includegraphics[width = 0.25\textwidth]{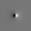}
    \\
    \includegraphics[width = 0.25\textwidth]{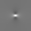}
    \includegraphics[width = 0.25\textwidth]{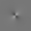}
    \includegraphics[width = 0.25\textwidth]{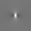}
    \caption{Visualizations of Riesz transform kernels of first and second order. First row (from left to right): $\mathcal{R}_1$ and $\mathcal{R}_2$. Second row (from left to right): $\mathcal{R}^{(2,0)}$, $\mathcal{R}^{(1,1)}$, and $\mathcal{R}^{(0,2)}$. 
    }
    \label{fig:riesz-kernel}
\end{figure*}

The Riesz transform kernels of first and second order resemble those of the corresponding 
derivatives of smoothing filters such as Gaussian or Poisson filters (Fig.~\ref{fig:riesz-kernel}). This 
can be explained by the following relations
  \begin{align}
     &\mathcal{R}(f) =(-1)(-\triangle)^{-1/2} \nabla f \\
     &\mathcal{R}^{(k_1,k_2,...,k_d)} (f)(x)= (-1)^N(- \triangle)^{-N/2} \frac{\partial^N f(x)}{\partial^{k_1}x_1 \cdots \partial^{k_d}x_d},
  \end{align}
for $k_1 + ... + k_d=N$ and $N\in \N$. The fractional Laplace operator $\triangle^{N/2}$ acts as an isotropic low-pass filter.
The main properties of the Riesz transform can be summarized in the following way~\cite{unser10}:
\begin{itemize}
 \item \textbf{translation equivariance:} For  $x_0 \in \R^d$ define a translation operator $\mathcal{T}_{x_0}(f)(x):L_2(\R^d) \to L_2(\R^d)$ as $\mathcal{T}_{x_0}(f)(x) = f(x-x_0)$. It holds that 
       \begin{equation}
            \mathcal{R}_j(\mathcal{T}_{x_0}(f))(x) = \mathcal{T}_{x_0}(\mathcal{R}_j(f))(x),
       \end{equation}
         where $j\in\{1,\cdots,d\}$. This property reflects the fact that the Riesz transform commutes with the translation operator.
 \item \textbf{steerability:}
       The directional Hilbert transform $\mathcal{H}_v:L_2(\R^d) \to L_2(\R^d)$ in direction $v\in\R^d$, $||v||=1$, is defined as $\mathcal{F}(\mathcal{H}_v(f))(u) = i \text{ } \sign(\langle u,v\rangle)$. 
       $\mathcal{H}_v$ is steerable in terms of the Riesz transform, that is it can be written as a linear combination of the Riesz transforms
       \begin{equation}
           \mathcal{H}_v(f)(x) = \sum_{j=1}^d v_j \mathcal{R}_j(f)(x)  = \langle \mathcal{R} (f)(x), v\rangle.
       \end{equation}
       Note that in 2d,  for a unit vector $v=(\cos\theta, \sin\theta),\ \theta\in\left[0,2\pi\right]$, the directional Hilbert transform becomes 
        $\mathcal{H}_v(f)(x) = \cos(\theta)\mathcal{R}_1(f)(x) + \sin(\theta)\mathcal{R}_2(f)(x)$.
        This is equivalent to the link between gradient and directional derivatives \cite{unser10} and a very useful property for learning oriented features.
\item \textbf{all-pass filter \cite{felsberg01}:}
    Let $H = (H_1, \cdots, H_d)$ be the Fourier transform of the Riesz kernel, i.e. $\mathcal{F}(\mathcal{R}_j(f))(u) = i\frac{u_j}{|u|}\mathcal{F}(f)(u) = H_j(u)\mathcal{F}(f)(u)$. 
    The energy of the Riesz transform for frequency $u \in \R^d$ is defined as the norm of the $d$-dimensional vector $H(u)$ and has value $1$ for all non-zero frequencies $u\neq0$, i.e.
       \begin{equation}
           ||H(u)|| = 1, \quad u \neq 0.
       \end{equation} 
       The all-pass filter property reflects the fact that the Riesz transform is a non-local operator and that every frequency is treated fairly and equally. Combined with scale equivariance, this eliminates the need for multiscale analysis or multiscale feature extraction.
       \\
 \item \textbf{scale (dilation) equivariance:} For $a>0$ define a dilation or rescaling operator $L_{a}:L_2(\R^d) \to L_2(\R^d)$ as $L_{a}(f)(x) = f(\frac{x}{a})$. Then
       \begin{equation}
           \mathcal{R}_j(L_{a}(f))(x) =  L_{a}(\mathcal{R}_j(f))(x),
       \end{equation} 
        for $j\in\{1,\cdots,d\}$.  That is, the Riesz transform does not only commute with translations but also with scaling. 
\end{itemize}



Scale equivariance enables an equal treatment of the same objects at different scales. As this is the key property of the Riesz transform for our application, we will briefly present a proof. 
We restrict to the first order in the Fourier domain. The proof for higher orders follows directly from the one for the first order. 
\begin{lemma}
The Riesz transform is scale equivariant, i.e. 
\begin{equation}
\label{scale:invariance}
\mathcal{R}_{i}f\Big(\frac{\cdot}{a}\Big) = [\mathcal{R}_{i}f]\Big(\frac{\cdot}{a}\Big).
\end{equation}
for $f \in L_2(\R^d)$. 
\end{lemma}

 \begin{proof}
Remember that the Fourier transform of a dilated function is given by 
$\mathcal{F}(f(\alpha \cdot))(u) = \frac{1}{\alpha^d} \mathcal{F}(f)(\frac{u}{\alpha})$. Setting $g(x) = f(\frac{x}{a})$, we have $\mathcal{F}(g)(u) = a^d \mathcal{F}(f)(au)$. This yields 
\begin{align*}
    &\mathcal{F}\Bigg(\mathcal{R}_j\Big(f\big(\frac{\cdot}{a}\big)\Big)\Bigg)(u) =
    \mathcal{F}\Big(\mathcal{R}_j(g)\Big)(u) = 
    \\
    &= i \frac{u_j}{|u|} \mathcal{F}(g)(u) 
    = i \frac{u_j}{|u|} a^d \mathcal{F}(f)(au) = \\
    &= a^d  \Big(i\frac{au_j}{a|u|}\Big)\mathcal{F}(f)(au) =  
     a^d \mathcal{F}\Big(\mathcal{R}_j(f)\Big)(au) =  
     \\
     &= \mathcal{F}\Bigg(\mathcal{R}_j(f)\Big(\frac{\cdot}{a}\Big)\Bigg)(u).
\end{align*}
\end{proof}

\begin{figure*}[h]
    \centering
    \includegraphics[width = 0.35\textwidth]{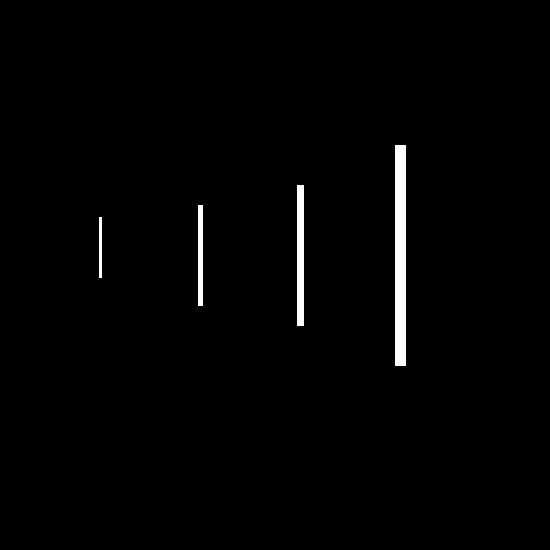}
    \includegraphics[width = 0.6\textwidth]{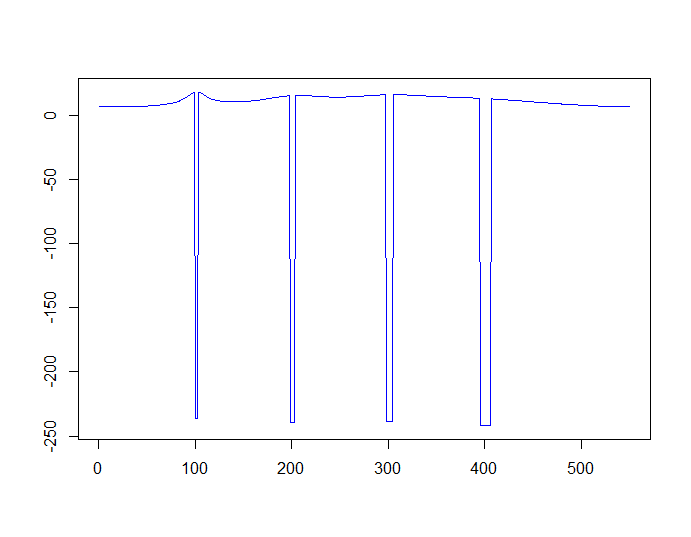}
    \caption{Illustration of the Riesz transform on a mock example of $550\times 550$ pixels: aligned rectangles with equal aspect ratio and constant gray value $255$ (left) and response of the second order Riesz transform $\mathcal{R}^{(2,0)}$ of the left image sampled horizontally through the centers of the rectangles (right).}
    \label{fig:scale-inv-ex}
\end{figure*}

Fig.~\ref{fig:scale-inv-ex} provides an illustration of the scale equivariance. It shows four rectangles with length-to-width ratio $20$ and varying width ($3$, $5$, $7$, and $11$ pixels) together with the gray value profile of the second order Riesz transform $R^{(2,0)}$ along a linear section through the centers of the rectangles. 
In spite of the different widths, the Riesz transform yields equal filter responses for each rectangle (up to rescaling). 
In contrast, to achieve the same behaviour in Gaussian scale space, the scale space has to be sampled (i.e. a subset of scales has to be selected), the $\gamma$-normalized derivative \cite{lindeberg98} has to be calculated for every scale, and finally the scale yielding the maximal absolute value has to be selected. 
In comparison, the simplicity of the Riesz transform achieving the same in just one transform without sampling scale space and without the need for a scale parameter is striking.

\section{Riesz transform neural networks}
In the spirit of structured receptive fields \cite{jacobsen16} and Gaussian derivative networks \cite{lindeberg21}, we use the Riesz transforms of first and second order instead of standard convolutions to define Riesz layers. 
As a result, Riesz layers are scale equivariant in a single forward pass. 
Replacing standard derivatives with the Riesz transform has been motivated by \cite{koethe05}, while using a linear combination of Riesz transforms of several order follows \cite{depeursinge12}. 

\subsection{Riesz layers}
The base layer of the Riesz networks is defined as a linear combination of Riesz transforms of several orders implemented as 1d convolution across feature channels (Fig.~\ref{fig:network-blocks}). Here, we limit ourselves to first and second order Riesz transforms. Thus, the linear combination reads as
\begin{align}
    J_{\mathcal{R}}(f) &= C_0 + \sum_{k=1}^{d}{C_k\cdot \mathcal{R}_k(f)} + \nonumber
     \\
     &+\ \mathop{\sum}_{k,l\in \mathbb{N}_0, k+l=2}{C_{k,l}\cdot \mathcal{R}^{(k,l)}(f)},
    \label{base:layer}
\end{align}
where $\{C_0, C_k \vert k\in\{1,\cdots,d\} \}\cup\{C_{k,l} \vert l,k \in \N_0, l+k=2\}$ are parameters that are learned during training. 

Now we can define the general layer of the network (Fig.~\ref{fig:network-blocks}). Let us assume that the $K$th network layer takes input 
$F^{(K)} = ( F^{(K)}_1,\cdots,  F^{(K)}_{c^{(K)}}) \in \big(L_2(\R^d)\big)^{c^{(K)}}$ 
with $c^{(K)}$ feature channels and has output  
$F^{(K+1)}= (F^{(K+1)}_1,\cdots,F^{(K+1)}_{c^{(K+1)}})\in \big(L_2(\R^d)\big)^{c^{(K+1)}}$  
with $c^{(K+1)}$ channels. Then the output in channel $j \in \{1,\cdots,c^{(K+1)} \}$ is given by
\begin{equation}
    F^{(K+1)}_j = \sum_{i=1}^{c^{(K)}} J_{K}^{(j,i)}(F^{(K)}_i).
    \label{full:layer}
\end{equation}
Here, $J_{K}^{(j,i)}$ is defined in the same way as $J_{\mathcal{R}}(f)$ from equation \eqref{base:layer}, but trainable parameters may vary for different input channels $i$ and output channels $j$, i.e. 
\begin{align}
    J^{(j,i)}_K(f) &= C_0^{(j,i,K)} + \sum_{k=1}^{d}{C_k^{(j,i,K)}\cdot \mathcal{R}_k(f)} + \nonumber \\ &+ \mathop{\sum}_{k,l\in \mathbb{N}_0, k+l=2}{C_{k,l}^{(j,i,K)}\cdot \mathcal{R}^{(k,l)}(f)}.
\end{align}
In practice, the offset parameters $C_0^{(i,j,K)}$, $i=1,\cdots,c^{(K)}$ are replaced with a single parameter defined as $C_0^{(j,K)} := \sum_{i=1}^{c^{(K)}}C_0^{(j,i,K)}$.

\begin{figure*}
    \centering
    \includegraphics[width = 0.49\textwidth]{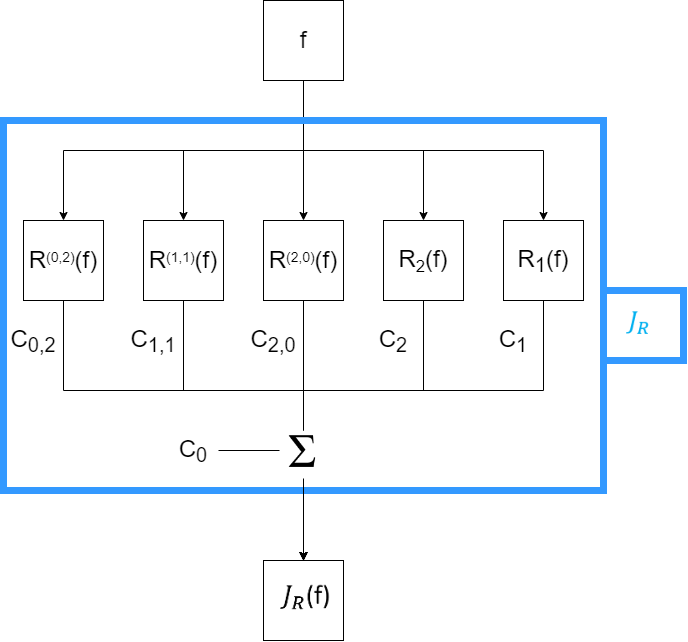}
    \hfill
    \includegraphics[width = 0.4\textwidth]{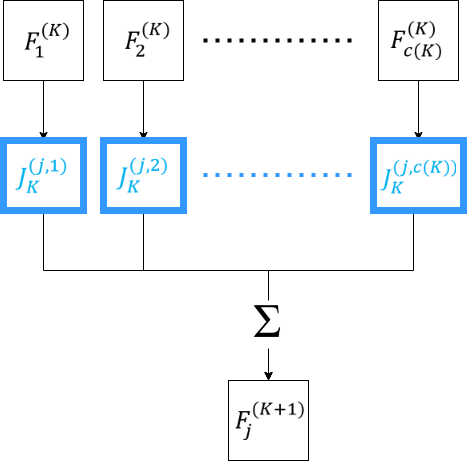}
    \caption{Building blocks of Riesz networks: the base Riesz layer from equation (\ref{base:layer}) (left) and the full Riesz layer from equation (\ref{full:layer}) (right).}
    \label{fig:network-blocks}
\end{figure*}

\subsection{Proof of scale equivariance}

We prove the scale equivariance for $J_{\mathcal{R}}(f)$. That implies scale equivariance for $J^{(j,i)}_K(f)$ and consequently for $F^{(K+1)}_j$ for arbitrary layers of the network. By construction (see Section \ref{sec:NetworkDesign}), this will result in provable scale equivariance for the whole network.
    Formally, we show that $J_{\mathcal{R}}(f)$ from equation (\ref{base:layer}) is scale equivariant, i.e.  
    \begin{equation}
        J_{\mathcal{R}}\Bigg(f(\frac{\cdot}{a})\Bigg)  =  J_{\mathcal{R}}(f) (\frac{\cdot}{a}),
    \end{equation}
    for $f \in L_2(\R^d)$. 
\begin{proof}
For any scaling parameter $a > 0$ and $x\in \R^d$, we have
\begin{eqnarray*}
     J_{\mathcal{R}}\Bigg(f(\frac{\cdot}{a})&\Bigg)& (x) = C_0 + \sum_{k=1}^{d}{C_k\cdot \mathcal{R}_k\Bigg(f(\frac{\cdot}{a})\Bigg)(x)} + 
     \\
     &+&\mathop{\sum}_{k,l\in \mathbb{N}_0, k+l=2}{C_{k,l}\cdot \mathcal{R}^{(k,l)}\Bigg(f(\frac{\cdot}{a})\Bigg)(x)} =
     \\
     &\overset{\mathrm{\eqref{scale:invariance}}}{=}&C_0 + \sum_{k=1}^{d}{C_k\cdot \mathcal{R}_k(f)(\frac{x}{a})} +
     \\
     &+&\mathop{\sum}_{k,l\in \mathbb{N}_0, k+l=2}{C_{k,l}\cdot \mathcal{R}^{(k,l)}(f)(\frac{x}{a})} \\ 
     &=& J_{\mathcal{R}}(f) (\frac{x}{a}). 
\end{eqnarray*}
\end{proof}

\subsection{Network design}
\label{sec:NetworkDesign}
The basic building block of modern CNNs is a sequence of the following operations: batch normalization, spatial convolution (e.g. $3\times3$ or $5\times5$), the Rectified Linear Unit (ReLU) activation function, and Max Pooling. Spatial convolutions have by default a limited size of the receptive field and Max Pooling is a downsampling operation performed on a window of fixed size. For this reason, these two operations are not scale equivariant and consequently CNNs are sensitive to variations in the scale.
Hence, among the classical operations, only batch normalization \cite{ioffe15} and ReLU activation preserve scale equivariance. To build our neural network from scale equivariant transformations, only, we restrict to using batch normalization, ReLUs, and Riesz layers, which serve as a replacement for spatial convolutions. In our setting, Max Pooling can completely be avoided since its main purpose is to combine it with spatial convolutions in cascades to increase the size of the receptive field while reducing the number of parameters.

Generally, a layer consists of the following sequence of transformations: batch normalization, Riesz layer, and ReLU.
Batch normalization improves the training capabilites and avoids overfitting, ReLUs introduce non-linearity, and the Riesz layers extract scale equivariant spatial features.
For every layer, the number of feature channels has to be selected. Hence, our network with $K \in \N$ layers can be simply defined by a $(K+2)$-tuple specifying the channel sizes\footnote{Channel dimension $c(0)$ denotes the dimension of input image, e.g. for gray value images $c(0)=1$. Channel dimension $c(K+1)$ denotes the dimension of the final output of the network. For the crack segmentation, the output map is the binary image, e.g. $c(K+1)=1$.} e.g. $(c^{(0)}, c^{(1)}, \cdots c^{(K)},c^{(K+1)})$. The final layer is defined as a linear combination of the features from the previous layer followed by a sigmoid function yielding the desired probability map as output.

The four layer Riesz network we apply here can be schematically written as $1 \to 16 \to 32 \to 40 \to 48 \to 1$ and has $(1\cdot 5\cdot16+16)+(16\cdot 5\cdot 32+32)+(32\cdot 5\cdot 40+40)+(40\cdot 5\cdot 48+48)+(48\cdot1+1)=18\,825$ 
trainable parameters.

\section{Experiments and applications}

In this section we evaluate the four layer Riesz network defined above on the task of  segmenting cracks in 2d slices from CT images of concrete. 
Particular emphasis is put on the network's ability to segment multiscale cracks and to generalize to crack scales unseen during training. 
To quantify these properties, we use images with simulated cracks. Being accompanied by an unambiguous ground truth, they allow for an objective evaluation of the segmentation results.

Additionally, in Appendix \ref{secA1} scale equivariance of the Riesz network is experimentally validated on the MNIST Large Scale data set \cite{jansson22}.

\paragraph{Data generation:}
Cracks are generated by the fractional Brownian motion (Experiment 1) or minimal surfaces induced by the facet system of a Voronoi tessellation (Experiment 2). Dilated cracks are then integrated into CT images of concrete without cracks. As pores and cracks are both air-filled, their gray value distribution should be similar. Hence, the gray value distribution of crack pixels is estimated from the gray value distribution observed in air pores. 
The crack thickness is kept fixed (Experiment 1) or varies (Experiment 2) depending on the objective of the experiment. As a result, realistic semi-synthetic images can be generated (see Fig.~\ref{fig:train-w3}). For more details on the simulation procedure, we refer to \cite{barisin22, jung22a}. Details on number and size of the images can be found below.
Finally, we show applicability of the Riesz network for real data containing cracks generated by tensile and pull-out
tests.

\begin{figure*}[h]
    \centering
    \includegraphics[width = 0.25\textwidth]{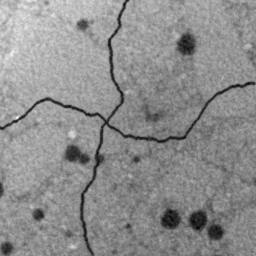}
    \includegraphics[width = 0.25\textwidth]{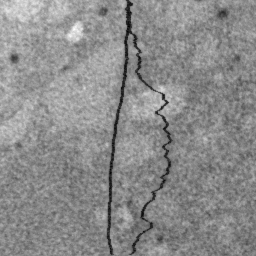}
    \includegraphics[width =    0.25\textwidth]{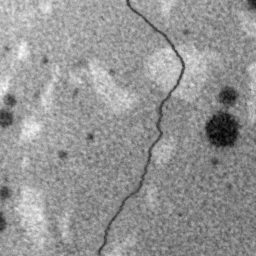}
    
    \includegraphics[width =    0.125\textwidth]{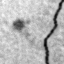}
    \includegraphics[width =    0.125\textwidth]{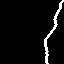}
    \includegraphics[width =    0.125\textwidth]{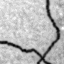}
    \includegraphics[width =    0.125\textwidth]{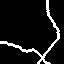}
    \includegraphics[width =    0.125\textwidth]{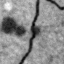}
    \includegraphics[width =    0.125\textwidth]{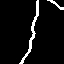}
    \caption{Cracks of width 3 used for training: before (first row) and after cropping 
    (second row). Image sizes are $256\times 256$ (first row) and $64 \times 64$ (second row).}
    \label{fig:train-w3}
\end{figure*}

\paragraph{Quality metrics:} 
As metrics for evaluation of the segmentation results we use precision (P), recall (R), F1-score (or dice coefficient, Dice), and Intersection over Union (IoU). 
The first three quality metrics are based on true positives \textit{tp} -- the number of pixels correctly predicted as crack, true negatives \textit{tn} -- the number of pixels correctly predicted as background, false positives \textit{fp} -- the number of pixels wrongly predicted as crack, and false negatives \textit{fn} -- the number of pixels falsely predicted as background. Precision, recall, and dice coefficient are then defined via
$$P = tp/(tp+fp),\quad R = tp/(tp+fn),$$
$$\text{Dice} = 2PR/(P + R).$$

IoU compares union and intersection of the foregrounds $X$ and $Y$ in the segmented image and the corresponding ground truth, respectively. That is
$$ IoU(X,Y) = \frac{|X\cap Y|}{|X \cup Y|}.$$
All these metrics have values in the range $[0,1]$ with values closer to $1$ indicating a better performance.

\paragraph{Training parameters:} 
If not specified otherwise, all models are trained on cracks of fixed width of $3$ pixels. Cracks for the training are generated in the same way as for Experiment 1 on $256\times 256$ sections of CT images of concrete. Then, 16 images of size $64\times 64$ are cropped without overlap from each of the generated images. In this way, images without cracks are present in the training set.
After data augmentation by flipping and rotation, the training set consists of 1\,947 images of cracks. Some examples are shown in Fig.~\ref{fig:train-w3}. For validation, another set of images with cracks of width $3$ is used. The validation data set's size is a third of the size of the training set. 

All models are trained for $50$ epochs with initial learning rate $0.001$ which is halved every $20$ epochs. ADAM optimization \cite{kingma14} is used, while the cost function is set to binary cross entropy loss. 

Crack pixels are labelled with $1$, while background is labelled with $0$. As there are far more background than crack pixels, we deal with a highly imbalanced data set. Therefore, crack and pore pixels are given a weight of $40$ to compensate for class imbalance and to help distinguishing between these two types of structures which hardly differ in their gray values.

\subsection{Measuring scale equivariance}
Measures for assessing scale equivariance have been introduced in \cite{worral19, sosnovik20}. For an image or feature map $f$, a mapping function $\Phi$ (e.g. a neural network or a subpart of a network), and a scaling function $L_a$ we define
\begin{equation}
   \Delta_a(\Phi) := \frac{|| L_a(\Phi(f)) - \Phi(L_a(f))||_2}{||L_a(\Phi(f)) ||_2}. 
\end{equation}

Ideally, this measure should be $0$ for perfect scale equivariance. In practice, due to scaling and discretization errors we expect it to be positive yet very small. 

To measure scale equvariance of the full Riesz network with randomly initialized weights, we use a data set consisting of $85$ images of size $512\times 512$ pixels with crack width $11$ and use downscaling factors $a\in\{2,4,8,16,32,64\}$. The evaluation was repeated for 20 randomly initialized Riesz networks. The resulting values of $\Delta_a$  are given in 
Fig.~\ref{fig:scale-error}. 

The measure $\Delta_a$ was used to validate the scale equivariance of Deep Scale-spaces (DSS) in \cite{worral19} and scale steerable equivariant networks in \cite{sosnovik20}. 
In both works, a steep increase in 
$\Delta_a$ is observed for downscaling factors larger than $16$, while for very small downscaling factors, $\Delta_a$ is reported to be below $0.01$. In \cite{sosnovik20}, $\Delta_a$ reaches $1$ for downscaling factor $45$.
The application scenario studied here differs from those of \cite{worral19, sosnovik20}. Results are thus not directly comparable but can be considered only as an approximate baseline. For small downscaling factors, we find $\Delta_a$ to be higher than in~\cite{sosnovik20} (up to $0.075$). However, for larger downscaling factors $(a>32)$, $\Delta_a$ increases more slowly e.g. $\Delta_{64} = 0.169$. This proves the resilience of Riesz networks to very high downscaling factors, i.e. large changes in scale.

\begin{figure}
    \centering
    \includegraphics[width = 0.49\textwidth]{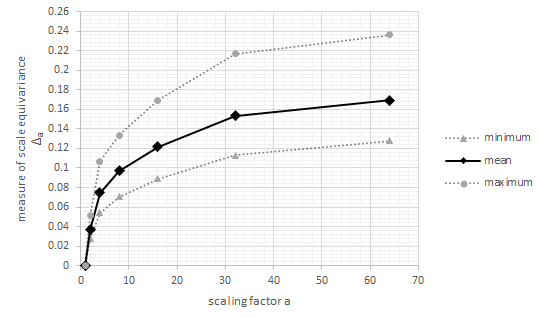}
    \caption{Measure of scale equivariance $\Delta_a$ for the four layer Riesz network with randomly initialized parameters w.r.t. the downsacling factor a. Mean (black), minimum, and maximum (gray) of 20 repetitions. Points on the line correspond to $a~\in~\{1,2,4,8,16,32,64\}$.}
    \label{fig:scale-error}
\end{figure}

\subsection{Experiment 1: Generalization to unseen scales}
Our models are trained on images of fixed crack width $3$. To investigate their behaviour on crack widths outside of the training set, we generate cracks of widths $\{1, 3, 5, 7, 9, 11\}$ pixels in images of size $512\times 512$, see Fig.~\ref{fig:test-fix-width-results}. Each class contains $85$ images. Besides scale generalization properties of the Riesz network, we check how well it generalizes to random variations in crack topology or shapes, too. 
For this experiment we will assume that the fixed width is known. This means that the competing methods will use a single scale which is adjusted to this width.

\subsubsection{Ablation study on the Riesz network}
We investigate how the network parameters and the composition of the training set affect the quality of the results, in order to learn how to design this type of neural networks efficiently.
 \begin{table*}[]
     \centering
     \begin{tabular}{|c|c|c|c|c|c|c|}
     \hline
     \multirow{2}{*}{Method} & \multicolumn{1}{|c|}{w1} & \multicolumn{1}{|c|}{w3} & \multicolumn{1}{|c|}{w5} & \multicolumn{1}{|c|}{w7} & \multicolumn{1}{|c|}{w9} & \multicolumn{1}{|c|}{w11}\\
     \hhline{~|-|-|-|-|-|-|}
     & Dice & Dice & Dice & Dice & Dice & Dice \\
     \hline
     \hline
     baseline & 0.352  & 0.895 & 0.941 & 0.954 & 0.962 & 0.964 \\
     \hline
     \hline
          trainset 489 & \cellcolor{lightgray}0.356 &  0.877 & 0.929 & 0.945 & 0.954  & 0.958 \\
     trainset 975 & \cellcolor{lightgray}0.365 & \cellcolor{gray}0.919 & \cellcolor{lightgray}0.942 & 0.954 & \cellcolor{lightgray}0.964 & \cellcolor{lightgray}0.966\\
     \hline
     width 1 & \cellcolor{gray}0.535 & 0.761 & 0.738 & 0.678 & 0.634 & 0.631 \\
     width 5 & 0.317 & 0.891 & 0.935 & 0.951 & 0.959 & 0.957 \\
     mixed width & 0.297 & 0.865 & 0.905 & 0.935 & 0.954 & 0.962 \\
     \hline

     layer 2 & 0.297 & 0.865 & 0.905 & 0.935 & 0.954 & 0.962 \\
     layer 3 & \cellcolor{lightgray}0.366 & \cellcolor{gray}0.915  & 0.940 & 0.954 & \cellcolor{lightgray}0.966& \cellcolor{lightgray}0.971 \\
     layer 5 & \cellcolor{gray}0.390 & \cellcolor{lightgray}0.914  & \cellcolor{lightgray}0.950  & \cellcolor{lightgray}0.960  & \cellcolor{lightgray}0.969 & \cellcolor{lightgray}0.972 \\
     \hline
    \end{tabular}
     \caption{Experiment 1. Ablation study: scale generalization ability of Riesz networks. Baseline is trained on $1\,947$ images with cracks of width $3$ and has $4$ layers. Cells are colored in lightgray if the metric is better than for the baseline, but not by more than $0.02$. Dark gray color is used for metrics being more than $0.02$ better compared to the baseline.}
     \label{tab:ablation}
 \end{table*}

\paragraph{Size of training set:} 
First, we investigate the robustness of the Riesz network to the size of the training set. The literature \cite{jacobsen16} suggests that neural networks based on \textit{structure receptive fields} are less data hungry, i.e. their performance with respect to the size of the training set is more stable than that of conventional CNNs. Since the Riesz network uses the Riesz transform instead of a Gaussian derivative as in \cite{jacobsen16}, it is expected that the same would hold here, too.

The use of smaller training sets has two main benefits. First, obviously, smaller data sets reduce the effort for data collection, i.e. annotation or simulation. 
Second, smaller data sets reduce the training time for the network if we do not increase the number of epochs during training.

We constrain ourselves to three sizes of training sets: $1\,947$, $975$, and $489$. These numbers refer to the sets after data augmentation by flipping and rotation. 
Hence, the number of original images is three times smaller. 
In all three cases we train the Riesz network for $50$ epochs and with similar batch sizes ($11,13,$ and $11$, respectively).
Results on unseen scales with respect to data set size are shown in Table~\ref{tab:ablation} and Fig.~\ref{fig:training-size-experiment} (left). We observe that the Riesz network trained on the smallest data set is competitive with counterparts trained on larger data sets albeit featuring generally $1-2\%$ lower Dice and IoU. 
\begin{figure*}
    \centering
    \includegraphics[width=0.32\textwidth]{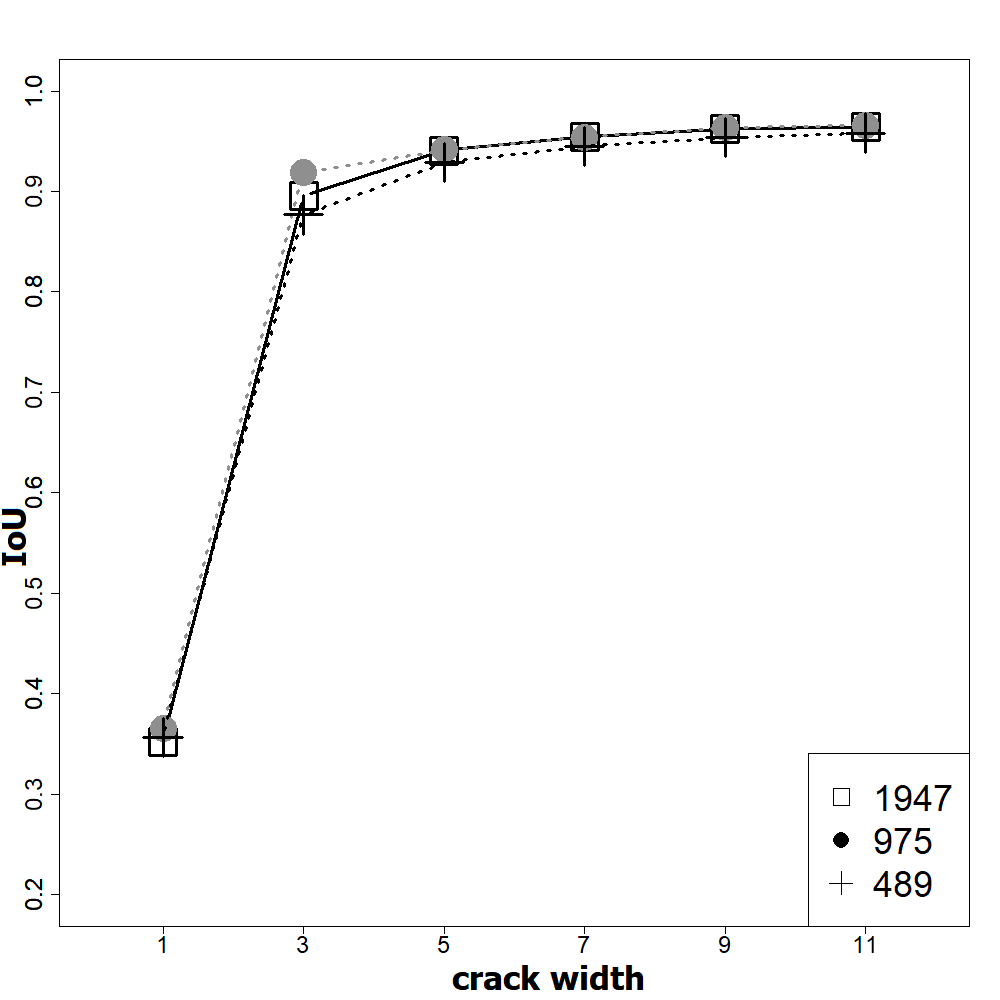}
    \includegraphics[width=0.32\textwidth]{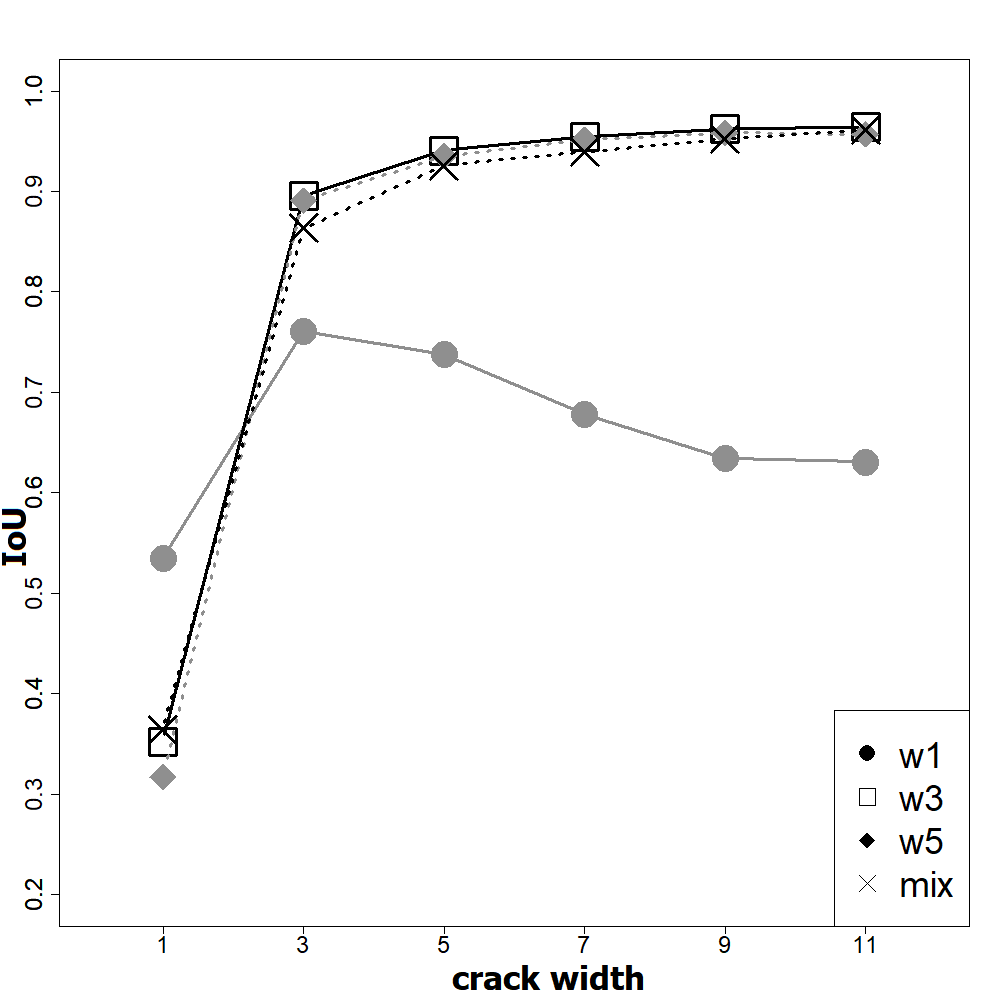}
    \includegraphics[width=0.32\textwidth]{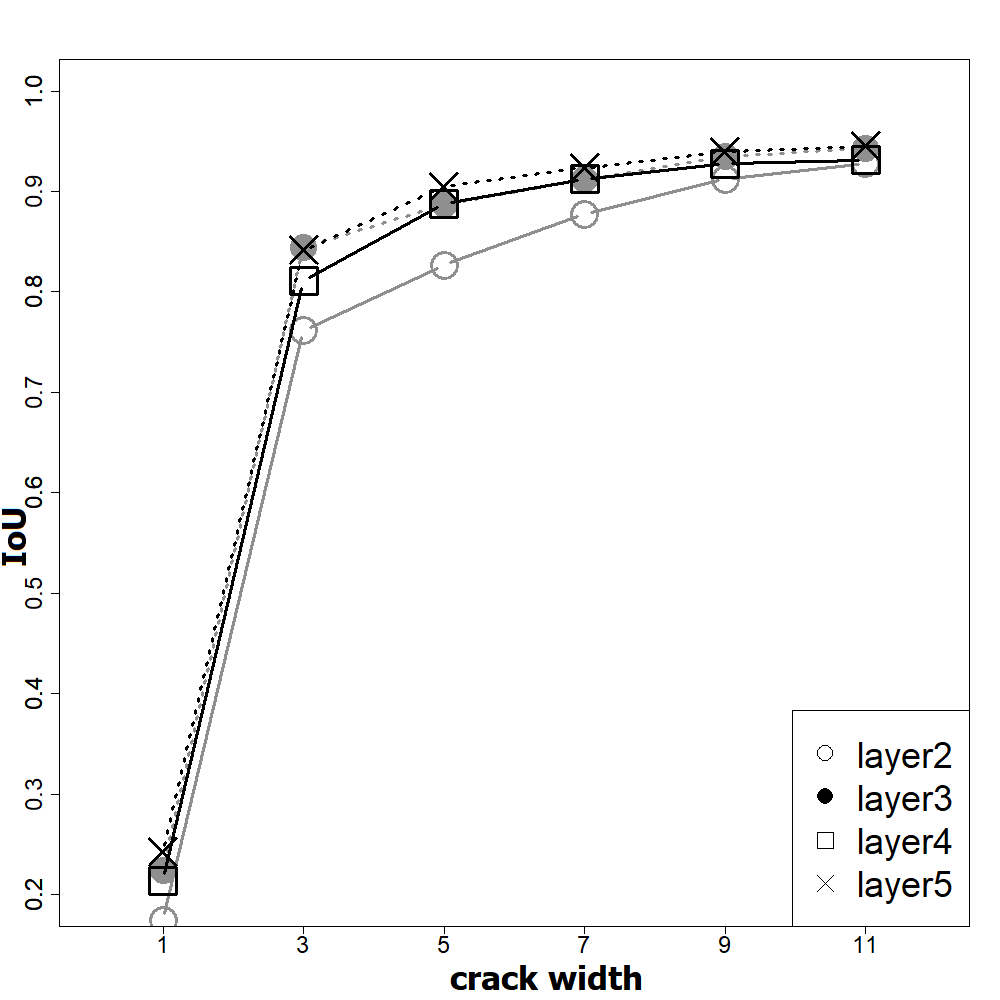}
    \caption{Experiment 1. Effect of the training set size (left), the crack width in the training set (center), and the network depth (right) on generalization to unseen scales. The baseline Riesz network is marked with $1\,947$ (left), w3 (center), and layer 4 (right) and with square symbol $\square$. Quality metric: IoU.}
    \label{fig:training-size-experiment}
    
\end{figure*}

\paragraph{Choice of crack width for training:}
There are two interesting questions with respect to crack width. Which crack width is suitable for training of the Riesz network? Do varying crack thicknesses in the training set improve performance significantly?

To investigate these questions, we choose three training data sets with cracks of fixed widths $1$, $3$, or $5$. A fourth data set combines crack widths $1$, $3$, and $5$. We train the Riesz network with these sets and evaluate its generalization performance across scales. 
Results are summarized in Fig.~\ref{fig:training-size-experiment} (center) and Table~\ref{tab:ablation}. 
Crack widths $3$ and $5$ yield similar results, while crack width $1$ seems not to be suitable, except when trying to segment cracks of width $1$.
Cracks of width $1$ are very thin, subtle, and in some cases even disconnected. Hence, they differ significantly from thicker cracks which are 8-connected and have a better contrast to the concrete background. 
This indicates that very thin cracks should be considered a special case which requires somewhat different treatment. Rather surprisingly, using the mixed training data set does not improve the metrics. Diversity with respect to scale in the training set seems not to be a decisive factor when designing Riesz networks.

\paragraph{Number of layers:}
Finally, we investigate the explanatory power of the Riesz network depending on network depth and the number of parameters. We train four networks with $2-5$ layers and  $2\,721$, $9\,169$, $18\,825$, and $34\,265$ parameters, respectively, on the same data set for $50$ epochs. The network with $5$ layers has structure $16\to32\to40\to48\to64$ and every other network is constructed from this one by removing the required number of layers at the end.
Results are shown in Table~\ref{tab:ablation} and in Fig.~\ref{fig:training-size-experiment} (right). The differences between the networks with $3$, $4$, and $5$ layers are rather subtle. For the Riesz network with only $2$ layers, performance deteriorates considerably ($3-5\%$ in Dice and IoU).\\$ $\\ 
In general, Riesz networks appear to be robust with respect to training set size, depth of network, and number of parameters. Hence, it is not necessary to tune many parameters or to collect thousands of images to achieve good performance, in particular for generalization to unseen scales. For the choice of crack width, $3$ and $5$ seem appropriate while crack width $1$ should be avoided. 

\begin{figure*}
    \centering
    \includegraphics[width = 0.32\textwidth]{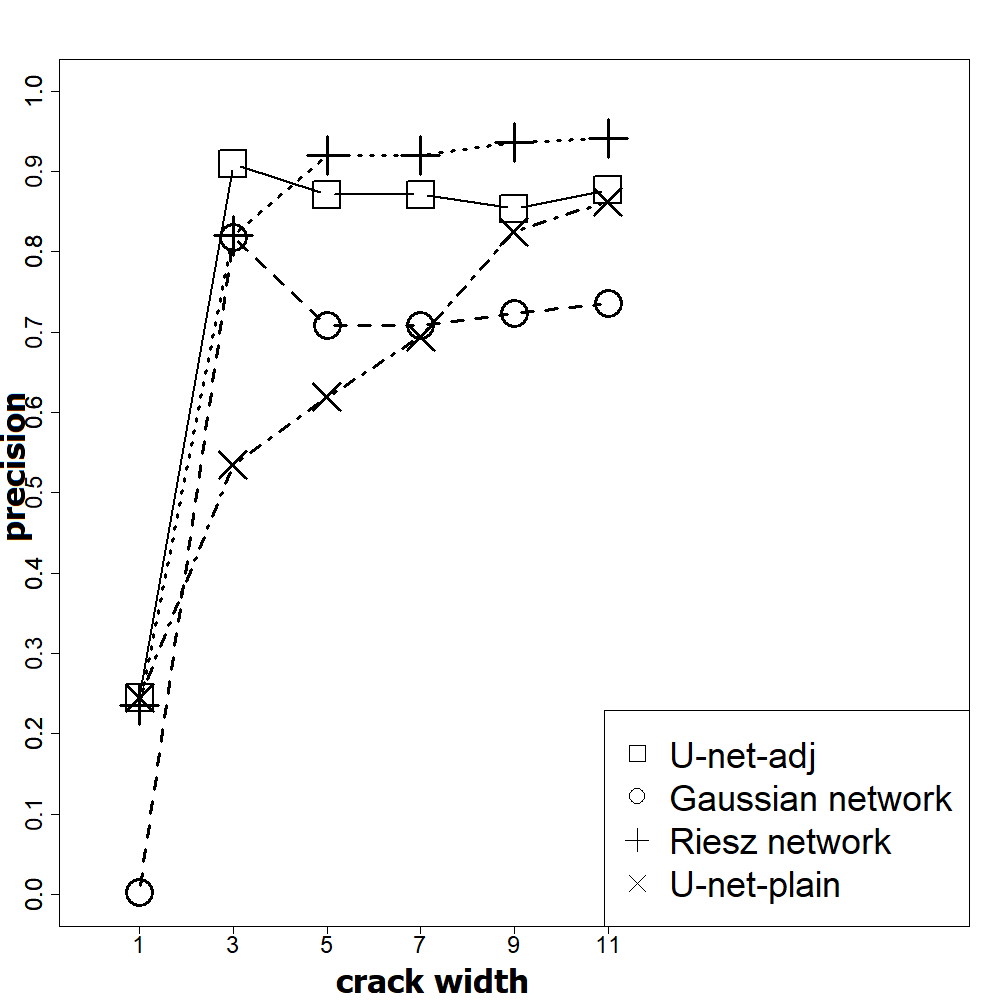}
    \includegraphics[width = 0.32\textwidth]{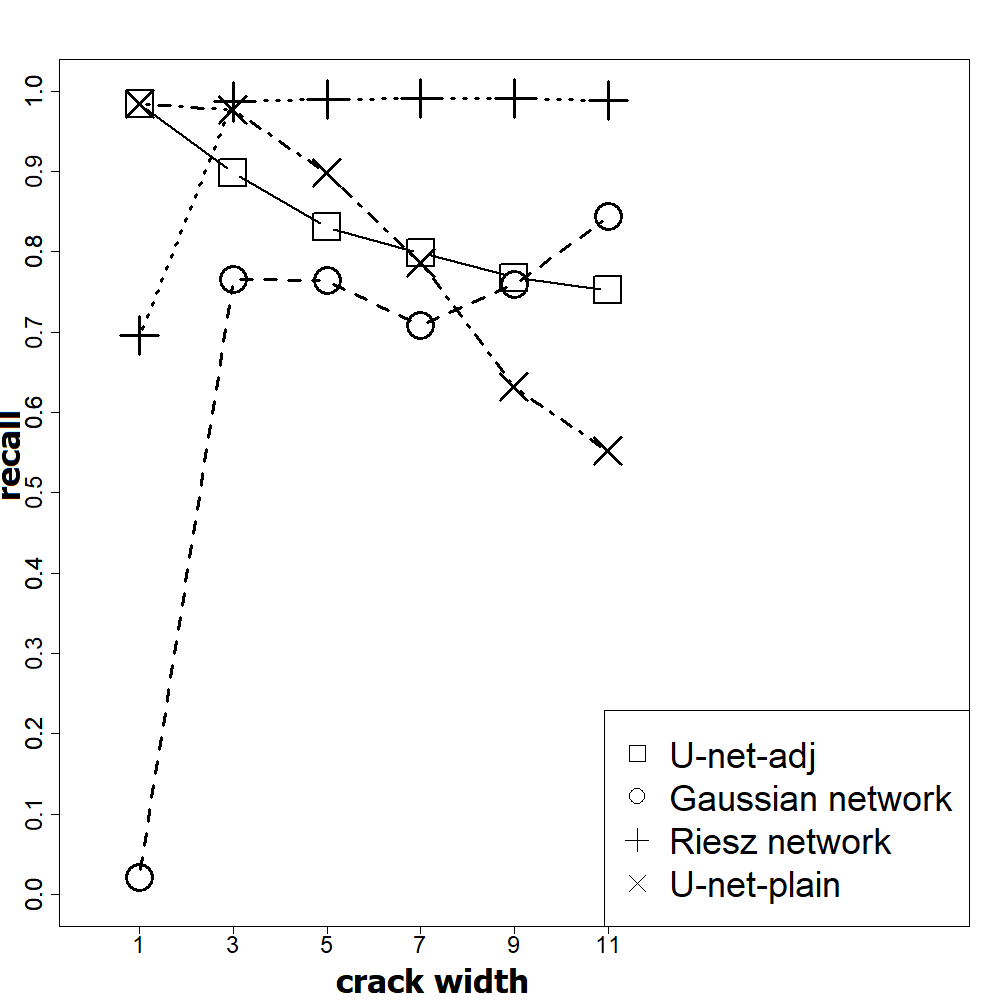}
    \includegraphics[width = 0.32\textwidth]{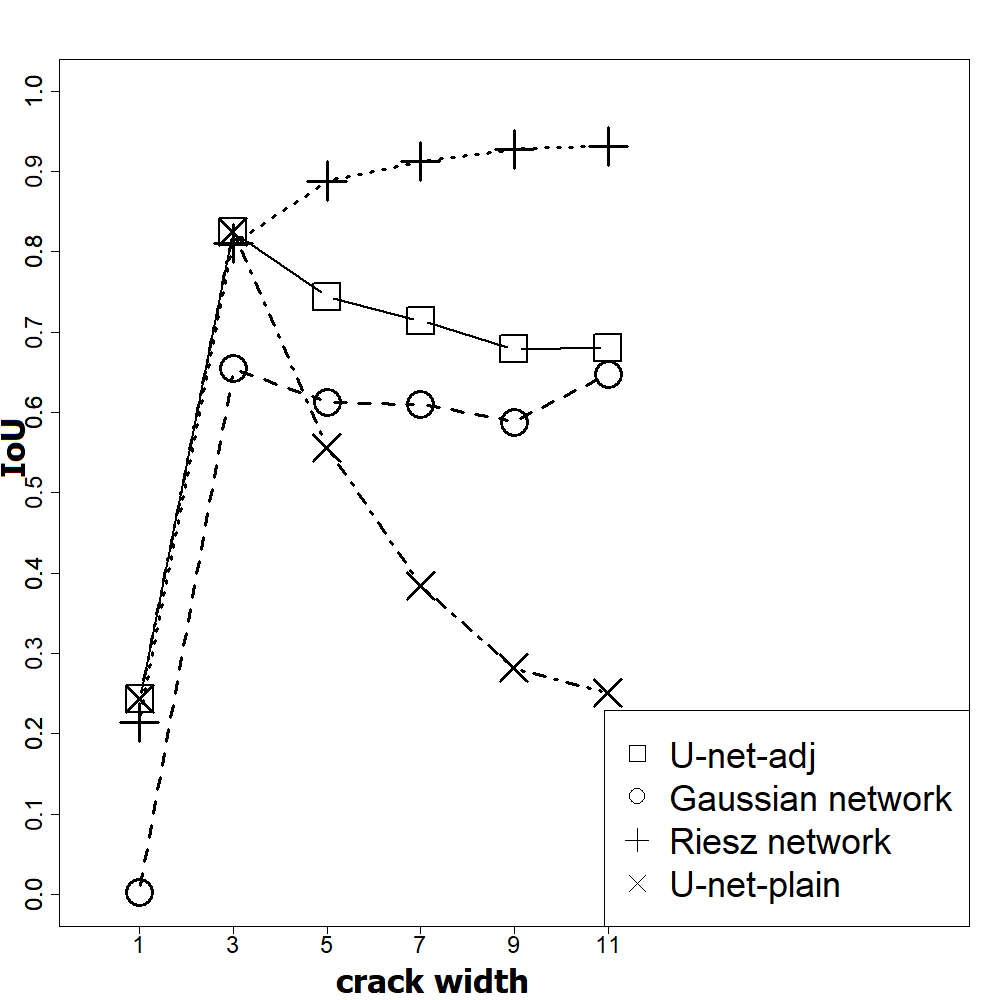}
    \caption{Experiment 1. Comparison of the competing methods. Results of the simulation study with respect to crack width. Training on crack width 3. Quality metrics (from left to right): precision, recall, and IoU.}
    \label{fig:graph-results}
\end{figure*}

\begin{figure*}
    \centering
    \hfill
    \includegraphics[width = 0.23\textwidth]{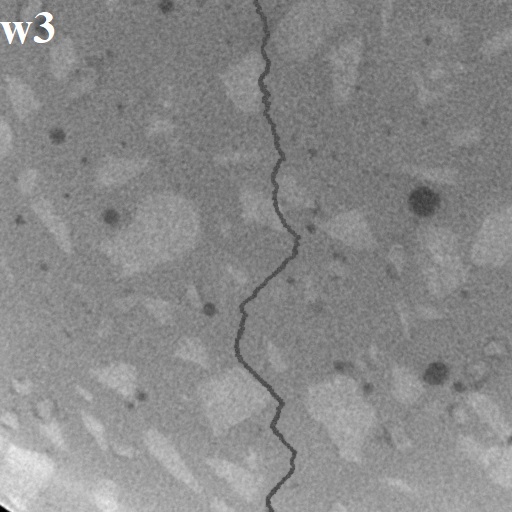}
    \includegraphics[width = 0.23\textwidth]{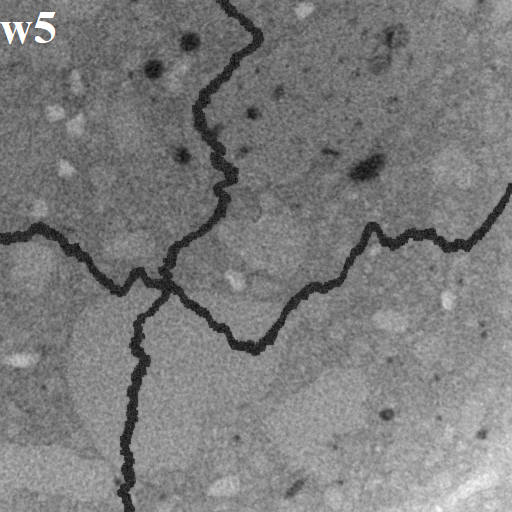}
    \includegraphics[width = 0.23\textwidth]{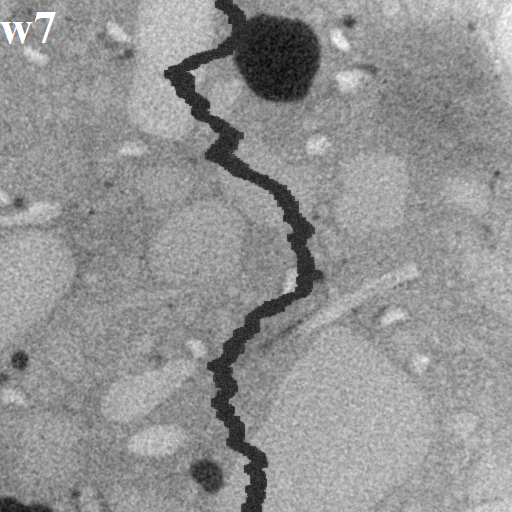}
     \includegraphics[width = 0.23\textwidth]{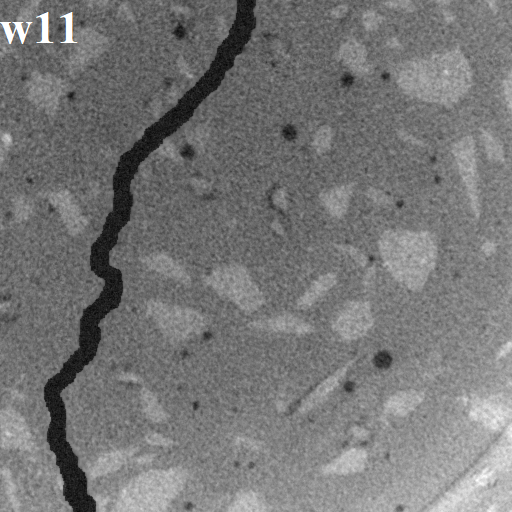}    
     
    \hfill
    \includegraphics[width = 0.23\textwidth]{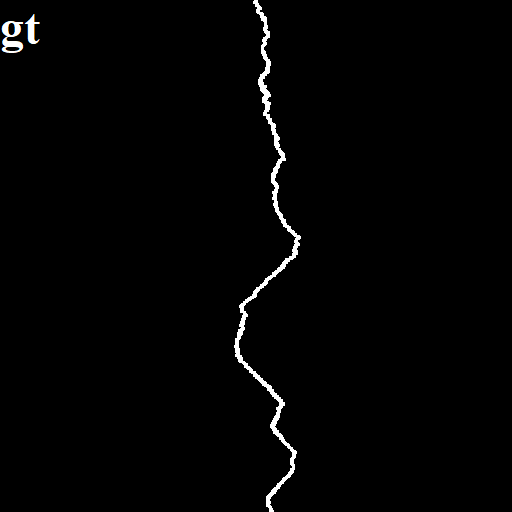}
    \includegraphics[width = 0.23\textwidth]{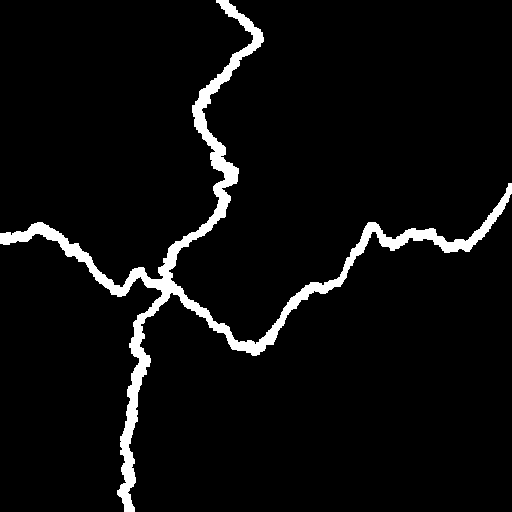}
    \includegraphics[width = 0.23\textwidth]{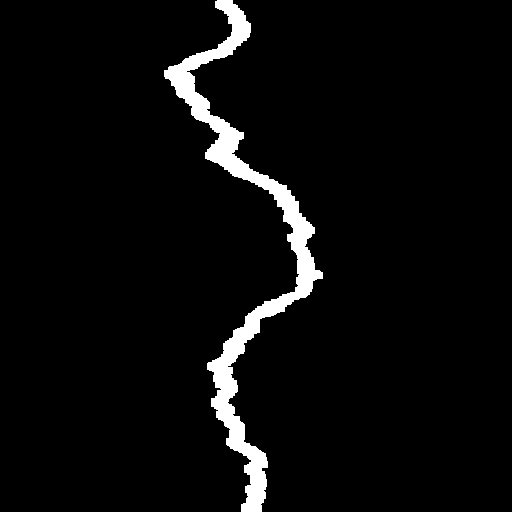}
     \includegraphics[width = 0.23\textwidth]{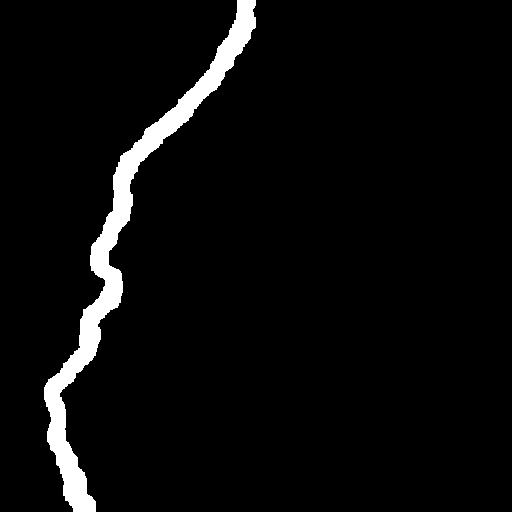}
    
    \hfill
    \includegraphics[width = 0.23\textwidth]{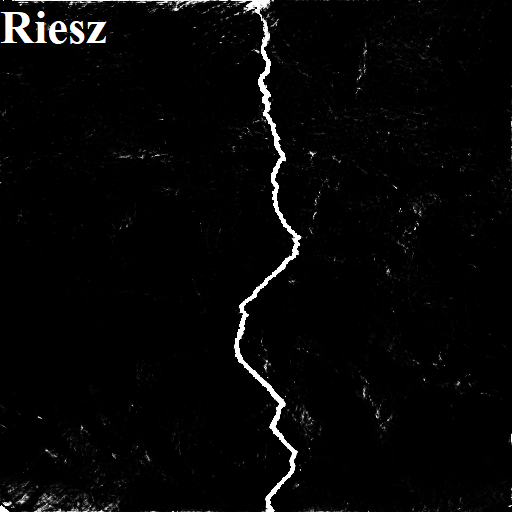}
    \includegraphics[width = 0.23\textwidth]{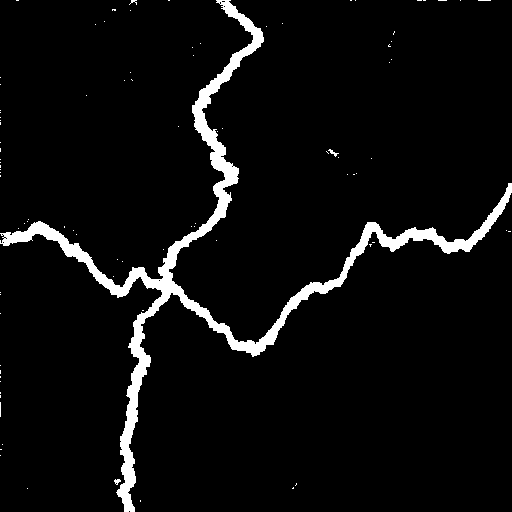}
    \includegraphics[width = 0.23\textwidth]{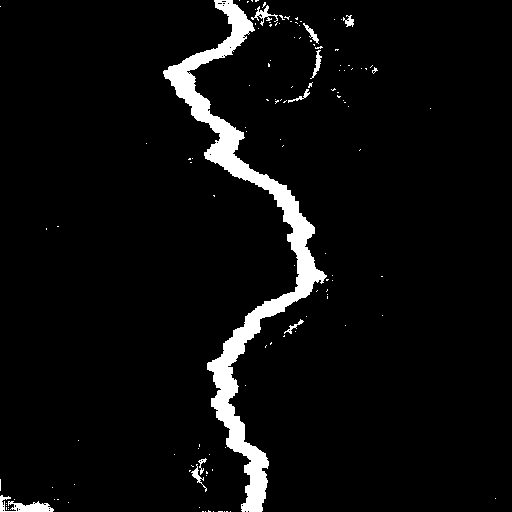}
    \includegraphics[width = 0.23\textwidth]{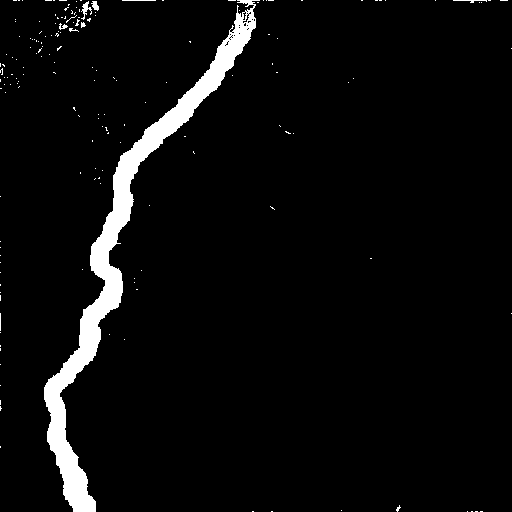}
    
    \hfill
    \includegraphics[width = 0.23\textwidth]{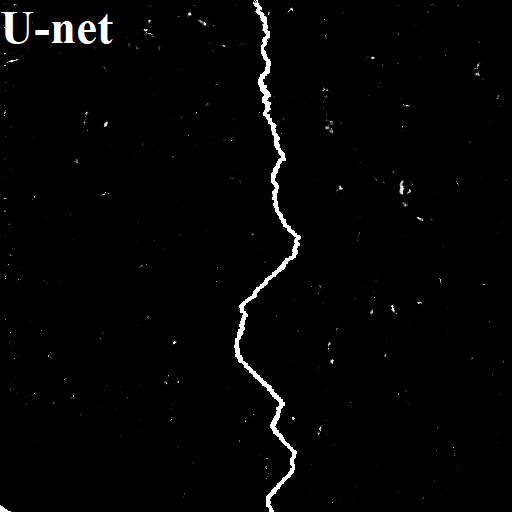}
    \includegraphics[width = 0.23\textwidth]{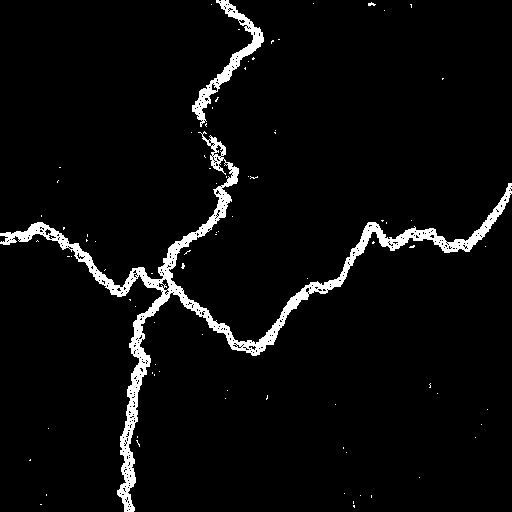}
    \includegraphics[width = 0.23\textwidth]{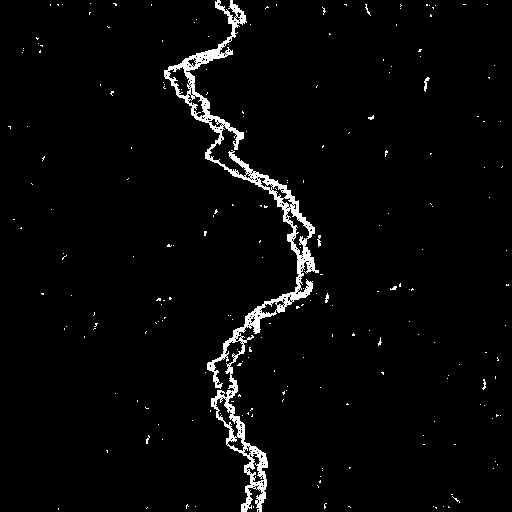}
    \includegraphics[width = 0.23\textwidth]{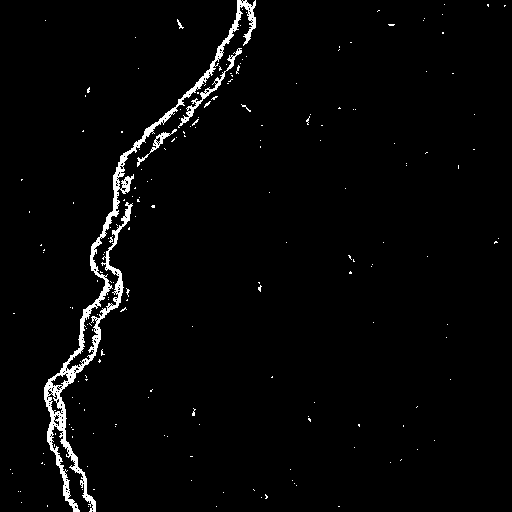}
    
    \hfill
    \includegraphics[width = 0.23\textwidth]{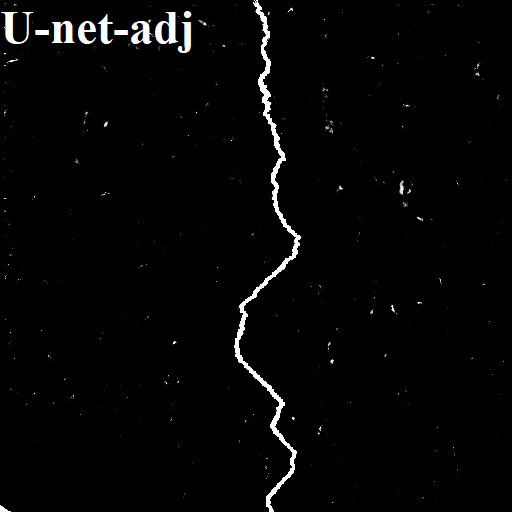}
    \includegraphics[width = 0.23\textwidth]{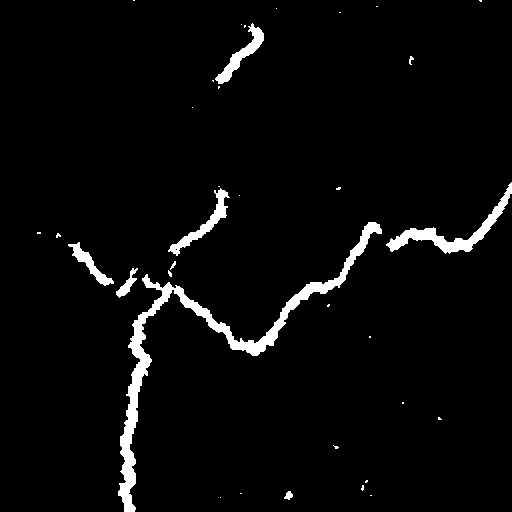}
    \includegraphics[width = 0.23\textwidth]{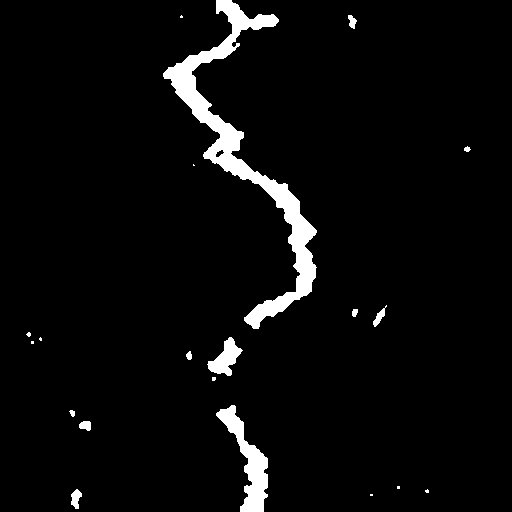}
    \includegraphics[width = 0.23\textwidth]{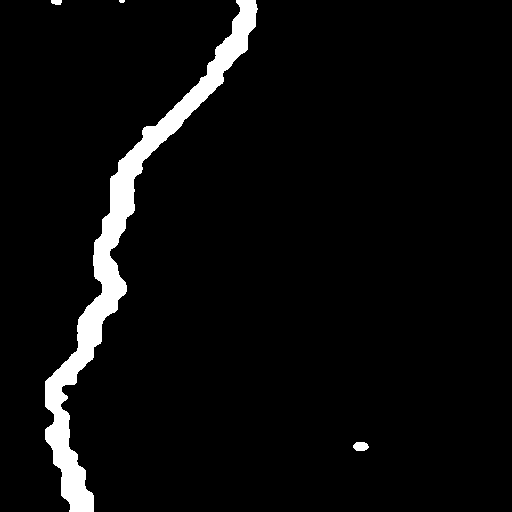}
    \\
    
    \hfill
    \includegraphics[width = 0.23\textwidth]{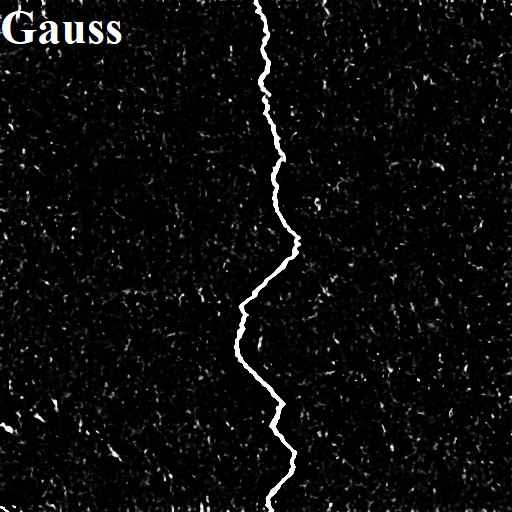}
    \includegraphics[width = 0.23\textwidth]{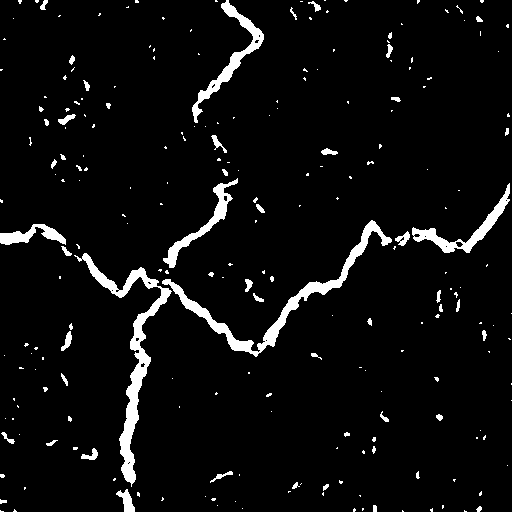}
    \includegraphics[width = 0.23\textwidth]{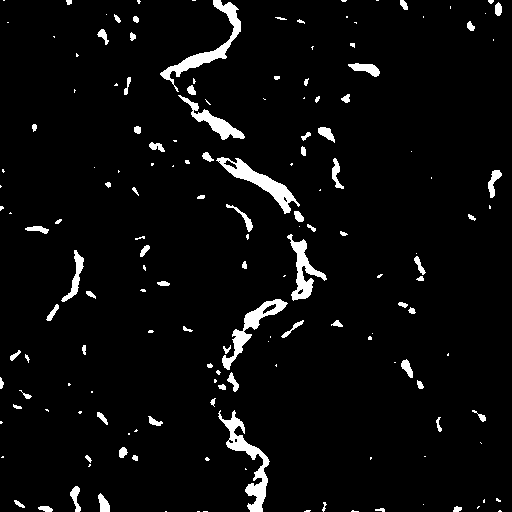}
     \includegraphics[width = 0.23\textwidth]{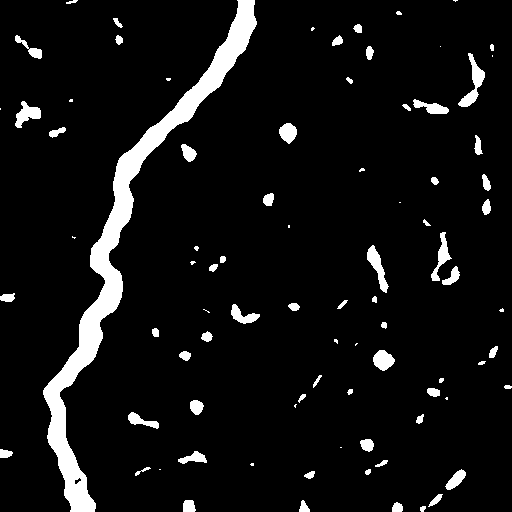}

    \caption{Experiment 1. Columns (from left to right): crack of widths 3, 5, 7, and 11. Rows (from top to bottom): input image, ground truth, Riesz network, plain U-net, U-net with scale adjustment, and Gaussian derivative network. All images have size $512\times 512$ pixels.}
    \label{fig:test-fix-width-results}
\end{figure*}

 \begin{table*}
     \centering
     \begin{tabular}{|c|c|c|c|c|c|c|}
     \hline
    \multirow{2}{*}{Method} & \multicolumn{1}{|c|}{w1} & \multicolumn{1}{|c|}{w3} & \multicolumn{1}{|c|}{w5} & \multicolumn{1}{|c|}{w7} & \multicolumn{1}{|c|}{w9} & \multicolumn{1}{|c|}{w11}\\
      \hhline{~|-|-|-|-|-|-|}
     & Dice & Dice & Dice & Dice & Dice & Dice \\
     \hline
     U-net plain & 0.391 & 0.904  & 0.715  & 0.555  & 0.440  & 0.401 \\
      U-net scale adj. & 0.391 & 0.904 & 0.853 & 0.833 & 0.809& 0.810
      \\
      U-net-mix scale adj. & \textbf{0.420} & \textbf{0.917} & 0.929 & 0.916 & 0.921 & 0.921 \\
      \hline
      Gaussian network & 0.004 & 0.765 & 0.764 & 0.709  & 0.759& 0.843 \\
       \hline
       Riesz network & 0.352 & 0.895 & \textbf{0.941} & \textbf{0.954} & \textbf{0.962} & \textbf{0.964} \\
     \hline
    \end{tabular}
     \caption{Experiment 1. Comparison with competing methods: Dice coefficients for segmentation of cracks of differing width. Training was performed on crack width 3. Best performing method bold. Both Gaussian network and U-net adj. are applied on a single scale that was selected according to the crack width.}
     \label{tab:comparison-fixed-scale}
 \end{table*}
 
\subsubsection{Comparison with competing methods}
\label{sec:competing-methods}

\paragraph{Competing methods:}
The four layer Riesz network is compared to two other methods -- Gaussian derivative networks \cite{lindeberg21}  and 
U-net \cite{brox15} on either rescaled images \cite{jansson22} or an image pyramid \cite{jung22}. 
The Gaussian derivative network uses scale space theory based on the Gaussian kernel and the diffusion equation. Using the $\gamma$-normalized Gaussian derivatives from \cite{lindeberg98},
layers of first and second order Gaussian derivatives are constructed \cite{lindeberg21}.
We shortly state differences between our reimplementation and the original work \cite{lindeberg21}. In order to reduce the computation time, we use a version of the Gaussian network that has a single scale channel corresponding to the training thickness during the training, while additional channels with shared weights are added for the inference, i.e. testing the scale generalization.
This version of the Gaussian network is different from the original one \cite{lindeberg21}, which has more scale channels during training.
Using multiple scale channels in the training step has been found to result in better scale generalization properties if the different scale channels were
allowed to compete against each other during the training stage. However, this increases the computational burden compared to a single channel network.
In this section we use a single channel version of the network but with $\sigma$ adjusted to the crack width.
The sparser scale sampling ratio of 2 used in this reimplementation from Section \ref{sec:multiscale-data} is, however,
expected to lead to lower performance compared to using a scale sampling ratio of
$\sqrt{2}$, as used in the original work.

U-net has around $2.7$ million parameters, while the Gaussian derivative network has the same architecture as the Riesz network and hence the same number of parameters ($18$k).

We design an experiment for detailed analysis and comparison of the ability of the methods to  generalize to scales unseen during training.
In typical applications, the thickness of the cracks would not be known. Here, crack thickness is kept fixed such that the correct scale of cracks is a priori known. This allows for a selection of an optimal scale (or range of scales) such that we have a best case comparison. 
 For the Gaussian derivative network, scale is controlled by the standard deviation parameter $\sigma$ which is set to the half width of the crack.
Here, we have avoided the inference of the Gaussian network with multiple scale channels and have used the assumption that the scale is a priori known. For the U-net, scale is adjusted by downscaling the image to match the crack width used in the training data. 
Here, we restrict the downscaling to discrete factors in the set $\{2,4,8,... \}$ that were determined during validation. 
For widths $1$ and $3$, no downscaling is needed. For width $5$, the images are downscaled by $2$, 
for width $7$ by $4$, and for widths $9$ and $11$ by $8$. For completeness, we include results for the U-net without downscaling denoted by "U-net plain".
Table~\ref{tab:comparison-fixed-scale} yields the prediction quality measured by the Dice coefficient, while the other quality measures are shown in Fig. \ref{fig:graph-results}. Exemplary segmentation results are shown in Fig.~\ref{fig:test-fix-width-results}.

As expected, the performance of the plain U-net decreases with increasing scale. 
Scale adjustment stabilizes U-net's performance but requires manual selection of 
scales. Moreover, the interpolation in upsampling and downsampling might induce additional 
errors. The decrease in performance with growing scale is still apparent ($10-15\%$) but 
significantly reduced compared to the plain U-net $(55\%)$. 
To get more insight into performance and characteristics of the U-net, we add an 
experiment similar to the one from \cite{jansson22}: We train the U-net on crack widths $1$, 
$3$, and $5$ on the same number of images as for one single crack width. This case is referred to 
"U-net-mix scale adj." in Table~\ref{tab:comparison-fixed-scale}. 
Scales are adjusted similarly: w5 and w7 are downscaled by factor $2$, w9 and w11 are downscaled by factor $4$. The results are significantly better than those obtained by the U-net trained on the
single width ($10-15\%$ in Dice and IoU on unseen scales), but still remain worse than the
Riesz network trained on a single scale (around $7\%$ in Dice and IoU on unseen scales).

The Gaussian derivatives network is able to generalize steadily
across the scales (Dice and IoU $74\%$) but nevertheless performs worse than the scale 
adjusted U-net (around $10\%$ in IoU). Moreover, it is very sensitive to noise and typical CT 
imaging artifacts (Fig.~\ref{fig:test-fix-width-results}).

On the other hand, the Riesz network's performance is very steady with growing scale. 
We even observe improving performance in IoU and Dice with increase in crack thickness. 
This is due to pixels at the edge of the crack influencing the performance metrics less and 
less the thicker the crack gets. The Riesz network is unable to precisely localize cracks of 
width $1$ as, due to the partial volume effect, such thin cracks appear to be discontinuous. 
With the exception of the thinnest crack, the Riesz network has Dice coefficients above 
$94\%$ and IoU over $88\%$ for completely unseen scales. This even holds for the cases when the 
crack is more than $3$ times thicker than the one used for training.

\subsection{Experiment 2: Performance on multiscale data}
\label{sec:multiscale-data}

\begin{table*}[]
     \centering
     \hskip-0.75cm
     \begin{tabular}{|c|c c|c c|c c|c c|c c|c c|}
     \hline
     \multirow{2}{*}{Method} & \multicolumn{4}{|c|}{Multiscale cracks} \\
      \hhline{~|-|-|-|-|}
     & Precision & Recall & Dice & IoU \\
     \hline
     U-net, plain  & \underline{0.655} & 0.322 & 0.432 & 0.275 \\
     U-net pyramid 2 & 0.598 & 0.518  & 0.555 & 0.384 \\
     U-net pyramid 3 & 0.553 & 0.623 & \underline{0.586} & \underline{0.414} \\
     U-net pyramid 4 & 0.496 & \underline{0.705} & 0.582 & 0.411 \\
     \hline
     U-net-mix, plain  & 0.471 & 0.288 & 0.358 & 0.218 \\
     U-net-mix pyramid 2 & \underline{0.626} & 0.646  & 0.635 & 0.466\\
     U-net-mix pyramid 3 & 0.624 & 0.804 & 0.703 & 0.542 \\
     U-net-mix pyramid 4 & 0.583 & \underline{0.899} & \underline{0.707} & \underline{0.547} \\
     \hline
     Gaussian network 2 & \underline{0.553} & 0.503 & 0.527 & 0.358 \\
     Gaussian network 3 & 0.418 & 0.735 & \underline{0.533} & \underline{0.364} \\
     Gaussian network 4 & 0.306 & \underline{0.857} & 0.451 & 0.291\\
     \hline
     Riesz network & \textbf{0.901} & \textbf{0.902} & \textbf{0.902} & \textbf{0.821} \\
     \hline
    \end{tabular}
     \caption{Experiment 2. Performance on simulated multiscale cracks. The highest overall value is given in bold. For each competing method, the highest value is underlined. The scale sampling rate for the Gaussian network and U-net is set to $2$, while the number of sampled scales is in the set $\{1,2,3,4 \}$.}
     \label{tab:multiscale}
\end{table*}

Since cracks are naturally multiscale structures, i.e. crack thickness varies as the crack 
propagates, the performance of the considered methods on multiscale data is analyzed as
well. On the one hand, we want to test on data with an underlying ground truth without 
relying on manual annotation prone to errors and subjectivity. On the other hand, the 
experiment should be conducted in a more realistic and less 
controlled setting than the previous one, with cracks as similar as possible to real ones. 

We therefore use again simulated cracks, this time however with varying width. 
The thickness is modeled by an adaptive dilation.
See 
Fig.~\ref{fig:multiscale-test} for realization examples. 
The change rendering our experiment more realistic than the first one is to exploit no prior information about the scale. 
The Riesz network does not have to be adjusted for this experiment while the competing 
methods require scale selection as described in Section~\ref{sec:competing-methods}. 
Without knowing the scale, testing several configurations is the only option.
See Appendix \ref{secB} for examples. Note that in this experiment we used a different crack simulation technique \cite{jung22a} than in Experiment 1. In principle, we cannot claim that either of the two techniques generates more realistic cracks. 
However, this change serves as an additional goodness check for the methods since these simulation techniques can be seen as independent.  

We adjust the U-net as follows: We downscale the image by several factors from 
$\{2,4,8,16... \}$. The forward pass of the U-net is applied to the original and every 
downscaled image. Subsequently, the downscaled images are upscaled back to the original 
size. All predictions are joined by the maximum operator. 
We report results for several downscaling factor combinations specified by a number $N$, 
which is the number of consecutive downscaling factors used, starting at the smallest factor $2$. Similarly as in Experiment 1, we report results of two U-net models: the first model is trained on cracks of width 3 as the other models in the comparison. The second model is trained on cracks with mixed widths. Including more crack widths in the training set has proven to improve the scale generalization ability in Experiment 1. Hence, the second model represents a more realistic setting that would be used in practice where the crack width is typically unknown. We denote the respective networks as "U-net pyramid" $N$ and "U-net-mix pyramid" $N$.

For the Gaussian network, we vary the standard deviation parameter $\sigma$ in the set
$\{1.5,3,6,12\}$. This selection of scales is motivated by the network having been trained 
on crack width $3$ with $\sigma = 1.5$. We start with the original $\sigma$ and double it in each step. Note that in the related study on the scale sampling for scale-channel deep networks \cite{jansson22}, using a scale sampling factor of $\sqrt{2}$ was found to lead to better performance than using a scale sampling factor of 2.
Hence, additional experimentation with hyperparameters of the method might improve the results. We decided to keep the sampling scheme as similar as possible to the one from U-net. The reason is that U-net with downscaling was more extensively tested on the crack segmentation task \cite{jung22,jung23ict}.
As for the U-net, we test several configurations, now specified by the 
number $N$ of consecutive $\sigma$ values used, starting at the smallest 
($1.5$). We denote the respective network "Gaussian network" $N$.

Results are reported in Table~\ref{tab:multiscale} and Fig.~\ref{fig:multiscale-test}. 
We observe a clear weakness of the Riesz network in segmenting thin 
cracks (Fig.~\ref{fig:multiscale-test}, first and last row). Despite of this, the recall 
is still quite high ($90\%$). However, this could be due to thicker cracks - which are handled very well - contributing 
stronger to these statistics as they occupy more pixels. 
Nevertheless, the Riesz network deals with the problem of the wide range scales 
without sampling the scale dimension, 
just with a single forward pass of the network. 

The performance of the U-net improves with including more levels in the pyramid, too. 
However, this applies only up to a certain number of levels after which the additional 
gain becomes minimal. Moreover, applying the U-net on downscaled images seems to induce
oversegmentation of the cracks (Fig.~\ref{fig:multiscale-test}, second and third row). Including a variety of crack widths in the training set improves the overall performance of U-net in all metrics. This confirms the hypothesis that U-net significantly benefits from variations in the training set. However, this model of U-net is still outperformed by the Riesz network trained on a single crack width.
The Gaussian network behaves similarly as the U-net, with slightly worse performance (according to Dice or IoU) but better crack coverage (Recall). As the number of $\sigma$ values grows, the recall increases
but at the same time artifacts accumulate across scales reducing precision. 
The best balance on this data set is found to be three scales. 

\begin{figure*}
    \centering
    \includegraphics[width = 0.24\textwidth]{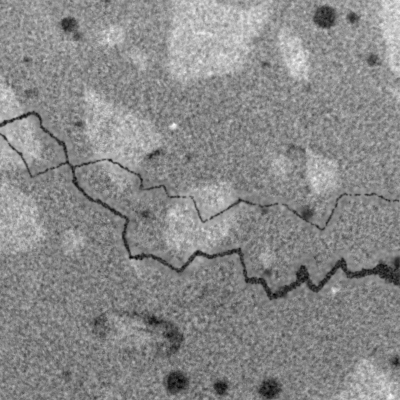}
    \includegraphics[width = 0.24\textwidth]{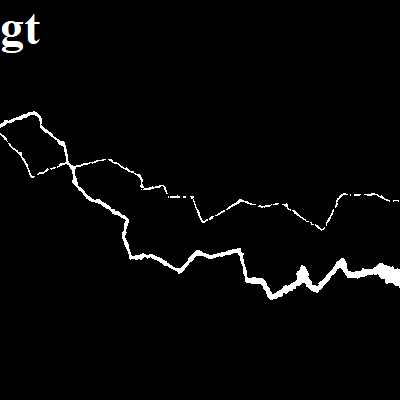}
    \includegraphics[width = 0.24\textwidth]{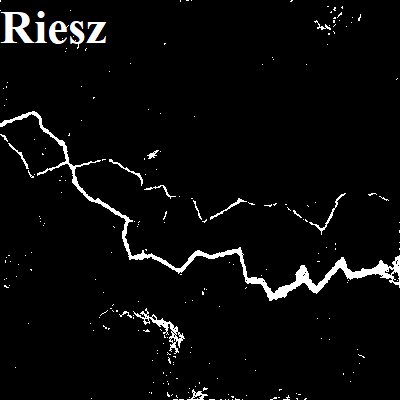}
    \includegraphics[width = 0.24\textwidth]{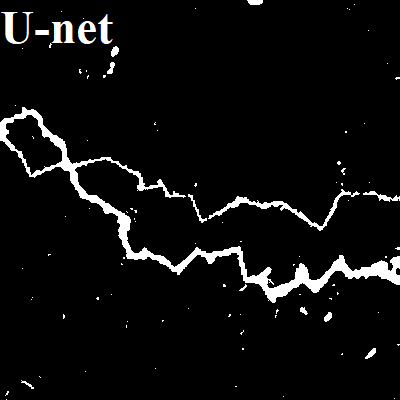}

    \includegraphics[width = 0.24\textwidth]{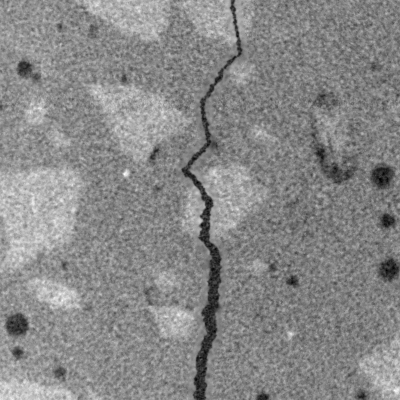}
    \includegraphics[width = 0.24\textwidth]{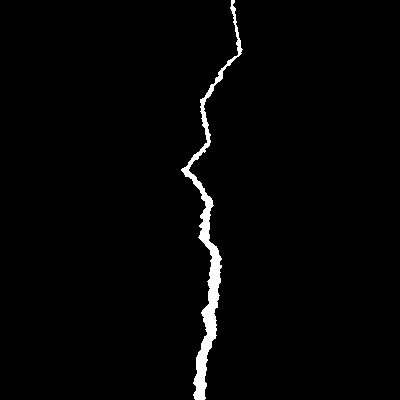}
    \includegraphics[width = 0.24\textwidth]{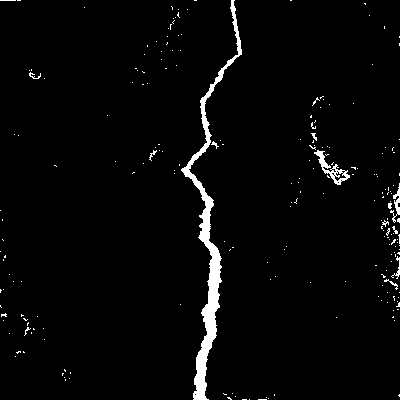}
    \includegraphics[width = 0.24\textwidth]{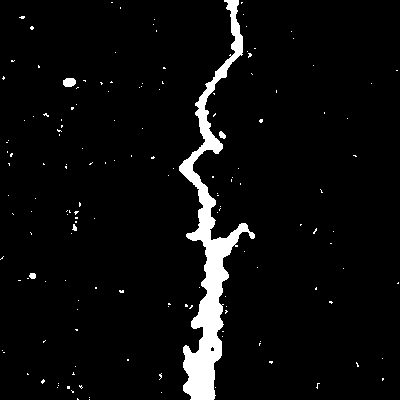}

    \includegraphics[width = 0.24\textwidth]{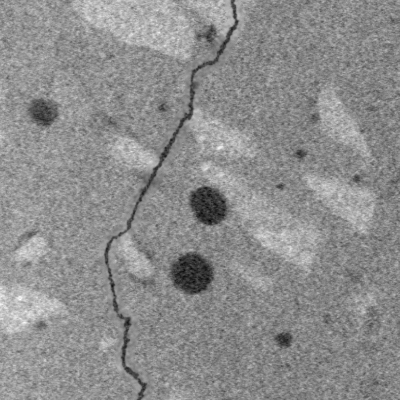}
    \includegraphics[width = 0.24\textwidth]{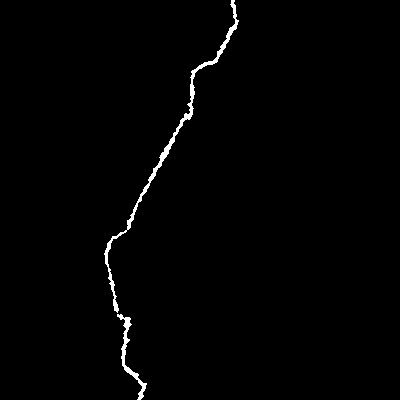}
    \includegraphics[width = 0.24\textwidth]{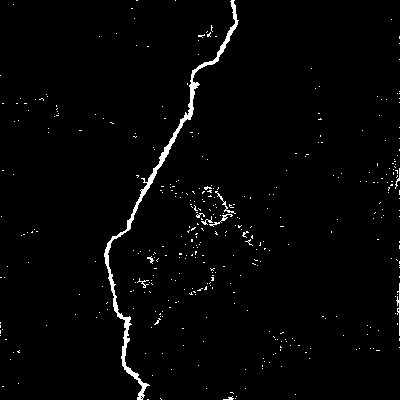}
    \includegraphics[width = 0.24\textwidth]{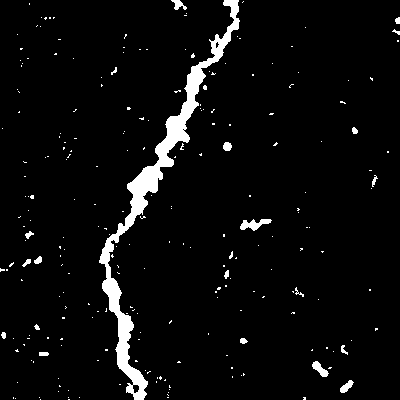}

    \includegraphics[width = 0.24\textwidth]{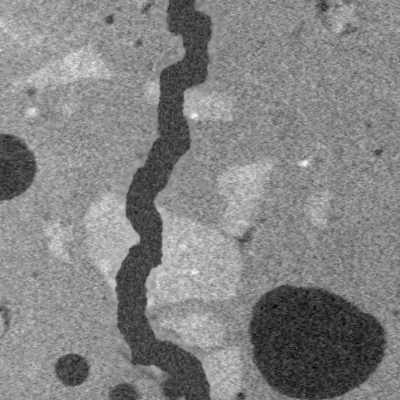}
    \includegraphics[width = 0.24\textwidth]{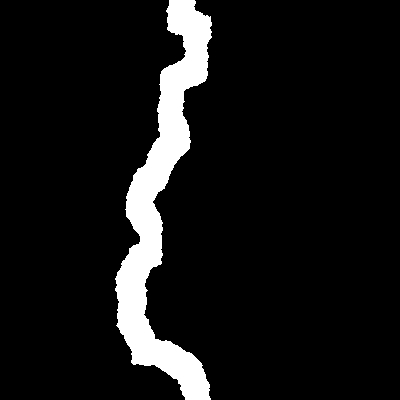}
    \includegraphics[width = 0.24\textwidth]{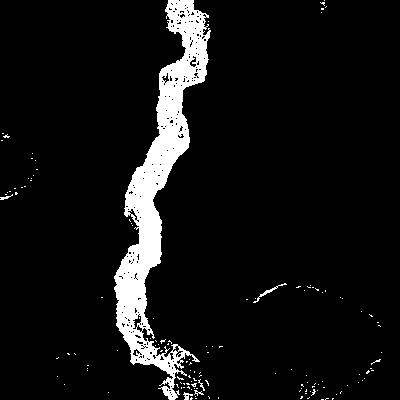}
    \includegraphics[width = 0.24\textwidth]{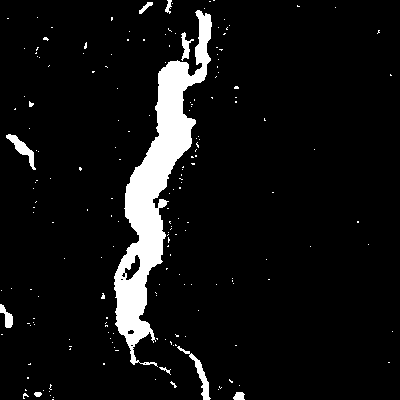}

    \includegraphics[width = 0.24\textwidth]{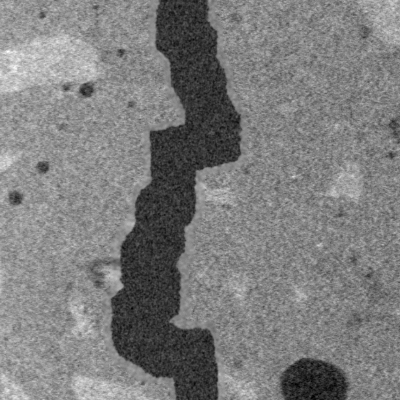}
    \includegraphics[width = 0.24\textwidth]{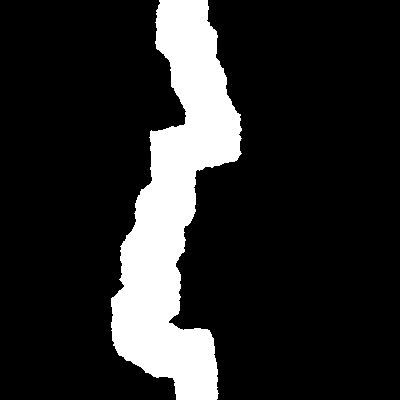}
    \includegraphics[width = 0.24\textwidth]{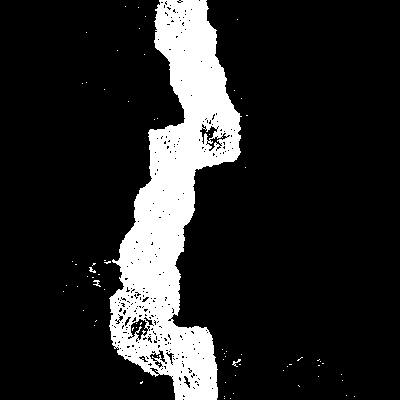}
    \includegraphics[width = 0.24\textwidth]{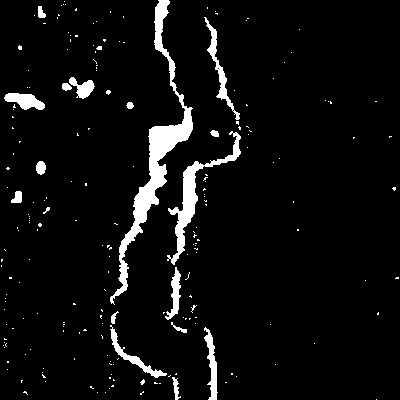}
    \\
    \includegraphics[width = 0.24\textwidth]{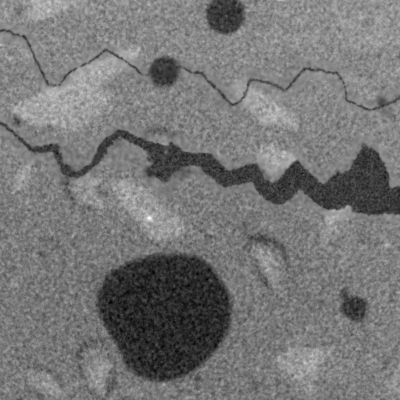}
    \includegraphics[width = 0.24\textwidth]{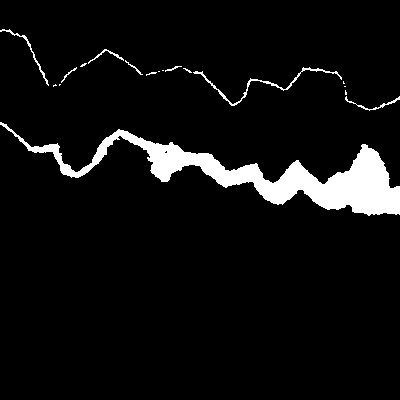}
    \includegraphics[width = 0.24\textwidth]{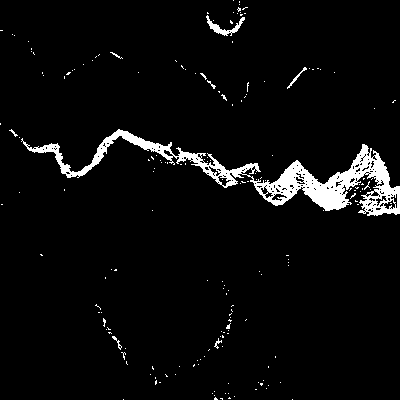}
    \includegraphics[width = 0.24\textwidth]{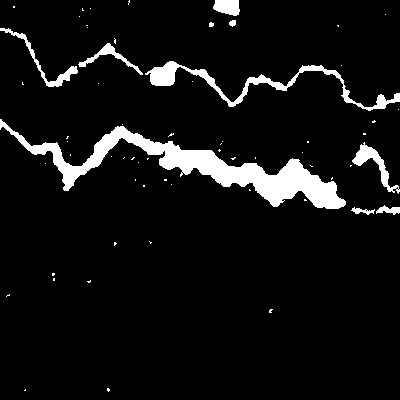}
    \caption{Experiment 2. Cracks with varying width. 
From left to right: input image, results of the Riesz network and the U-net with 4 pyramid levels. Image size $400 \times 400$ pixels.}
    \label{fig:multiscale-test}
\end{figure*}

\begin{figure*}
    \centering
    \includegraphics[width = 0.32\textwidth]{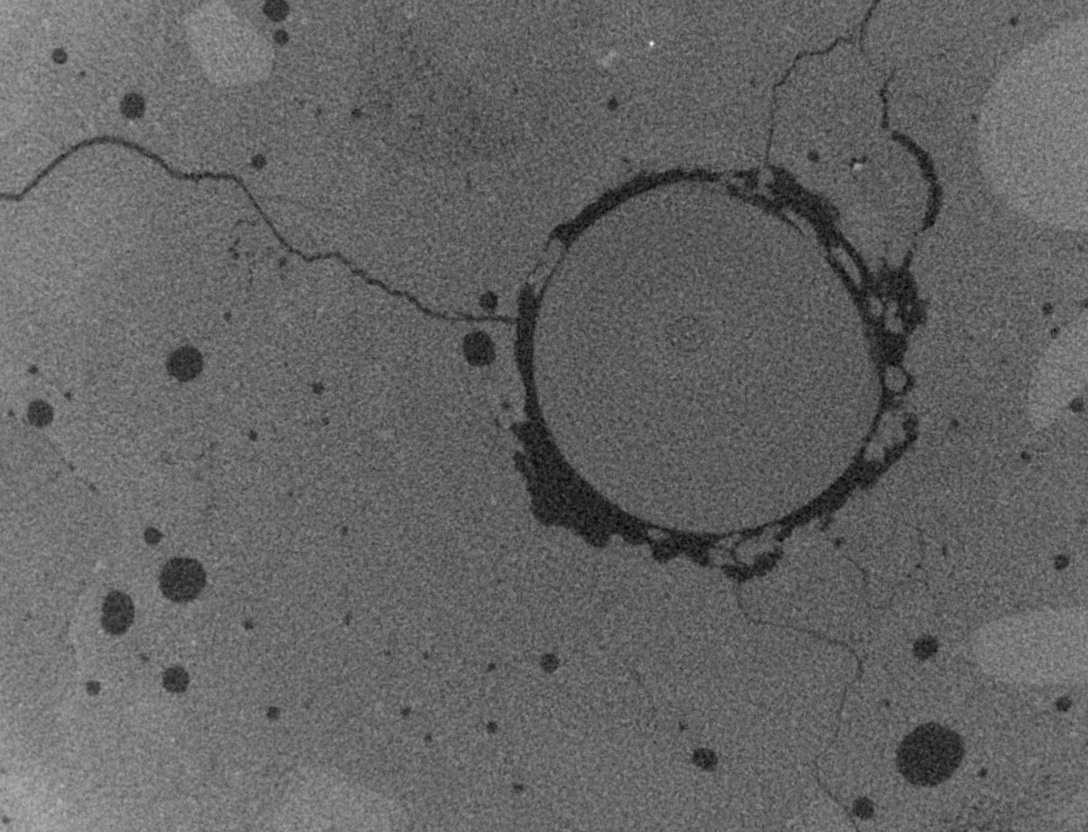}
    \includegraphics[width = 0.32\textwidth]{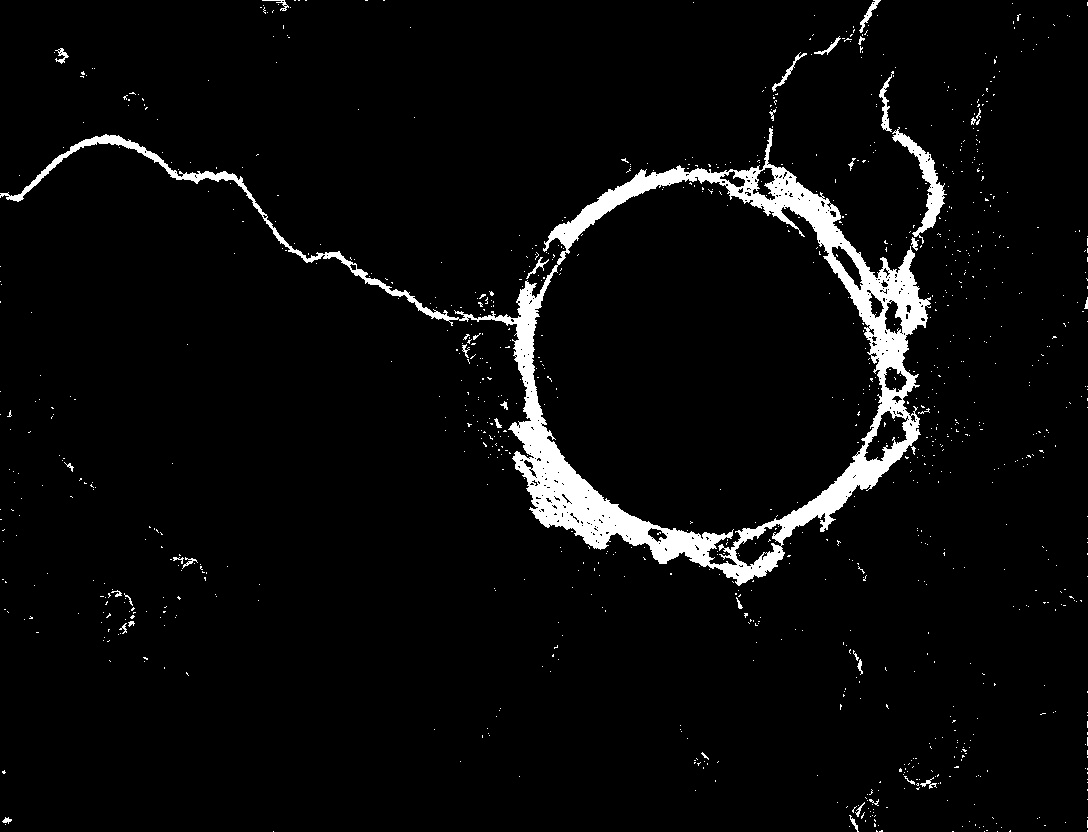}
    \includegraphics[width = 0.32\textwidth]{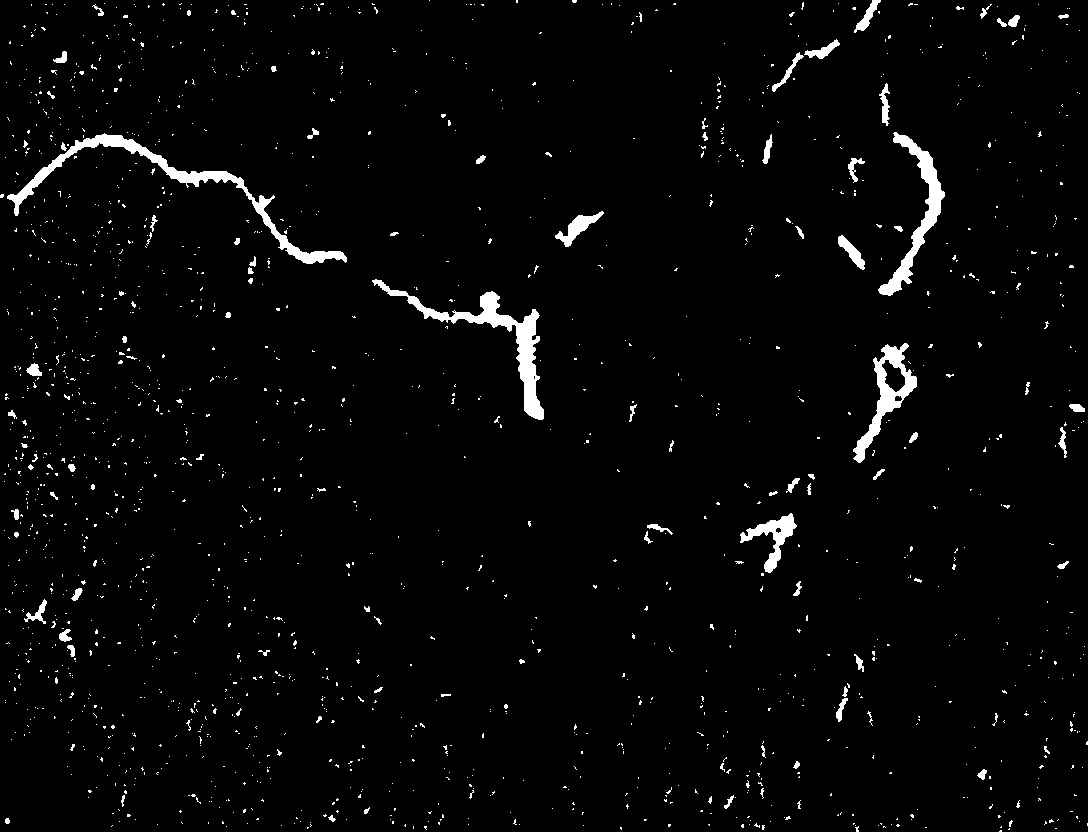}
    
    \includegraphics[width = 0.32\textwidth]{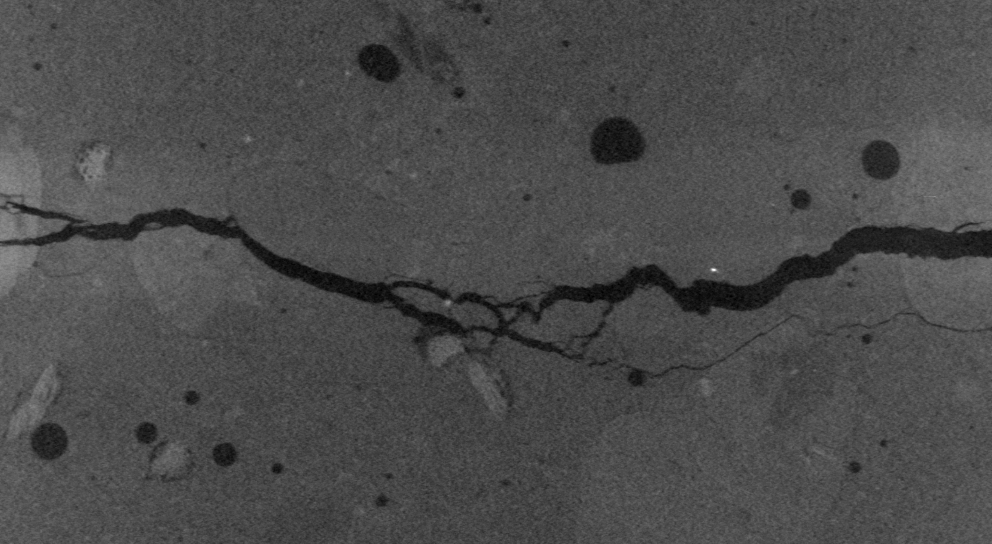}
    \includegraphics[width = 0.32\textwidth]{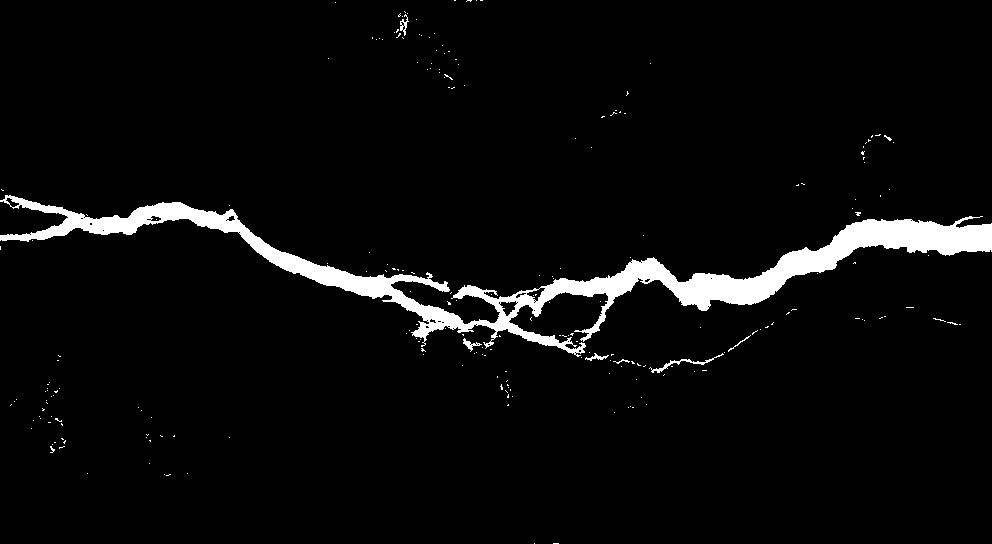}
    \includegraphics[width = 0.32\textwidth]{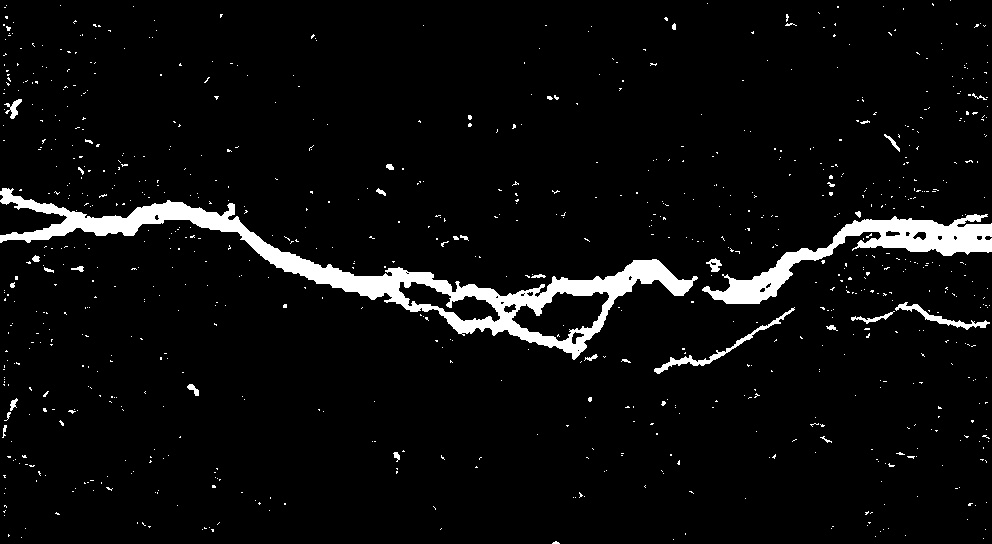}
    \caption{Experiment 3. Real cracks in concrete: slice from input CT image, results of the Riesz network and of U-net with 2 pyramid levels. 
    Image sizes are $832 \times 1\,088$ (1st row) and $544 \times 992$ (2nd row).}
    \label{fig:real-data}
\end{figure*}

\subsection{Experiment 3: Application to cracks in CT images of concrete}

Finally, we check the methods' performance on real data: 
cracks in concrete samples generated by tensile and pull-out tests. In these examples, the crack thickness 
varies from $1$ or $2$ pixels to more than $20$ pixels (Fig.~\ref{fig:real-data}). 
This 
motivates the need for methods that automatically generalize to completely unseen scales. 
Here, we can assess the segmentation results qualitatively, only, as no ground truth is
available. Manual segmentation of cracks in high resolution images is time consuming and 
prone to individual biases.
Additional experiments on real cracks in the different types of concrete are shown in Appendix \ref{secC}.

The first sample (Fig.~\ref{fig:real-data}, first row) is a concrete cylinder with a glass fiber reinforced composite bar
embedded along the center line. A force is applied to this bar to pull it out of the 
sample and thus initiate cracking. 
Cracking starts around the bar and branches in three 
directions: left, right diagonal, and down (very subtle, thin crack). Crack thicknesses and
thus scales vary depending 
on the crack branch. 
As before, our Riesz network is able to handle all but the finest crack thicknesses 
efficiently in a single forward pass without specifying the scale range. The U-net on the 
image pyramid requires a selection of 
downsampling steps (Appendix \ref{secB}), accumulates artifacts from all levels of the pyramid, and slightly 
oversegments thin cracks (left branch).

The second sample (Fig.~\ref{fig:real-data}, second row) features a horizontal crack induced by a tensile test. Here we observe 
permanently changing scales, similar to our simulated multiscale data. The crack thickness 
varies from a few to more than $20$ pixels. Once more, the Riesz network handles the scale variation well and segments almost all cracks with minimal artifacts. 
In this example, U-net covers the cracks well, too, even the very subtle ones. 
However, it accumulates more false positives in the areas of concrete without any cracks than the Riesz network.

\section{Conclusion}
In this paper we introduced a new type of scale invariant neural network based on the Riesz transform as filter basis instead of standard convolutions. 
Our Riesz neural network is scale invariant in one forward pass without specifying 
scales or discretizing and sampling the scale dimension. 
Its ability to generalize to scales differing from those trained on is tested and 
validated in segmenting cracks in 2d slices from CT images of concrete. 
Usefulness 
of the method become manifest in the fact that only one fixed scale is needed for training, while preserving generalization to completely unseen scales. This reduces the effort for data collection, generation or simulation. 
Furthermore, our network has relatively few parameters (around $18$k) which reduces the danger of overfitting.

Experiments on simulated yet realistic multiscale cracks as well as on real cracks 
corroborate the Riesz network's potential. Compared to other deep learning methods 
that can generalize to unseen scales, the Riesz network yields improved, more robust, and more stable results.

A detailed ablation study on the network parameters reveals several interesting features:
This type of networks requires relatively few data to generalize well. The 
Riesz network proves to perform well on a data set of approximately $200$ images before 
augmentation. This is particularly useful for deep learning tasks where data acquisition is exceptionally complex or expensive. 
The performance based on the depth of the network and the number of parameters has been analyzed. Only three layers of the
network suffice to achieve good performance on cracks in 2d slices of CT images. Furthermore, the choice of crack thickness in the training set is found to be not 
decisive for the performance. Training sets with crack widths $3$ and $5$ yield very similar results.

The two main weaknesses of our approach in the crack segmentation task are 
undersegmentation of thin cracks and edge effects around pores. 
In CT images, thin cracks appear brighter than thicker cracks due to the partial volume effect reducing the contrast between the crack and concrete. For the same reason thin cracks look discontinued. 
Thin cracks might therefore require special treatment. 
In some situations, pore edge regions get erroneously segmented as crack. 
These can however be removed by a post-processing step and are no serious problem.

To unlock the full potential of the Riesz transform, validation on other types of 
problems is needed. 
Furthermore, scaling the Riesz network to larger, wider, and deeper models remains an open topic.
Our study as well as previous ones \cite{jacobsen16, luan2018gabor, penaud2022fully} imply that small models based on linear combination of the convolutions with fixed filters could yields results comparable to those of large CNN models. 
However, in order to state this reliably, training convergence, expressiveness, and run-times of these types of networks have to be compared systematically to those of CNNs.

In the future, the method should be applied in 3d since CT data 
is originally 3d. In this case, memory issues might occur during discretization of the Riesz 
kernel in frequency space.

An interesting topic for further research is to join translation and scale invariance  with rotation invariance to design a new generation of neural networks with encoded basic computer vision properties \cite{sifre13}. This type of neural network could be very efficient because it would have even less parameters and hence would require less training data, too.

\bmhead{Acknowledgments}
We thank Christian Jung (RPTU) for generating the multiscale crack images.



\section*{Appendix}

\appendix

\section{Experiment on MNIST Large~Scale Dataset}
\label{secA1}
We test the Riesz networks on a classification task on the  MNIST Large Scale \cite{jansson22} to test wider applicability of Riesz networks outside of crack segmentation task. This data set was derived from the MNIST data set \cite{lecun98} and it consists of images of digits between 0 and 9 belonging to one of ten classes (Fig. \ref{fig:mnist-classes2}) which are rescaled to a wide range of scales to test scale generalization abilities of neural networks (Fig. \ref{fig:mnist-scales}).

 \begin{figure*}[h]
    \centering
    \includegraphics[width = 0.18 \textwidth]{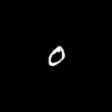}
    \includegraphics[width = 0.18 \textwidth]{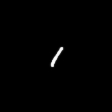}
    \includegraphics[width = 0.18 \textwidth]{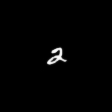}
    \includegraphics[width = 0.18 \textwidth]{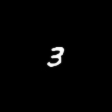}
    \includegraphics[width = 0.18 \textwidth]{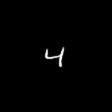}
    \includegraphics[width = 0.18 \textwidth]{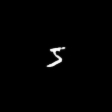}
    \includegraphics[width = 0.18 \textwidth]{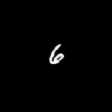}
    \includegraphics[width = 0.18 \textwidth]{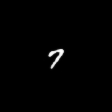}
    \includegraphics[width = 0.18 \textwidth]{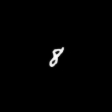}
    \includegraphics[width = 0.18 \textwidth]{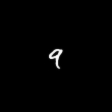}
    \caption{10 classes in MNIST Large Scale data set. All images have size $112 \times 112$.}
    \label{fig:mnist-classes2}
\end{figure*}

\begin{figure*}[h]
    \centering
    \includegraphics[width = 0.18 \textwidth]{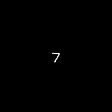}
    \includegraphics[width = 0.18 \textwidth]{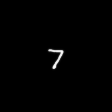}
    \includegraphics[width = 0.18 \textwidth]{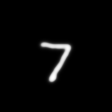}
    \includegraphics[width = 0.18 \textwidth]{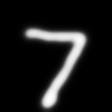}
    \includegraphics[width = 0.18 \textwidth]{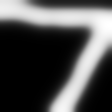}
    \caption{Variation of scales in MNIST Large Scale data set (from left to right): scales 0.5, 1, 2, 4 and 8. All images have size $112 \times 112$.}
    \label{fig:mnist-scales}
\end{figure*}

Our Riesz network has the channel structure 12-16-24-32-80-10 with the softmax function at the end. In total, it has 20,882 parameters. Following \cite{lindeberg21}, only the central pixel in the image is used for classification. We use the standard CNN described in \cite{jansson22} but without any scale adjustments as a baseline to illustrate limited scale generalization property. This CNN has the channel structure 16-16-32-32-100-10 with the softmax function at the end and in total 574,278 parameters.
The training set has 50,000 images of the single scale $1$. We used a validation set of 1,000 images. The test set consists of scales ranging in $\left[0.5,8\right]$ with 10,000 images per scale. All images have size $112\times 112$.
Models are trained using the ADAM optimizer \cite{kingma14} with default parameters for 20 epochs with learning rate 0.001 which is halved every 3 epochs. Cross-entropy is used as loss function.
\begin{figure}[h]
    \centering
    \includegraphics[width = 0.49 \textwidth]{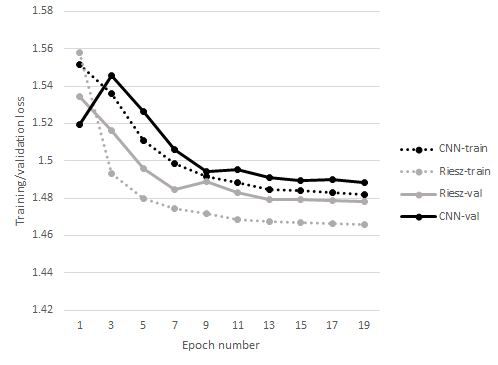}
    \caption{Train and validation loss for Riesz network and CNN (as a baseline).}
    \label{fig:mnist-scale-training}
\end{figure}

\begin{table*}[]
    \begin{tabular}{|c|c|c|c|c|c|c|c|c|c|} 
    \hline
          \multicolumn{2}{|c|}{scale} & 0.5 & 0.595 & 0.707 & 0.841 & 1 & 1.189 & 1.414 & 1.682 
         \\
         \hline
         \multicolumn{2}{|c|}{CNN} & 40.74 & 64.49 & 88.35 & 96.87 & 97.77 & 96.08 & 80.06 & 38.68   
         \\
         \hline
         \multicolumn{2}{|c|}{Riesz} & 96.34 & 97.59 & 98.06 & 98.54 &  98.58 & 98.50 & 98.45  & 98.40 
         \\
         \multicolumn{2}{|c|}{Riesz-pad20} & 96.33 & 97.57 & 98.07 & 98.48 & 98.63 & 98.54 & 98.49 & 98.46 
         \\
         \multicolumn{2}{|c|}{Riesz-pad40} & 96.34 & 97.55 & 98.07 & 98.47 & 98.63  & 98.58 & 98.53  & 98.44 
         \\
         \hline
         \multicolumn{2}{|c|}{FovAvg 17ch tr1 \cite{jansson22}} & 98.58 & 99.05 & \textbf{99.33} & \textbf{99.39} & \textbf{99.40} & \textbf{99.39} & \textbf{99.38} & \textbf{99.36} 
         \\
         \multicolumn{2}{|c|}{FovMax 17ch tr1 \cite{jansson22}} & \textbf{98.71} & \textbf{99.07} & 99.27 & 99.34 & 99.37 & 99.35 & 99.36 & 99.34 
         \\
         \hline
        \multicolumn{10}{c}{} \\
        \hline
         scale & 2 & 2.378 & 2.828 & 3.364 & 4 & 4.757 & 5.657 & 6.727 & 8
         \\
         \hline
         CNN & 25.90 & 24.91 & 23.64 & 21.34 & 19.91 & 18.87 & 18.04 & 15.64 & 11.79 
         \\
         \hline
         Riesz & 98.39 & 98.24 & 98.01 & 97.51 & 96.42 & 93.5 & 81.58 & 67.66 & 51.82 
         \\
         Riesz-pad20 & 98.39 & 98.35 & 98.33 & 98.16 & 97.78 &  97.08 & 95.48 & 91.10 & 79.78 
         \\
         Riesz-pad40 &  98.46 &  98.39 & 98.34 & 98.29 & 98.16 & 97.80  & 96.82 &  \textbf{93.75} &   \textbf{83.6} 
         \\
         \hline
         FovAvg 17ch tr1 \cite{jansson22} & \textbf{99.35} & 99.31 & 99.22 & 99.12 & 98.94 & 98.47 & 96.20 & 89.17 & 71.31
         \\
         FovMax 17ch tr1 \cite{jansson22}  & 99.33 & \textbf{99.35} & \textbf{99.34} & \textbf{99.35} & \textbf{99.34} & \textbf{99.27} & \textbf{97.88} & 92.76 & 79.23 
         \\
    \hline
    \end{tabular}
    \caption{Classification accuracy (in \%) of  MNIST Large Scale data set. Best performing method bold.} 
    \label{tab:mnist-scale-generalization}
\end{table*}

Fig. \ref{fig:mnist-scale-training} shows validation and training loss during 20 epochs. Interestingly, the Riesz network converges faster and even its validation loss remains lower than the training loss of CNN.
Accuracies for the different scales are shown in 
Table~\ref{tab:mnist-scale-generalization}.

The Riesz network shows stable accuracy for scales in the range $\left[0.5,4 \right]$. The CNN, which has way more degrees of freedom, is only competitive for scales close to the training scale. 
Results for 
two scale adjusted versions of the CNN as reported in \cite{jansson22} are also given in Table \ref{tab:mnist-scale-generalization}. Their performance is slightly superior to the Riesz network (around $ 1-2\%$). However, it is important to note that this approach uses (max or average) pooling over $17$ scales.

Further works considering the MNIST Large Scale data set are \cite{sangalli21,lindeberg21}. Unfortunately, no numeric values of the accuracies are provided, so we can compare the results only qualitatively. The Riesz network's accuracy varies less on a larger range of scales than those of the scale-equivariant networks on Gaussian or morphological scale spaces from \cite{sangalli21} that were trained on scale $2$. The Gaussian derivative network \cite{lindeberg21} trained on scale 1 yields results in a range between $98\%$ and $99\%$ for medium scales $\left[0.7,4.7\right]$ using pooling over $8$ scales. The Riesz network yields similar values but without the need for scale selection. 

On the smallest scale of 0.5, the Riesz network seems to give a better result than \cite{lindeberg21}, while it is outperformed on the largest scales.
The reason for the latter is that digits start to reach the boundary of the image. To reduce that effect, we pad the images by 20 and 40 pixels with the minimal gray value.
Indeed, this improves the accuracy significantly for larger scales (Table~\ref{tab:mnist-scale-generalization}), while it remains equal for the rest of the scales. For example, for scale $8$, accuracy increases from $51.8\%$ to $79.8\%$ (padding 20) and $83.6\%$ (padding 40). This is a better accuracy than that reported in \cite{lindeberg21} and \cite{jansson22} for models trained on scale 1.


\section{Experiments on scale selection for competing methods related to Riesz network}
\label{secB}


The largest benefit of the Riesz network is avoiding the sampling of the scale dimension. 
Here, we give more detailed insight into scale sampling in practice for competing methods: U-net applied on rescaled images and Gaussian derivative networks. 
We show how segmentation results change as we add additional scales to the output. 
As we add new scales, cracks that belong (or are close) to the added scales get segmented. However, additional noise gets segmented, too. These noise pixels that are misclassified as cracks originate from two sources: interpolation error and high frequency noise.
For simulated data this is shown in 
Fig. \ref{fig:multiscale-test-appendix5} and Fig. \ref{fig:multiscale-test-appendix6}. 
For real cracks see Fig.~\ref{fig:real-data-appendix}.

The main drawback is that one needs to select the range of scales on which to apply these methods. Since the scale dimension in the images is bounded from above by the size of the  view window, when having images of different sizes scale sampling needs to be adjusted or recalibrated. It is not trivial how to achieve this in a general manner. In contrast, the Riesz transform enables simultaneous, continuous, and equal treatment of all scales automatically adapting to the image size.

\begin{figure*}[h]
    \centering
    \includegraphics[width = 0.24\textwidth]{figures/fresh-figures/multiscale/input-multiscale-61.jpg}
    \includegraphics[width = 0.24\textwidth]{figures/fresh-figures/multiscale/gt-multiscale-61.jpg}
    \\
    \includegraphics[width = 0.24\textwidth]{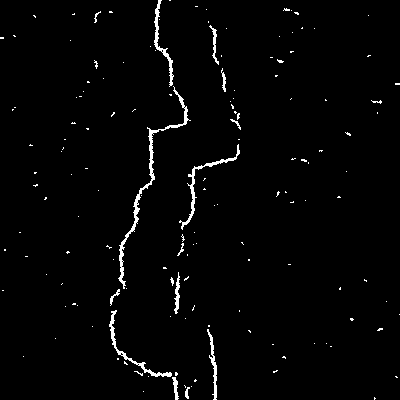}
    \includegraphics[width = 0.24\textwidth]{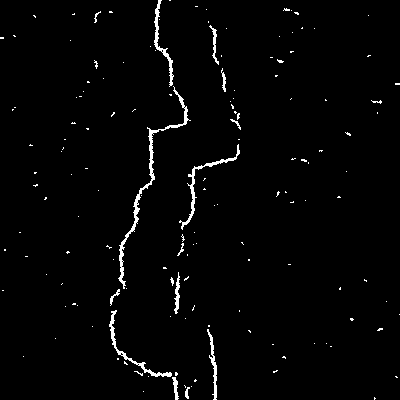}
    \includegraphics[width = 0.24\textwidth]{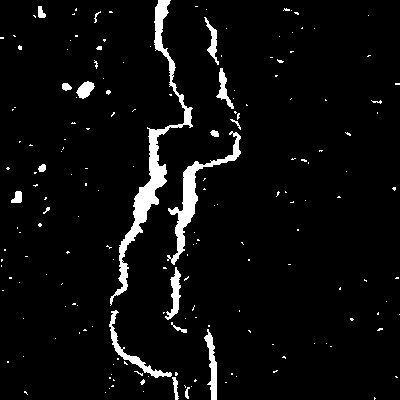}
    \includegraphics[width = 0.24\textwidth]{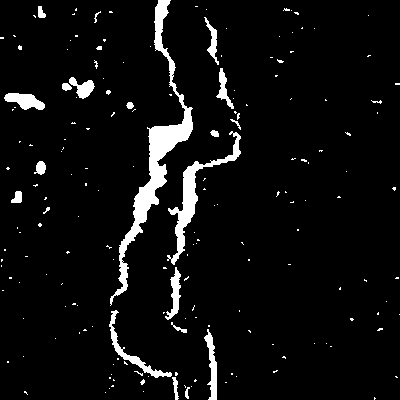}
    
    \includegraphics[width = 0.24\textwidth]{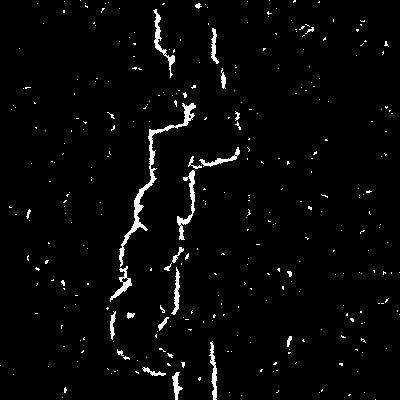}
    \includegraphics[width = 0.24\textwidth]{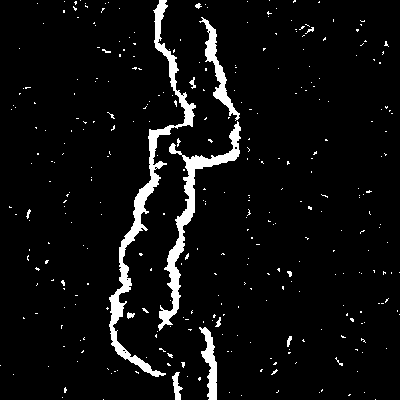}
    \includegraphics[width = 0.24\textwidth]{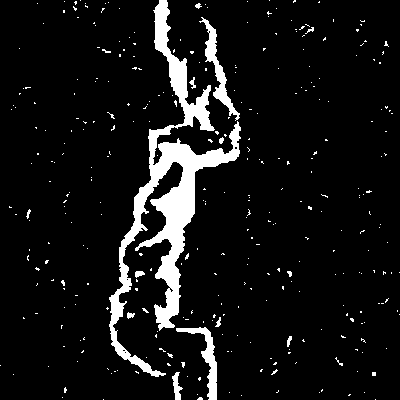}
    \includegraphics[width = 0.24\textwidth]{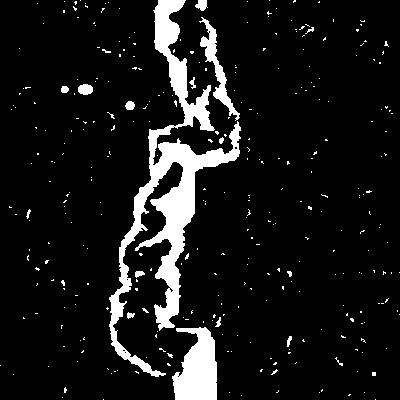}
    
    \includegraphics[width = 0.24\textwidth]{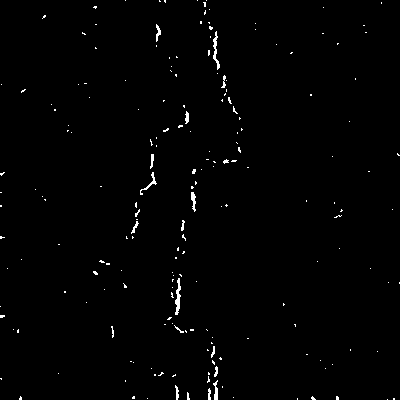}
    \includegraphics[width = 0.24\textwidth]{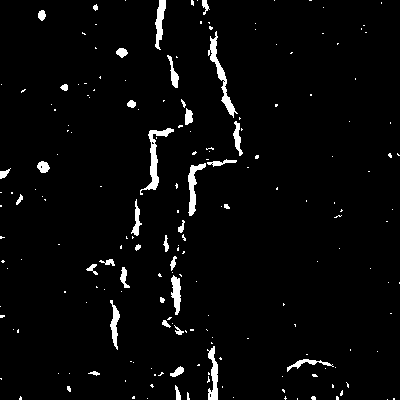}
    \includegraphics[width = 0.24\textwidth]{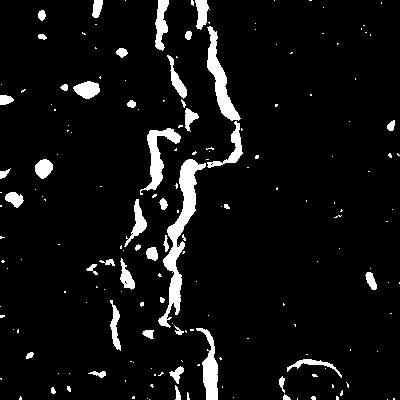}
    \includegraphics[width = 0.24\textwidth]{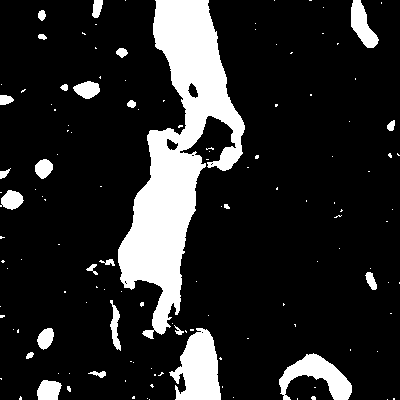}
    \caption{Experiment 2. Cracks with varying width.
    First row: input image and ground truth image. Second row: U-net applied to several levels of pyramids $\{1,2,3,4\}$ (from left to right). Third row: U-net-mix applied to several levels of pyramids $\{1,2,3,4\}$ (from left to right).
    Fourth row: Gaussian derivative networks aggregated on growing subsets of scale set $\{1.5, 3, 6, 12\}$ (from left to right). Image size $400 \times 400$ pixels.}
    \label{fig:multiscale-test-appendix5}
    \end{figure*}

    \begin{figure*}[h]
    \centering
    \includegraphics[width = 0.24\textwidth]{figures/fresh-figures/multiscale/input-multiscale-82.jpg}
    \includegraphics[width = 0.24\textwidth]{figures/fresh-figures/multiscale/gt-multiscale-82.png}
    \\
    \includegraphics[width = 0.24\textwidth]{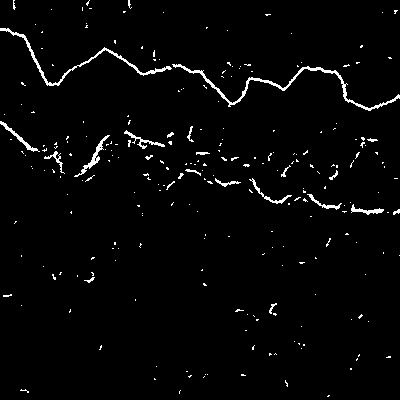}
    \includegraphics[width = 0.24\textwidth]{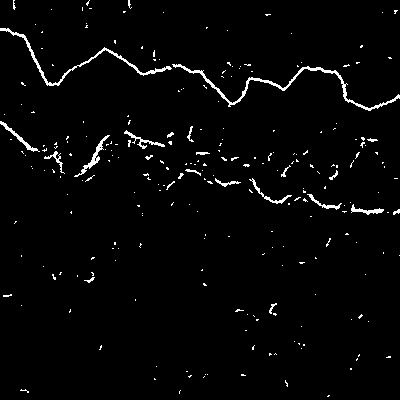}
    \includegraphics[width = 0.24\textwidth]{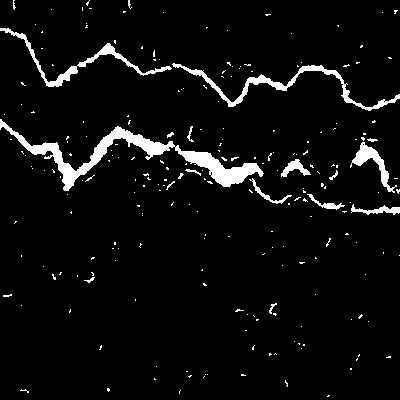}
    \includegraphics[width = 0.24\textwidth]{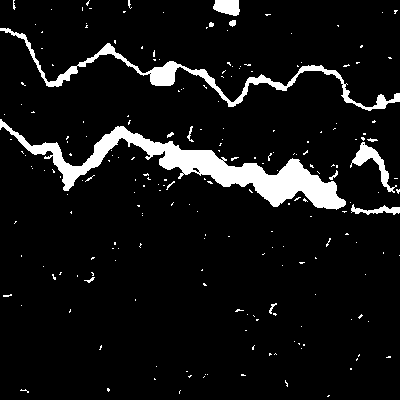}
    
    \includegraphics[width = 0.24\textwidth]{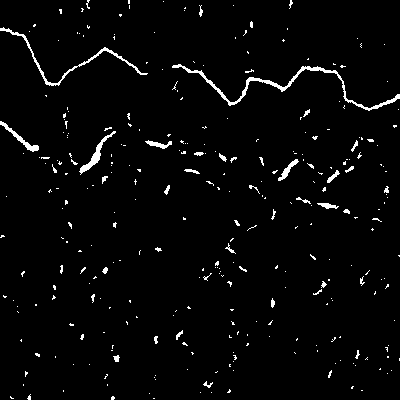}
    \includegraphics[width = 0.24\textwidth]{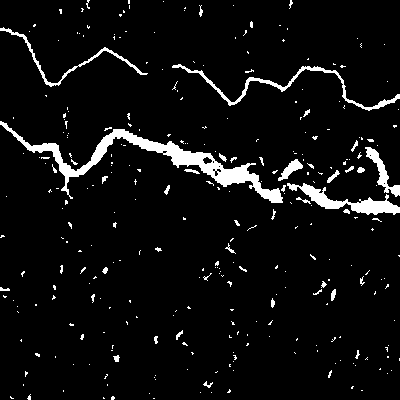}
    \includegraphics[width = 0.24\textwidth]{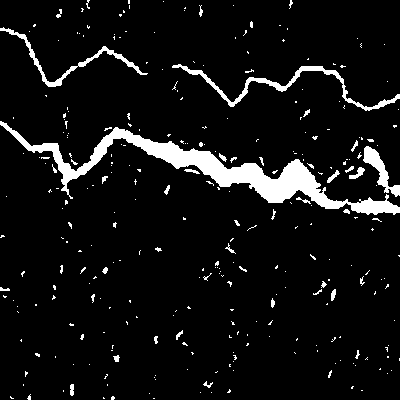}
    \includegraphics[width = 0.24\textwidth]{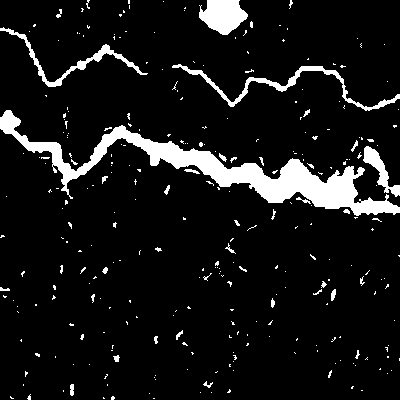}
    
    \includegraphics[width = 0.24\textwidth]{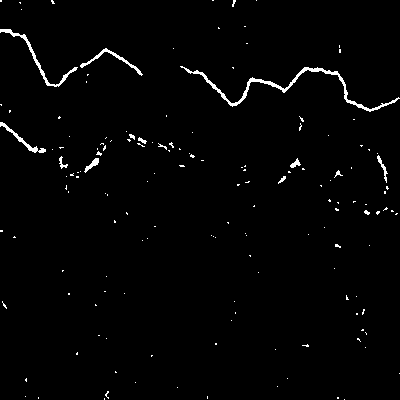}
    \includegraphics[width = 0.24\textwidth]{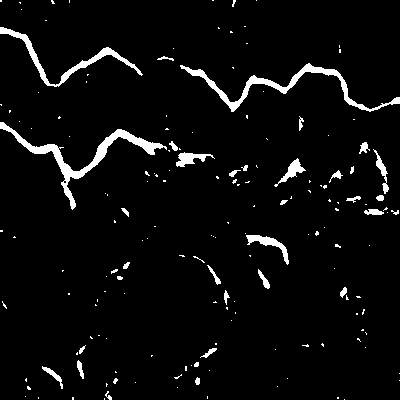}
    \includegraphics[width = 0.24\textwidth]{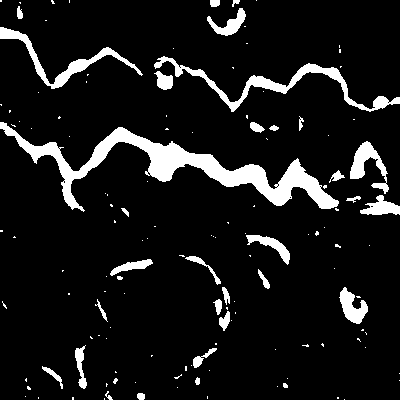}
    \includegraphics[width = 0.24\textwidth]{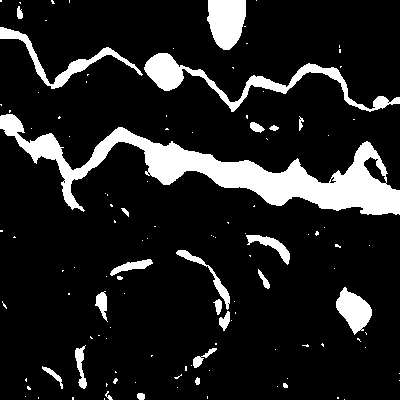}
    \caption{Experiment 2. Cracks with varying width.
    First row: input image and ground truth image. Second row: U-net applied to several levels of pyramids $\{1,2,3,4\}$ (from left to right). Third row: U-net-mix applied to several levels of pyramids $\{1,2,3,4\}$ (from left to right).
    Fourth row: Gaussian derivative networks aggregated on growing subsets of scale set $\{1.5, 3, 6, 12\}$ (from left to right). Image size $400 \times 400$ pixels.}
    \label{fig:multiscale-test-appendix6}
\end{figure*}

\begin{figure*}[h]
    \centering
    \includegraphics[width = 0.24\textwidth]{figures/real/real-ex2-crop.jpg}
    \includegraphics[width = 0.24\textwidth]{figures/fresh-figures/real/riesz-ex2-real.jpg}
    \\
    \includegraphics[width = 0.24\textwidth]{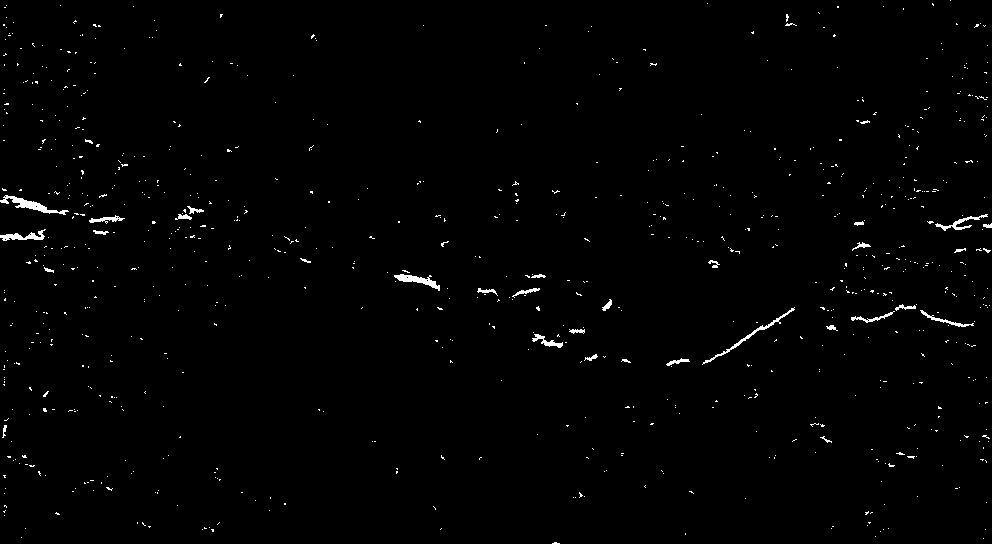}
    \includegraphics[width = 0.24\textwidth]{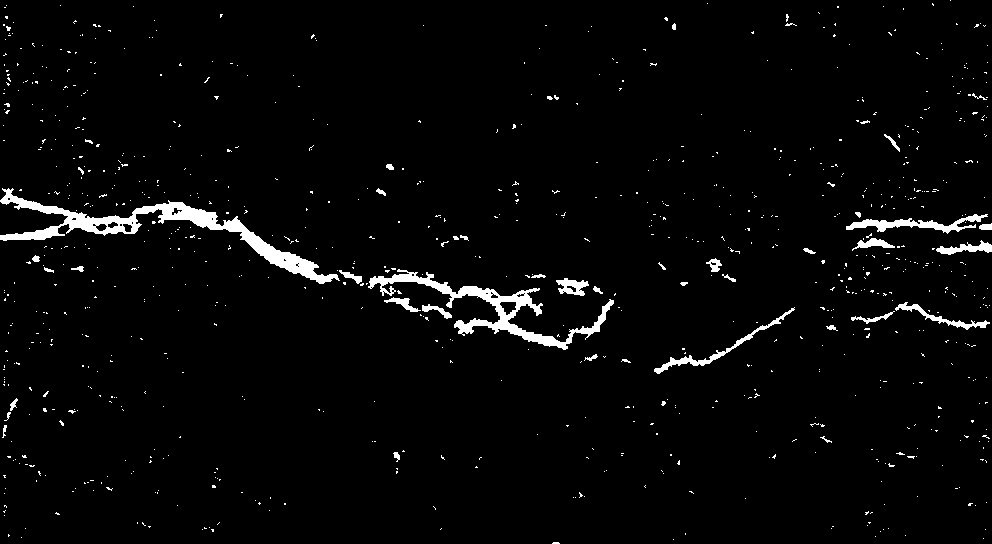}
    \includegraphics[width = 0.24\textwidth]{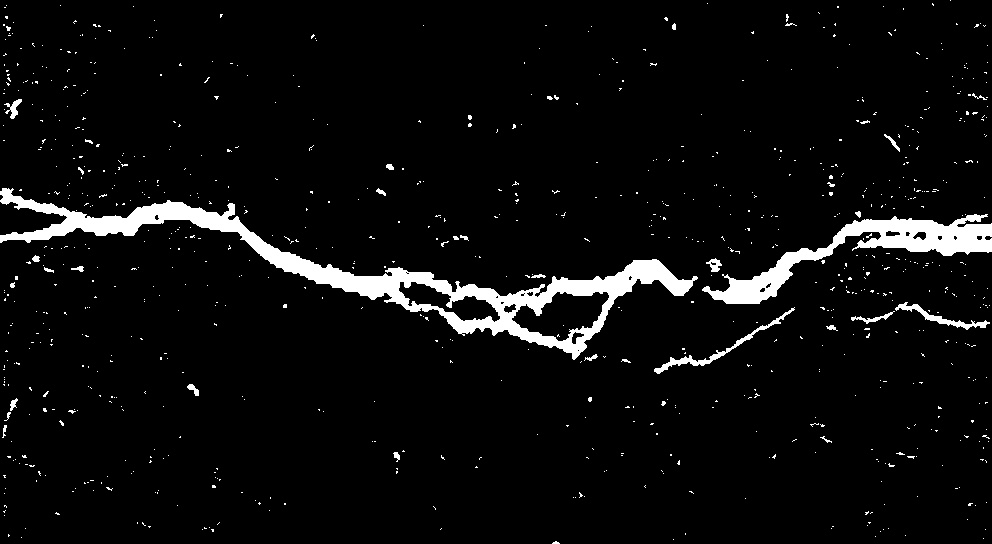}
    \includegraphics[width = 0.24\textwidth]{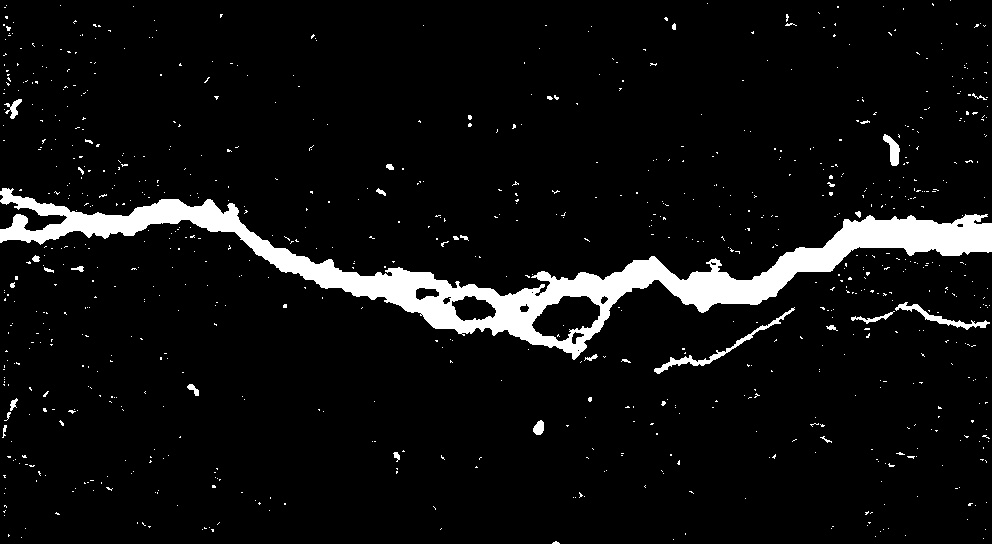}
    \\
    \includegraphics[width = 0.24\textwidth]{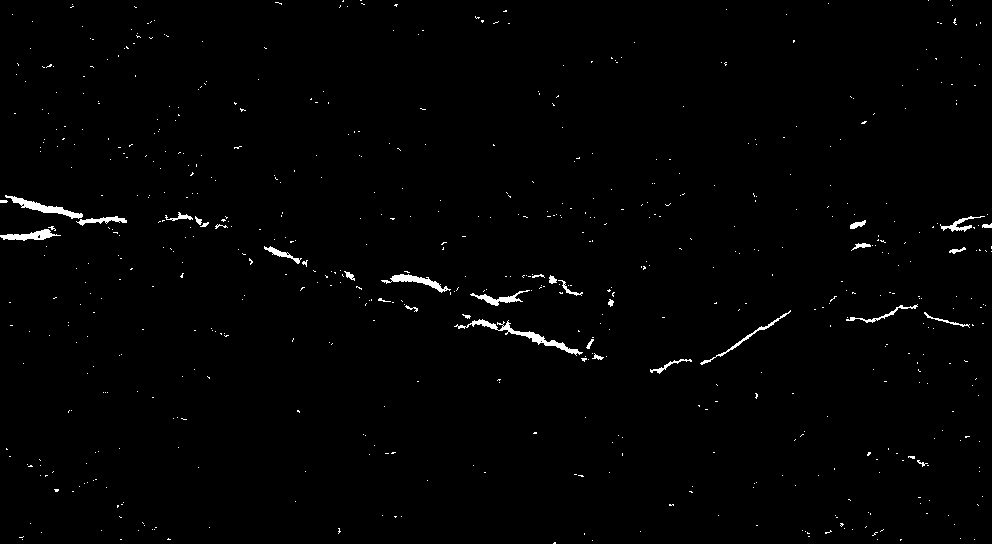}
    \includegraphics[width = 0.24\textwidth]{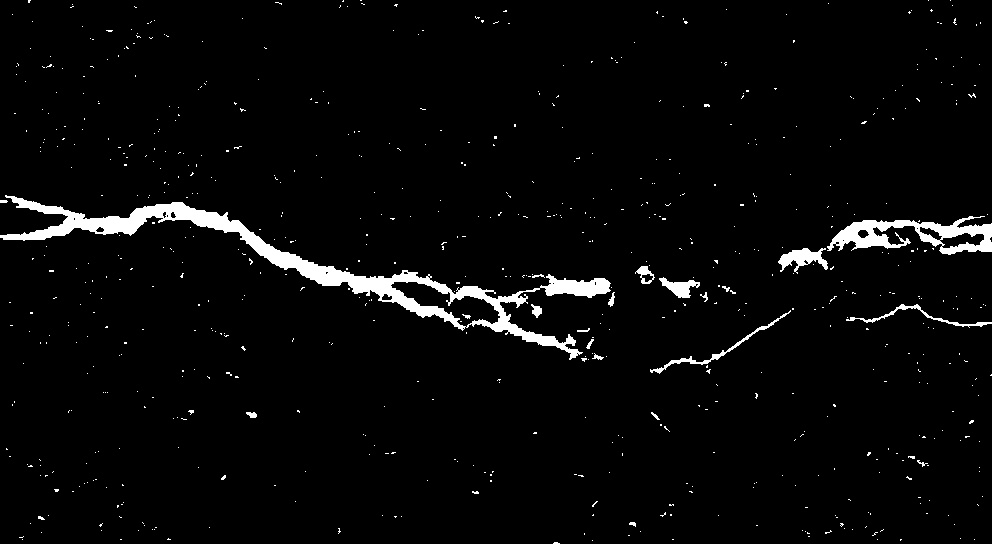}
    \includegraphics[width = 0.24\textwidth]{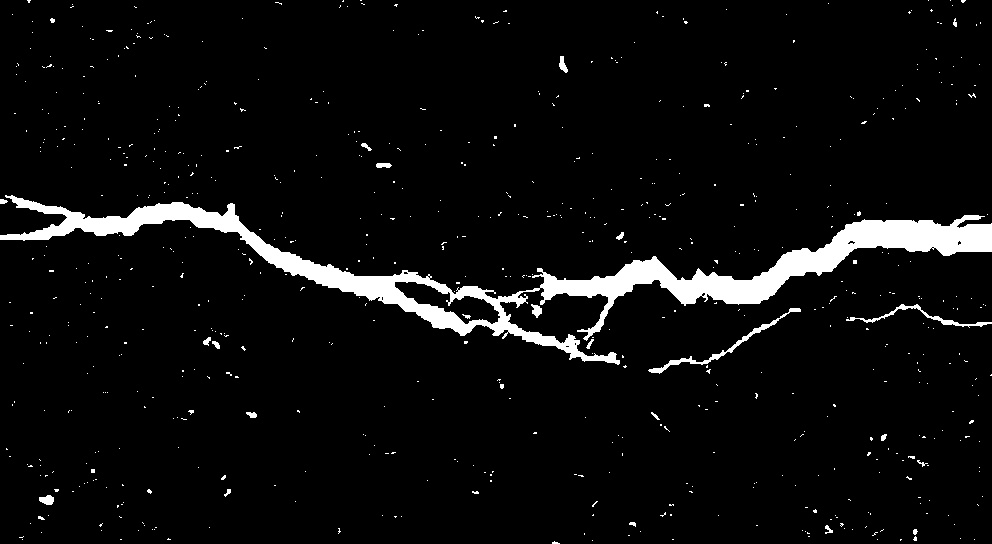}
    \includegraphics[width = 0.242\textwidth]{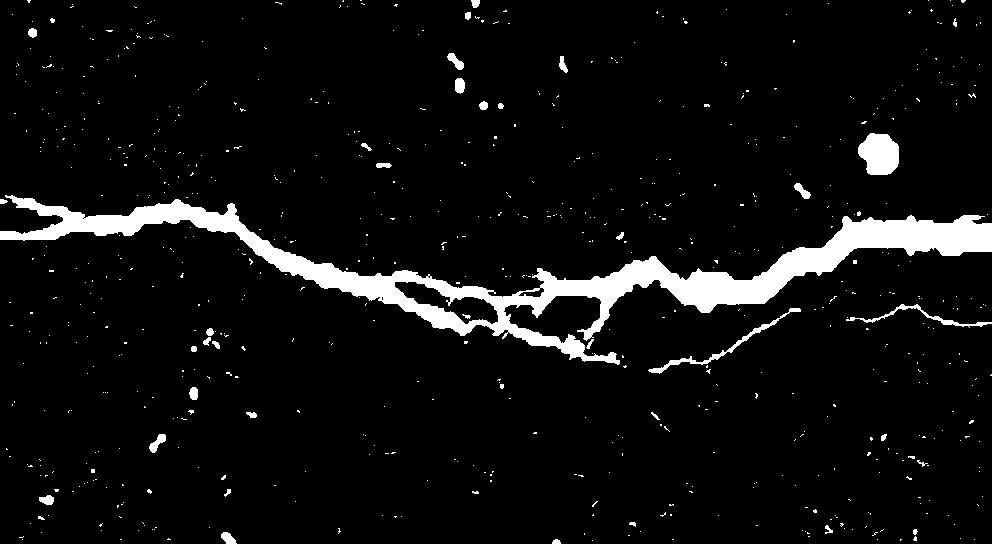}
    \caption{Experiment 3. Real cracks in concrete: slice from input CT image, results of the Riesz network and of U-net and U-net-mix with ranging pyramid levels (from 1 to 4). 
    Image sizes are $832 \times 1\,088$ (1st row) and $544 \times 992$ (4th row).}
    \label{fig:real-data-appendix}
\end{figure*}

\section{Experiments on different types of concrete: fiber reinforced concrete}
\label{secC}

It is a well-known weakness of concrete that it has low tensile strength, i.e. under high tensile force it fails abruptly and explosively. 
For that reason, reinforcement material is mixed with the cement paste creating a composite material.
Most common reinforcements are steel rebars.
Nowadays, fibers have become widely used as reinforcement in concrete creating a new class of reinforced concrete materials, e.g. ultra high performance fiber-reinforced concrete \cite{maryamh2022influence, maryamh2021influence,hauch2022predicting}.
A variety of materials can be used as fiber material, including glass, carbon, and basalt.
Since all of these materials have different mechanical properties, the properties of fiber reinforced concrete are connected to the properties of the concrete mixture, including the fiber material. 
Hence, a lot of effort has recently been invested in the investigation of fiber reinforced concrete samples with various material configurations.
In the context of CT imaging, different materials mean different energy absorption properties, i.e. fibers can appear both brighter or darker than concrete, which can result in very different images.
In the context of crack segmentation, this means that our methods should be able to efficiently handle these variations. This section compares the performance of the Riesz network, U-net, and U-net-mix from the previous sections on three different fiber reinforced concrete images.
We comment on possible post-processing steps to improve results and discuss the robustness of the methods in the context of fiber reinforced concrete. 

Fig. \ref{fig:crack-reinforced1} shows a sample of high performance concrete (HPC) with polypropylene fibers as reinforcement. See \cite{jung23ict} for more details on sample and crack initiation. In this image, fibers are long and appear dark and hence interfere with the crack in the center.
All three methods are able to extract the central and dominant crack in the middle. The Riesz network is not able to segment the thin crack on the left from the main crack, contrary to both U-nets. However, both U-nets accumulate a much larger amount of misclassified noise compared to the Riesz network. 

Fig. \ref{fig:crack-reinforced2} features a sample reinforced with steel fibers. For more details see \cite{kronenberger2018,schuler2020richtungsanalyse}. In this image, fibers appear bright and create uneven illumination effects. We use a simple pre-processing step to understand if we can reduce this effect and improve the performance of the methods.
Simple morphological openings with square structuring elements of half-sizes $2$ and $5$ are used for that purpose.
As the size of the structuring element increases, segmentation results improve for all three methods. While the Riesz network struggles with low contrast cracks on the right, both types of U-net segment falsely many non-crack voxels. 

The CT image from Fig. \ref{fig:crack-reinforced3} originates from ultra high performance concrete reinforced with steel fibers \cite{maryamh2021influence}. 
Again, fibers turn out to be bright structures in the images. 
These extremely highly X-ray absorbing fibers affect the gray value dynamics of the CT images.
Morphological openings with square structuring elements of half-sizes $2$ and $5$ are applied to reduce this effect. As we increase the size, crack segmentation improves for the Riesz networks. Both types of U-net segment large amounts of noise, even with opening as a pre-processing step, rendering them ineffective for this sample.

\begin{figure*}
    \centering
    \includegraphics[width = 0.4 \textwidth]{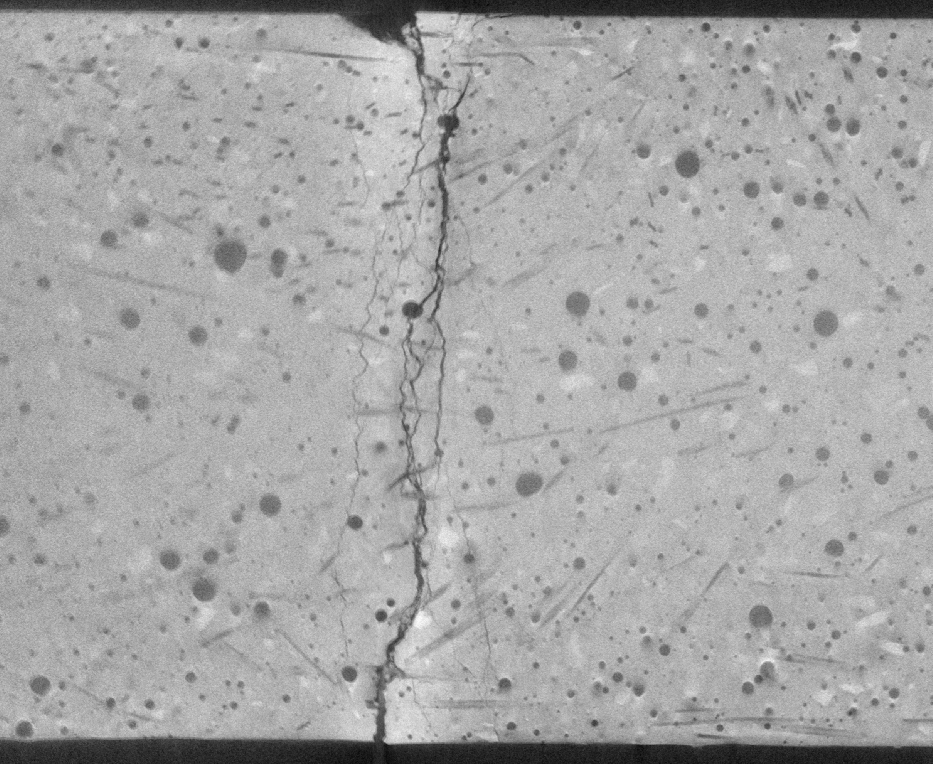}
    \includegraphics[width = 0.4 \textwidth]{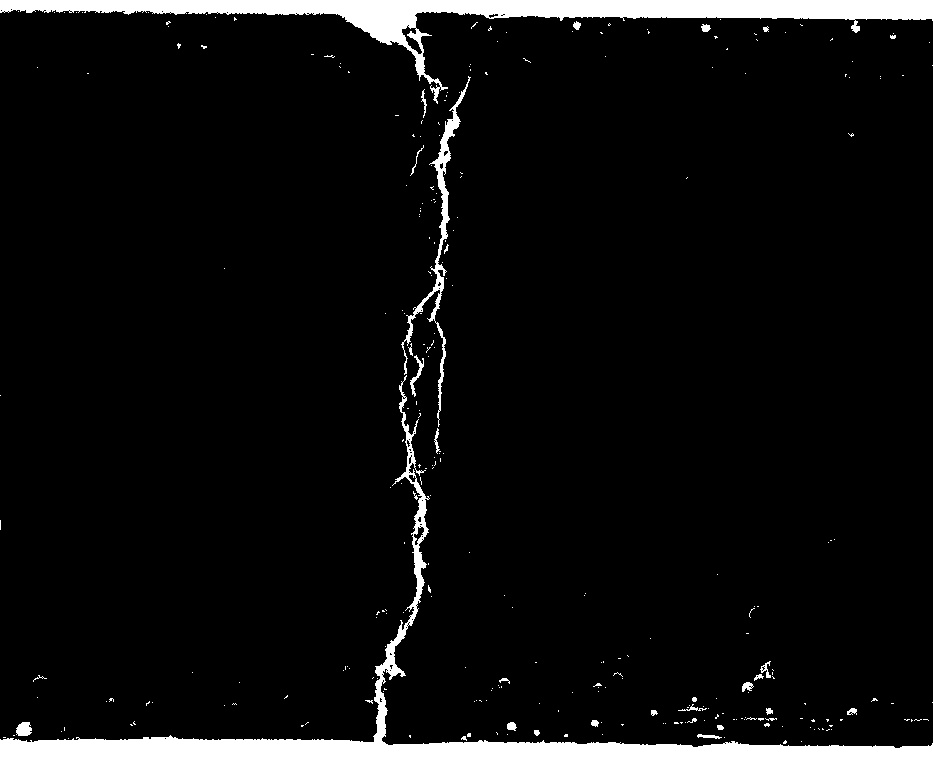}
    \includegraphics[width = 0.4 \textwidth]{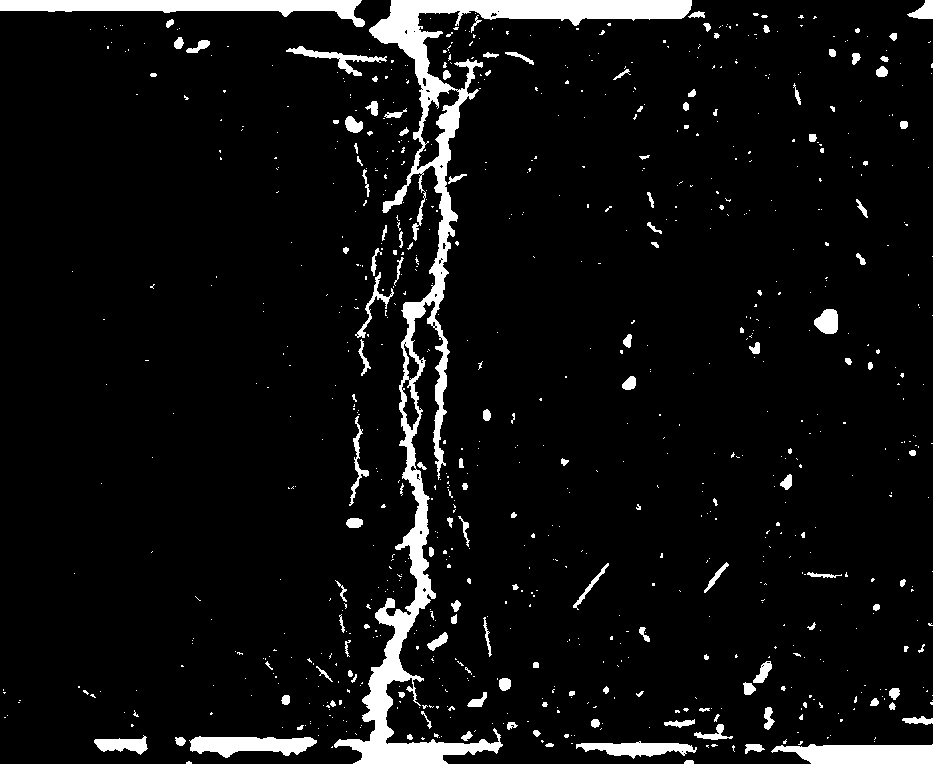}
    \includegraphics[width = 0.4 \textwidth]{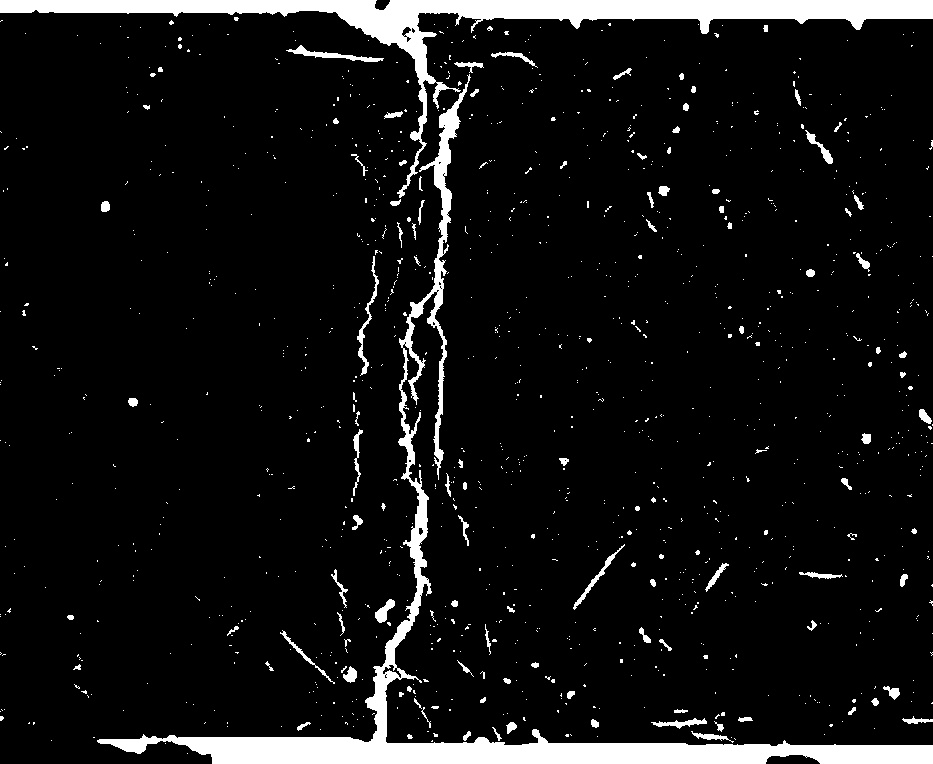}
    \caption{Cracks in high-performance concrete with polypropylene fibers. First row: input image (left), segmentation results from the Riesz network (right). Second row: U-net (left) and U-net mix (right). Image size is $933 \times 764.$}
    \label{fig:crack-reinforced1}
\end{figure*}

\begin{figure*}
    \centering
    \includegraphics[width = 0.32\textwidth]{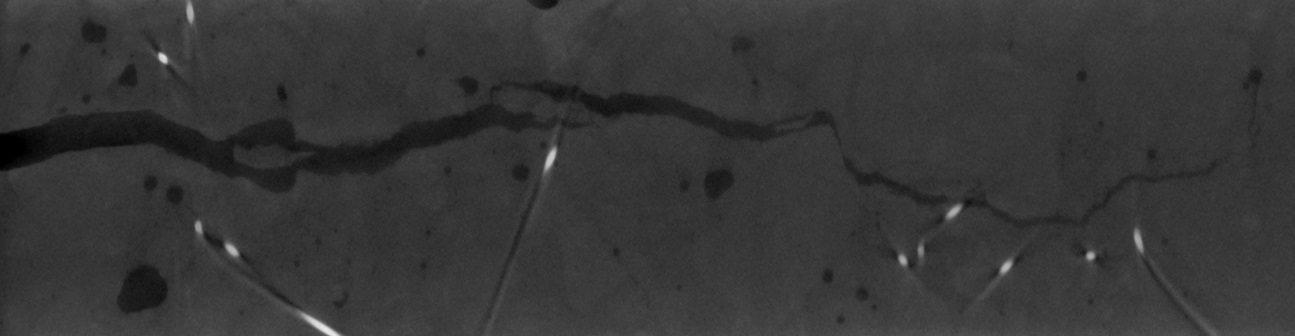}
    \includegraphics[width = 0.32 \textwidth]{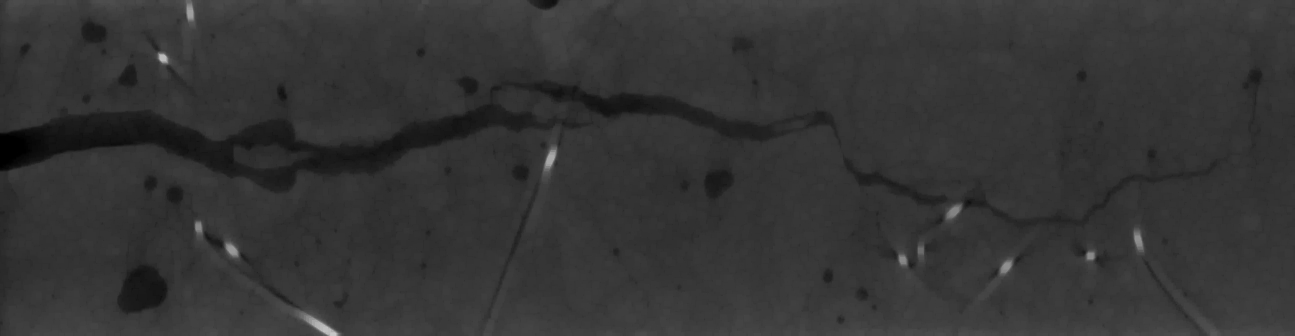}
     \includegraphics[width = 0.32 \textwidth]{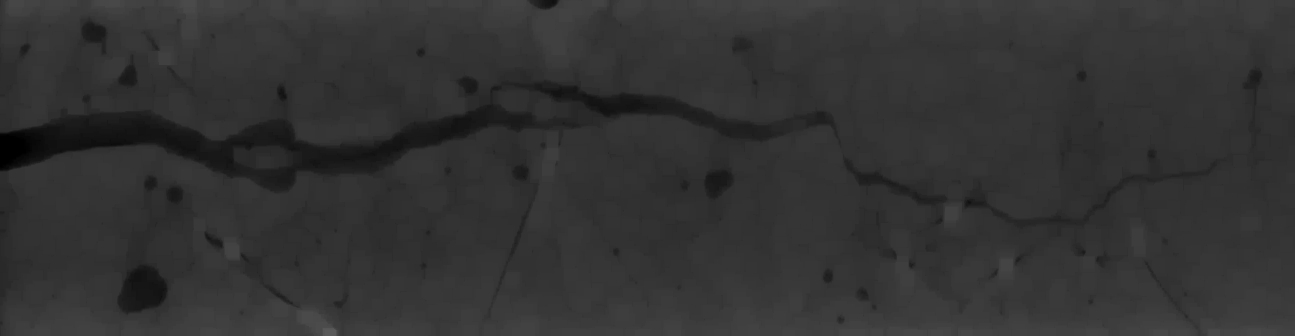}
     \\
    \includegraphics[width = 0.32\textwidth]{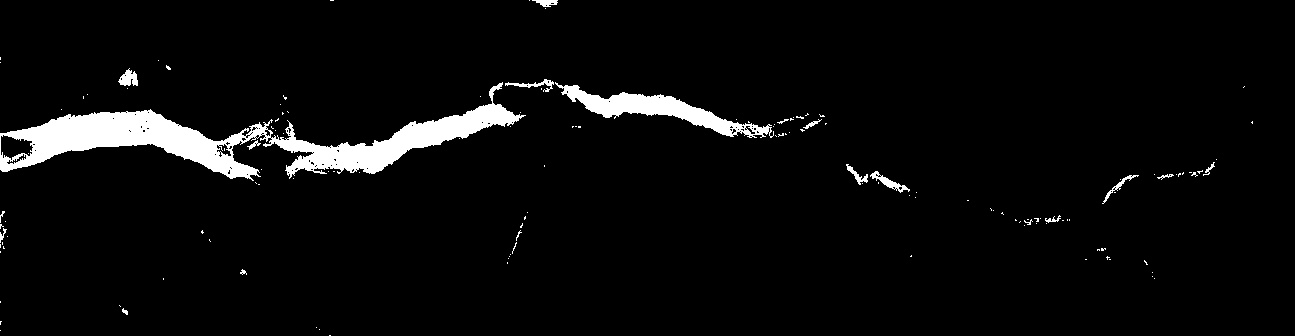}
    \includegraphics[width = 0.32 \textwidth]{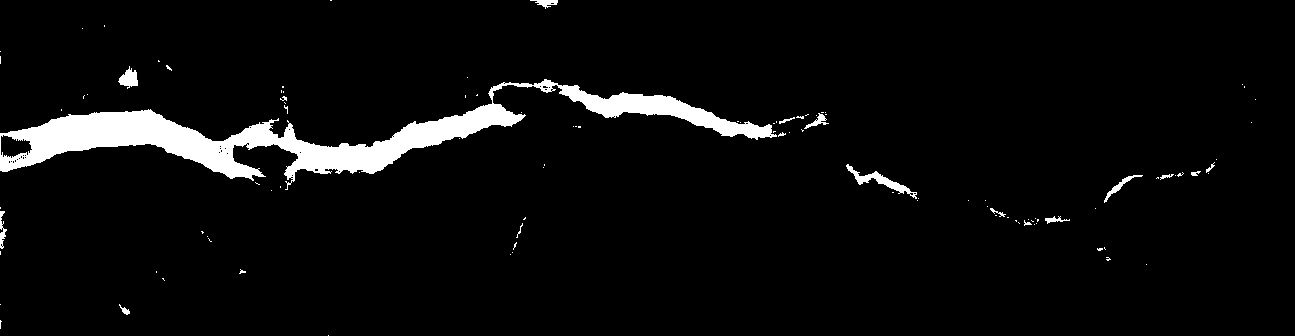}
    \includegraphics[width = 0.32 \textwidth]{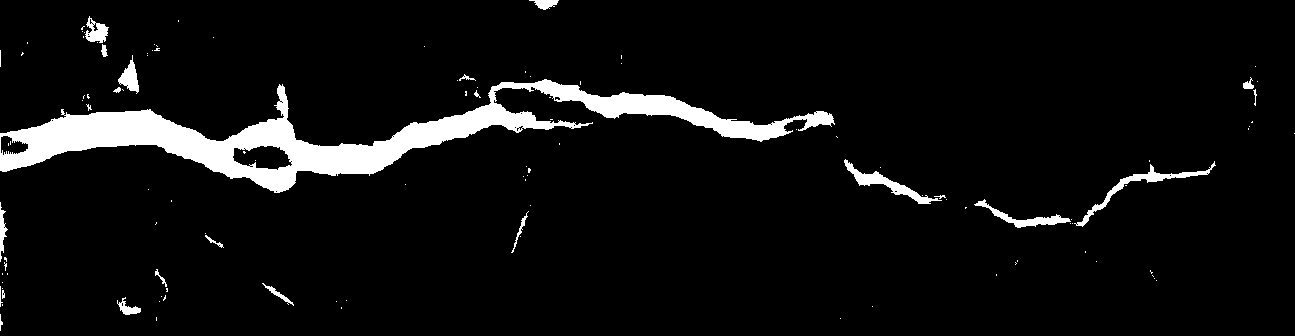}
    \\
    \includegraphics[width = 0.32 \textwidth]{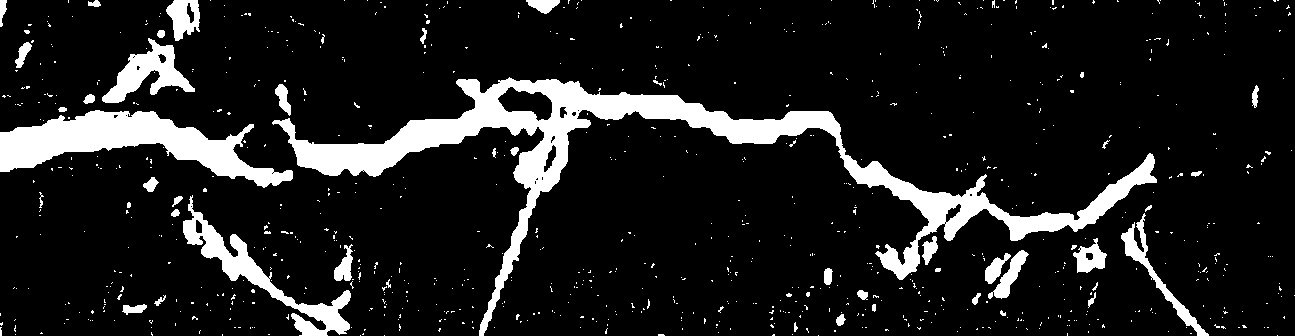}
    \includegraphics[width = 0.32 \textwidth]{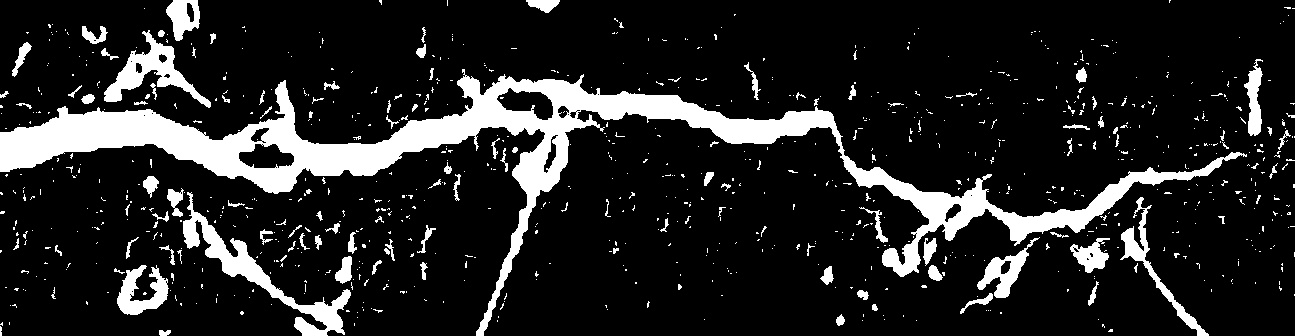}
    \includegraphics[width = 0.32 \textwidth]{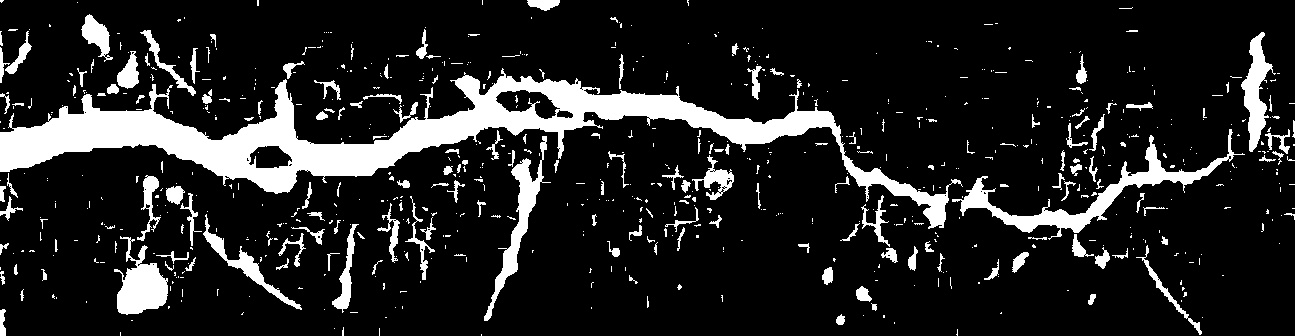}
    \\
    \includegraphics[width = 0.32\textwidth]{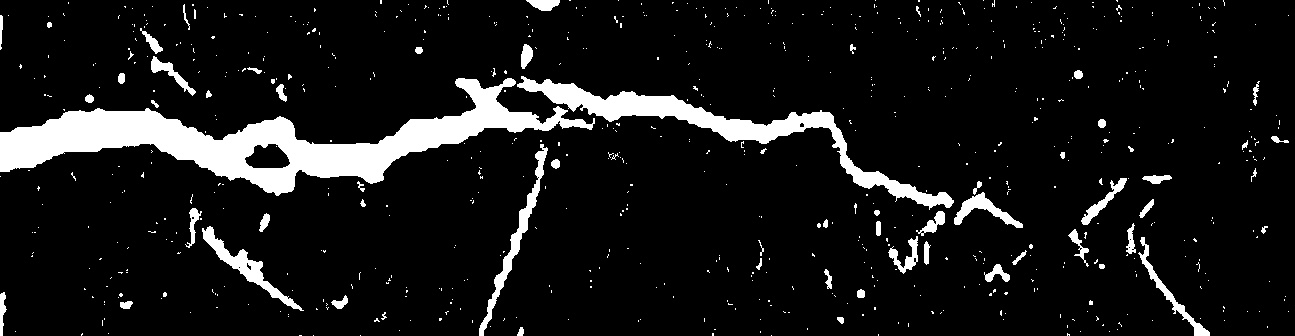}
    \includegraphics[width = 0.32 \textwidth]{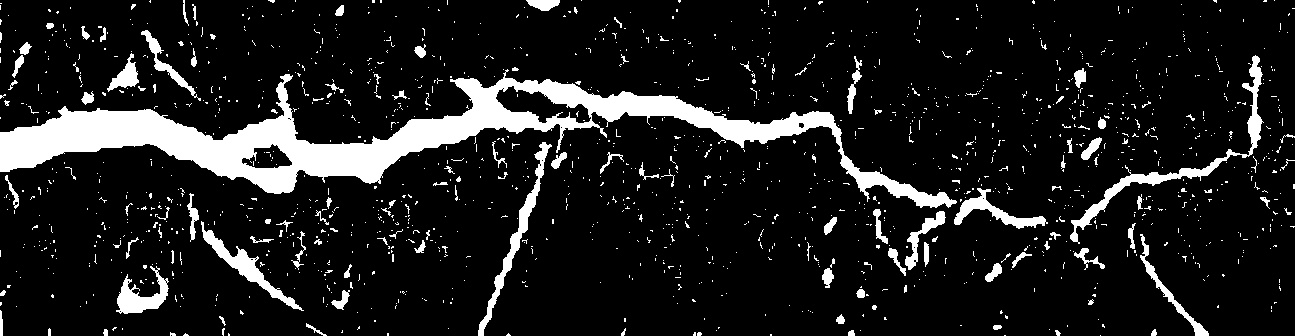}
    \includegraphics[width = 0.32\textwidth]{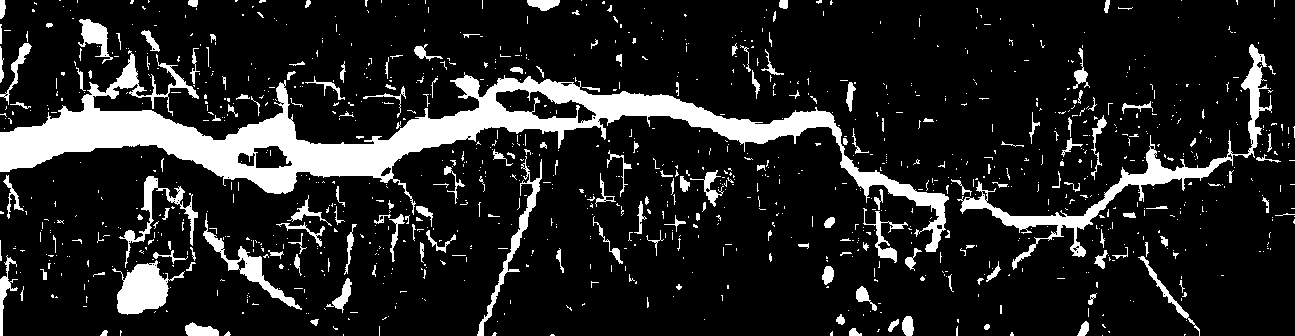}
    \caption{Cracks in concrete with steel fibers. Rows: input image, segmentation results from the Riesz network, U-net and U-net mix, respectively. Columns: original images, images after applying square closing of half-size $2$, and images after applying square opening of half-size $5$. Image size is $1\,295 \times 336$.}
    \label{fig:crack-reinforced2}
\end{figure*}

\begin{figure*}[h]
    \centering
    \includegraphics[width = 0.32\textwidth]{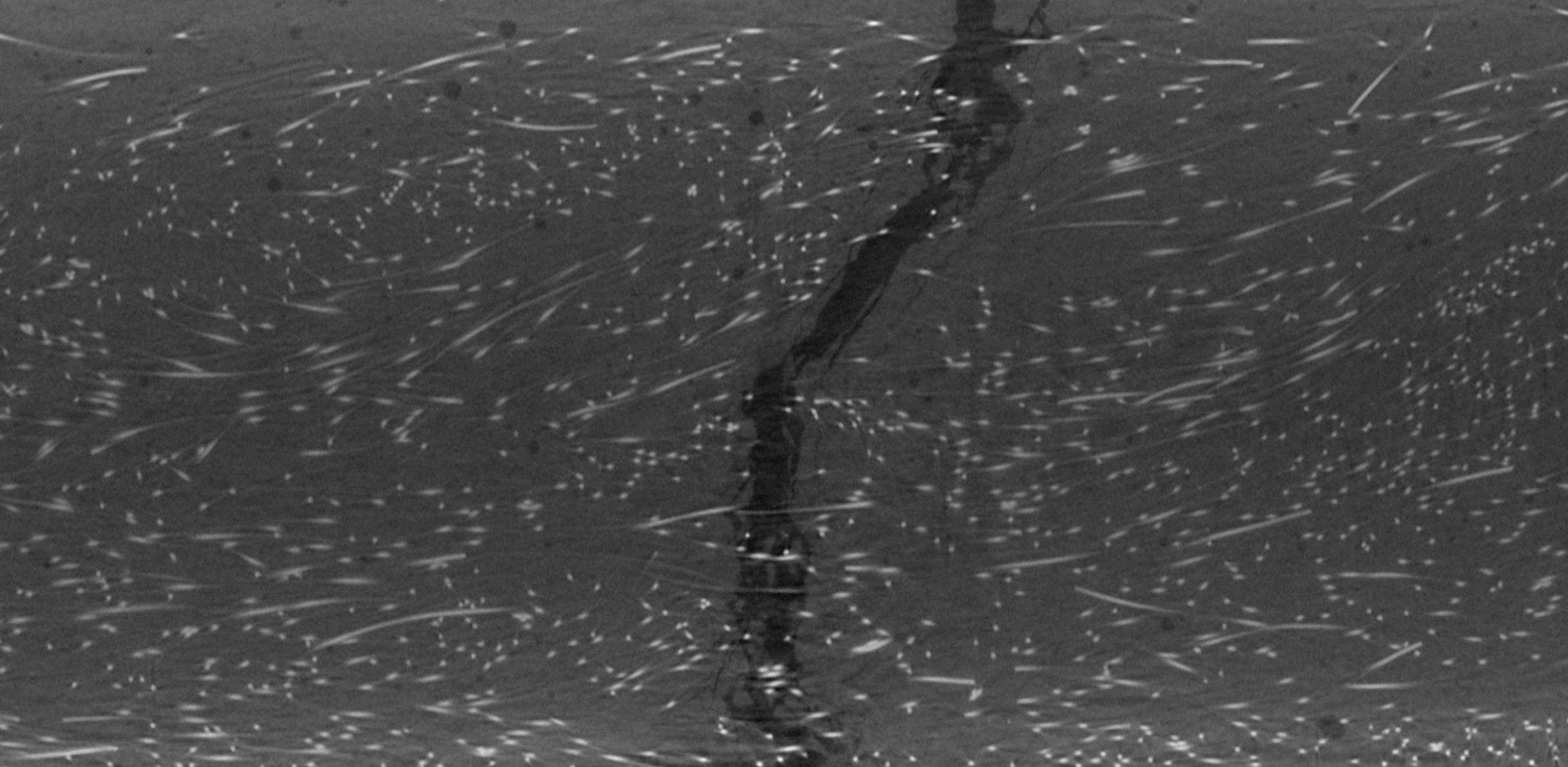}
    \includegraphics[width = 0.32 \textwidth]{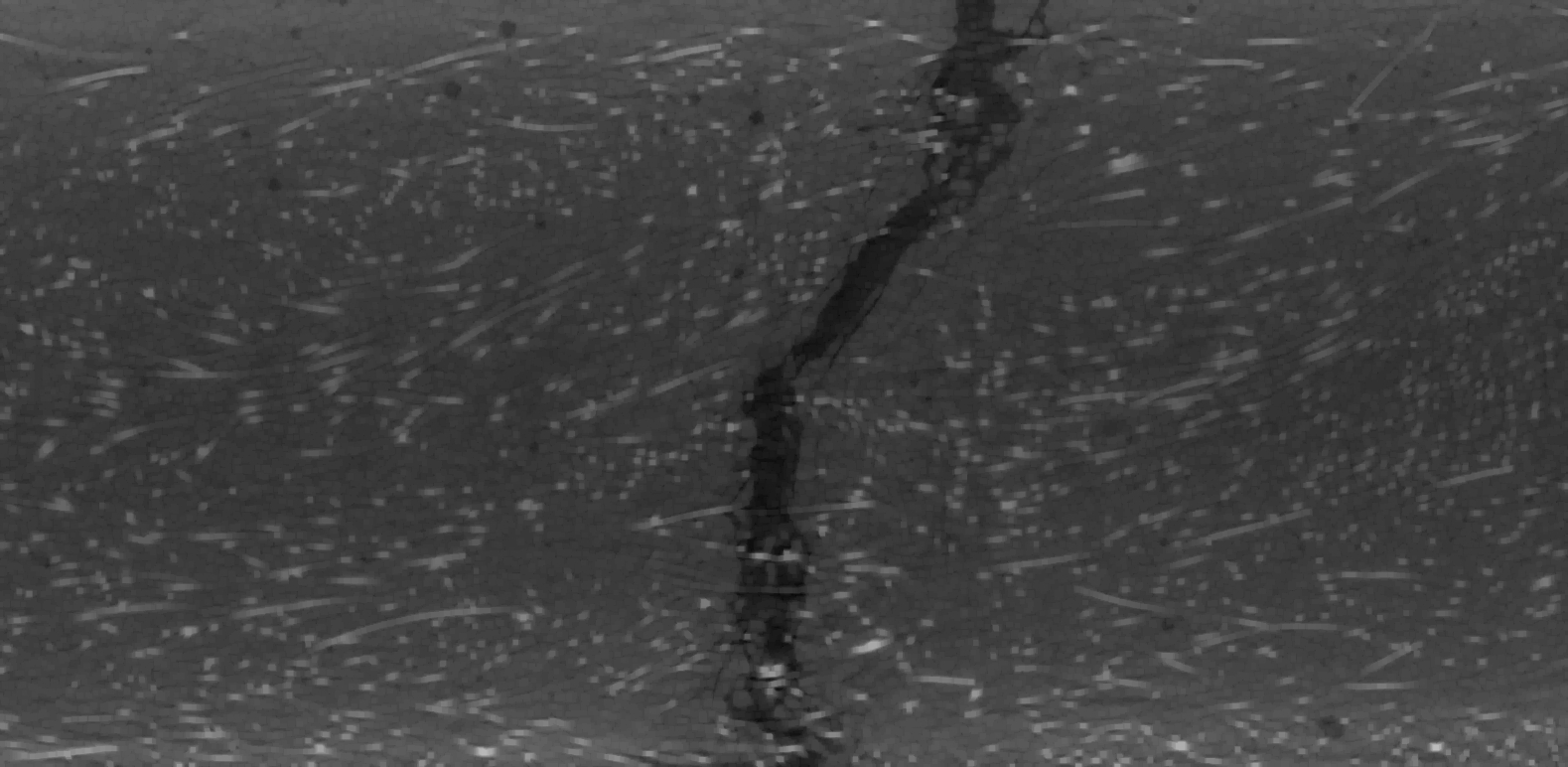}
    \includegraphics[width = 0.32 \textwidth]{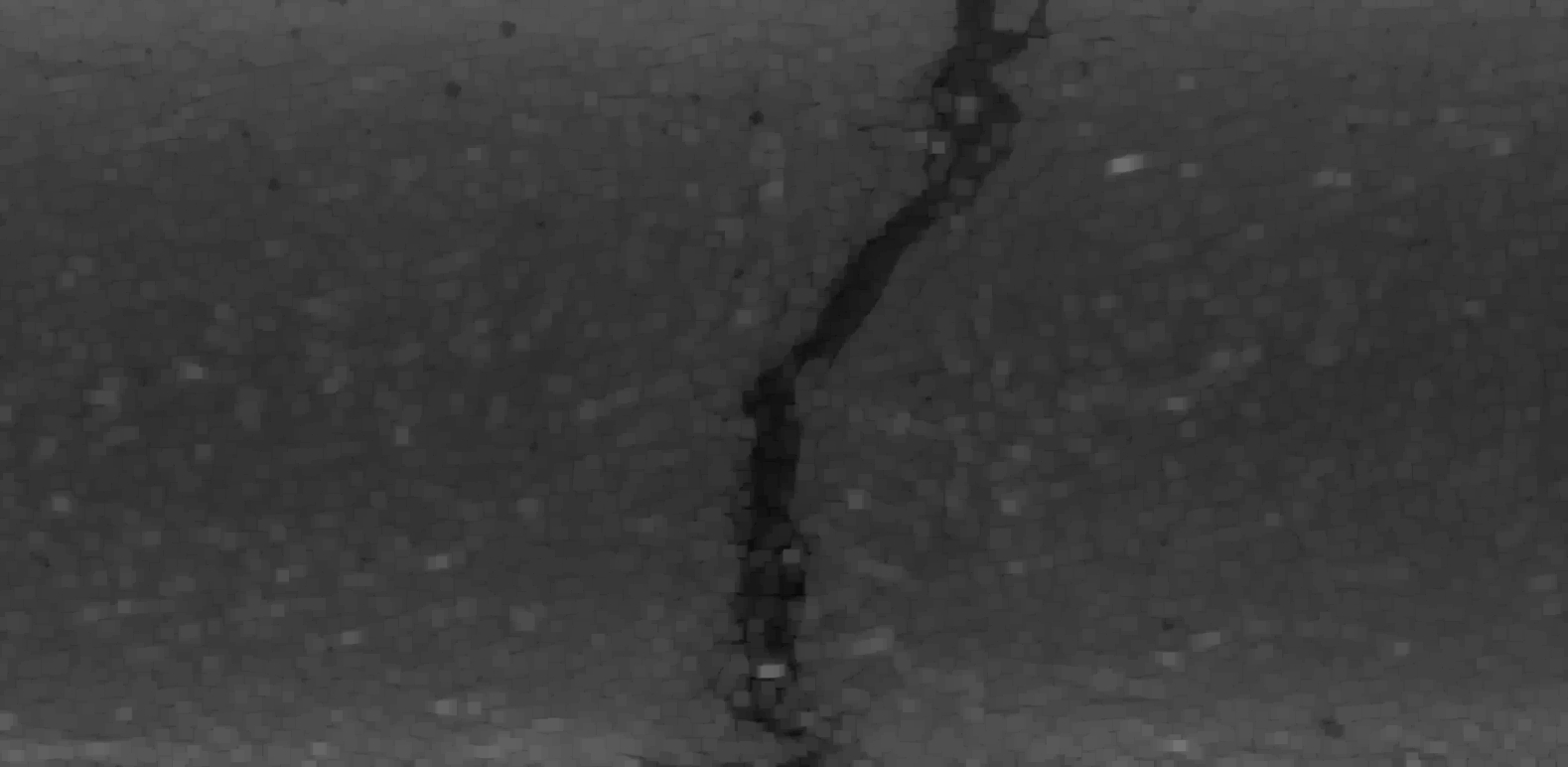}
    \\
    \includegraphics[width = 0.32 \textwidth]{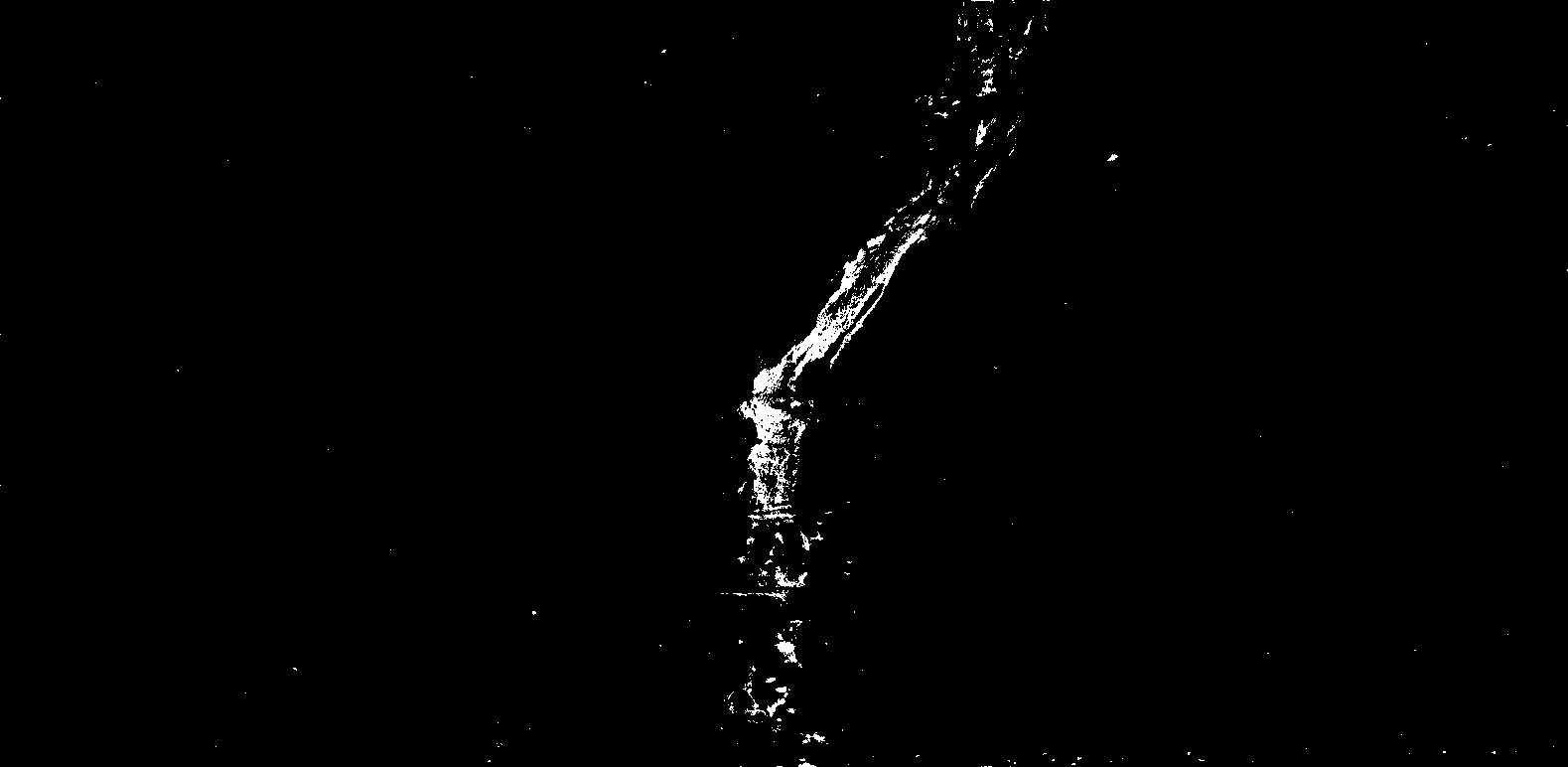}
    \includegraphics[width = 0.32 \textwidth]{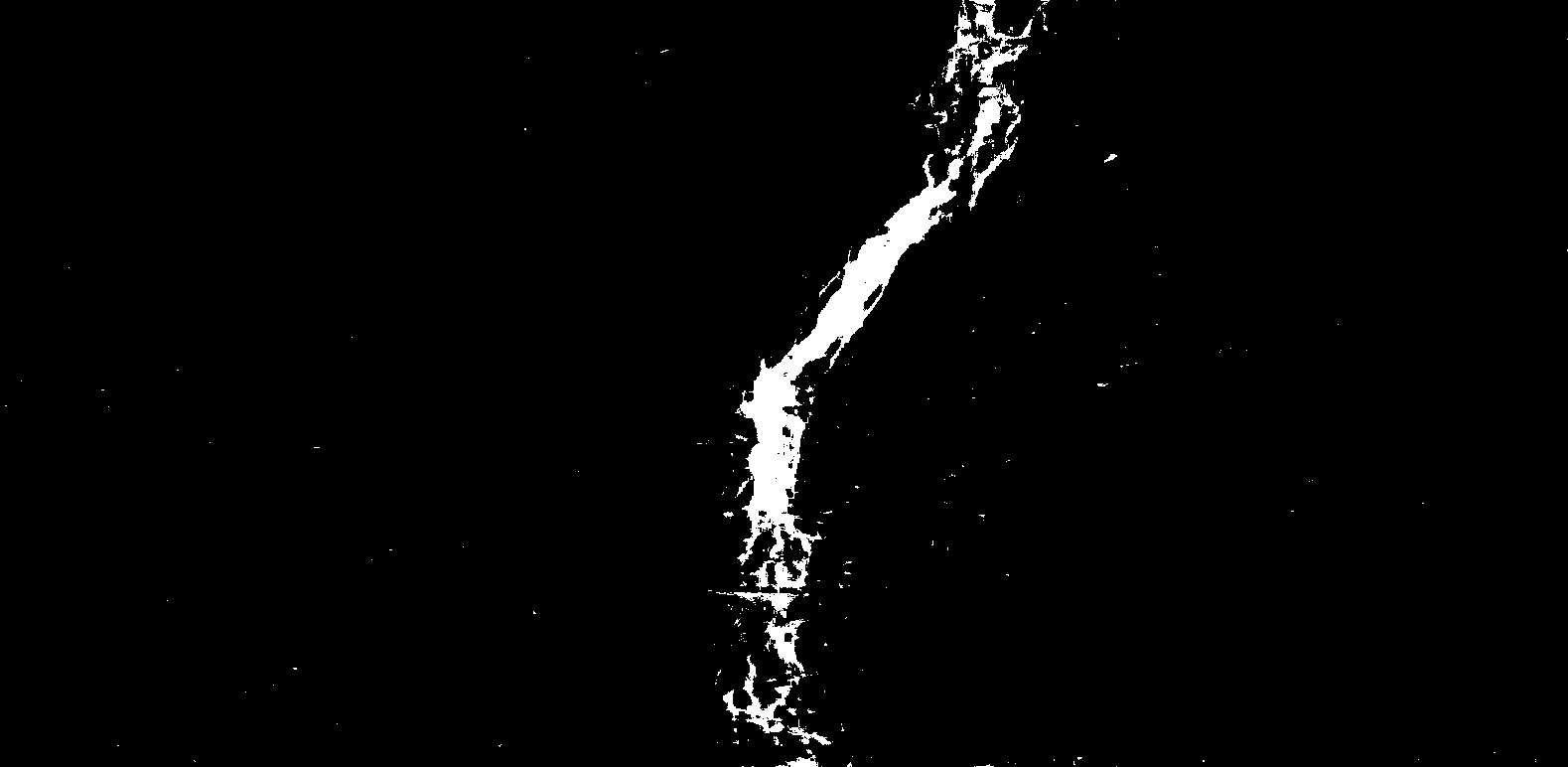}
    \includegraphics[width = 0.32 \textwidth]{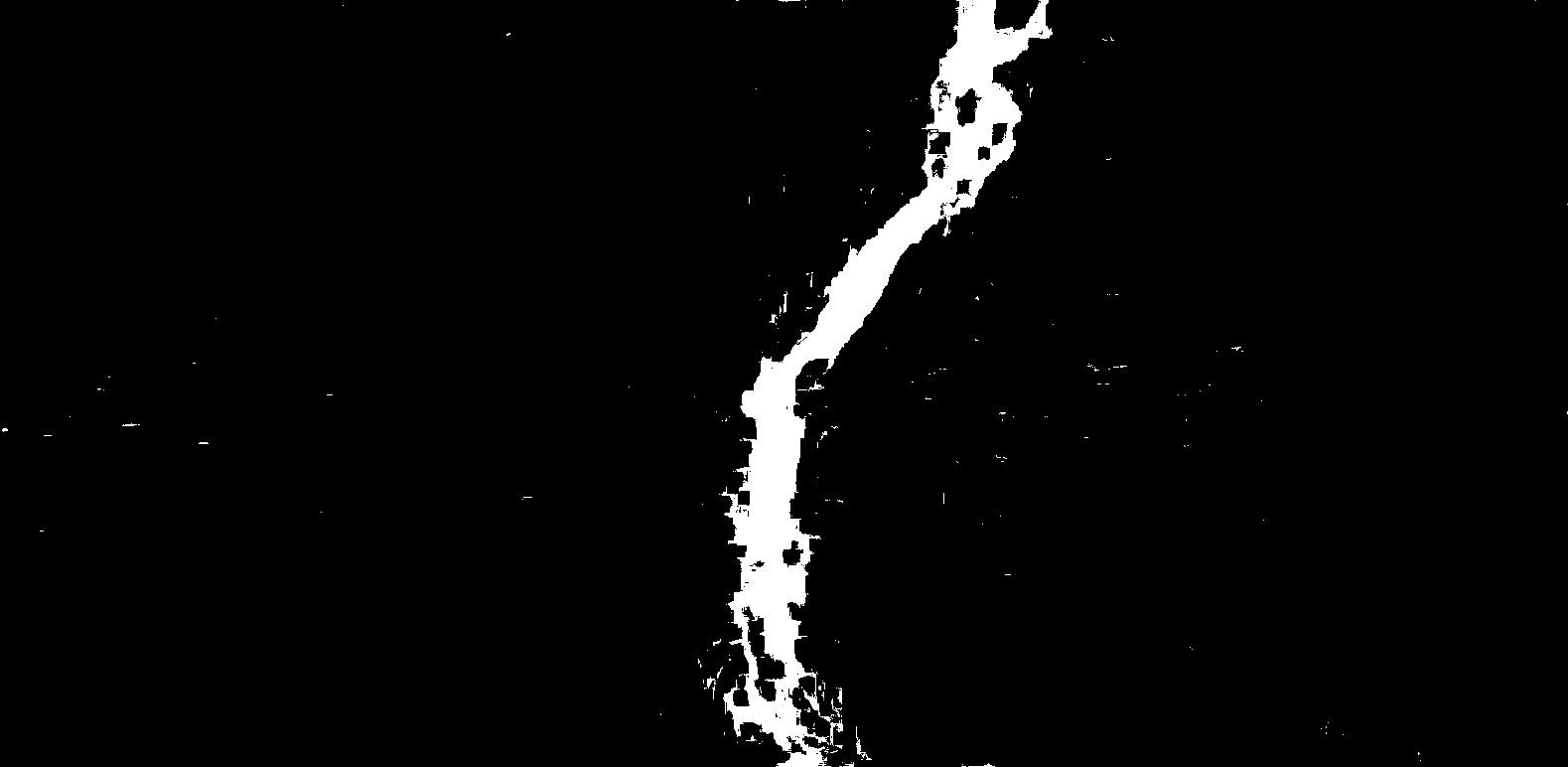}
    \\
    \includegraphics[width = 0.32 \textwidth]{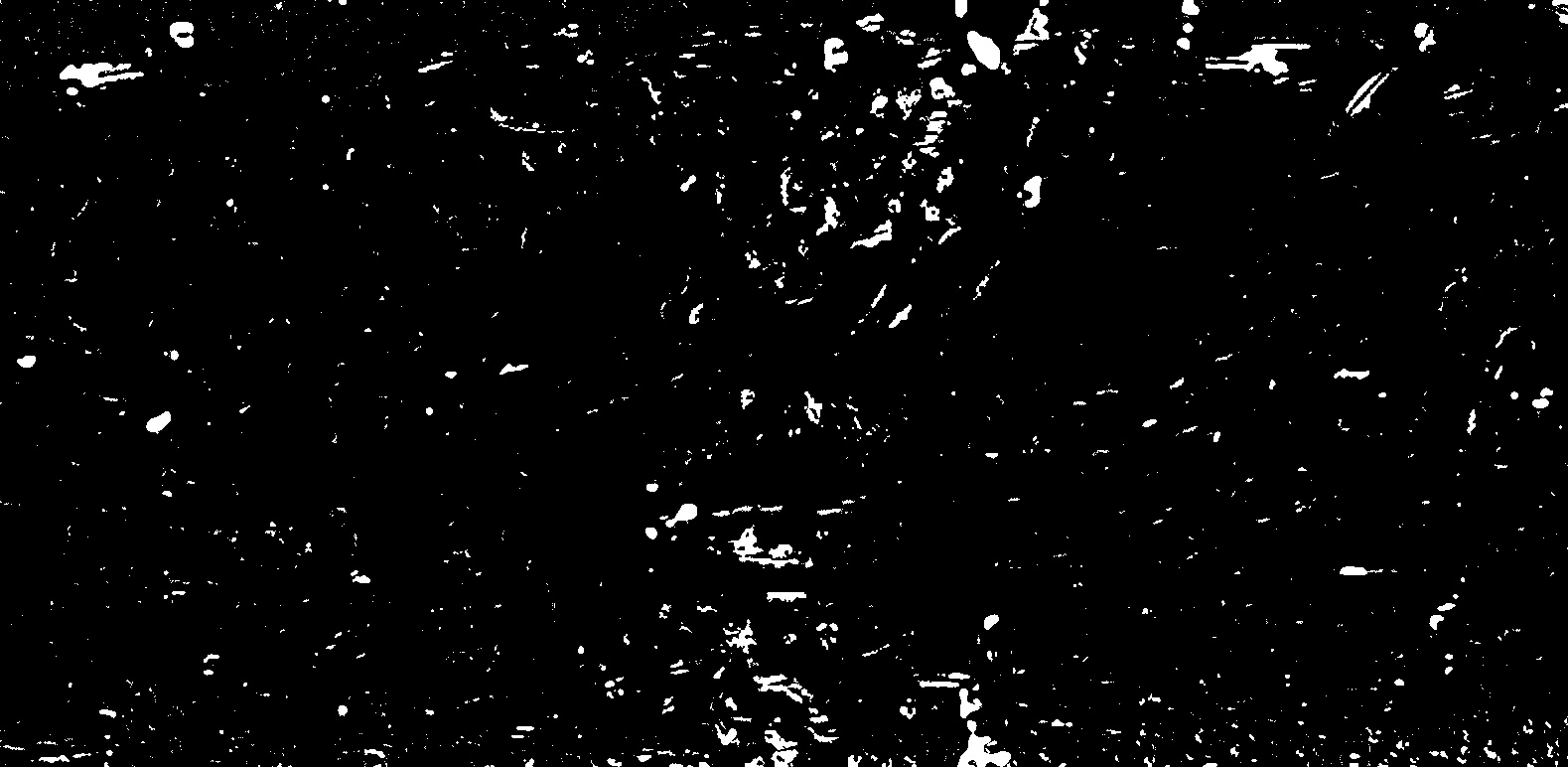}
    \includegraphics[width = 0.32 \textwidth]{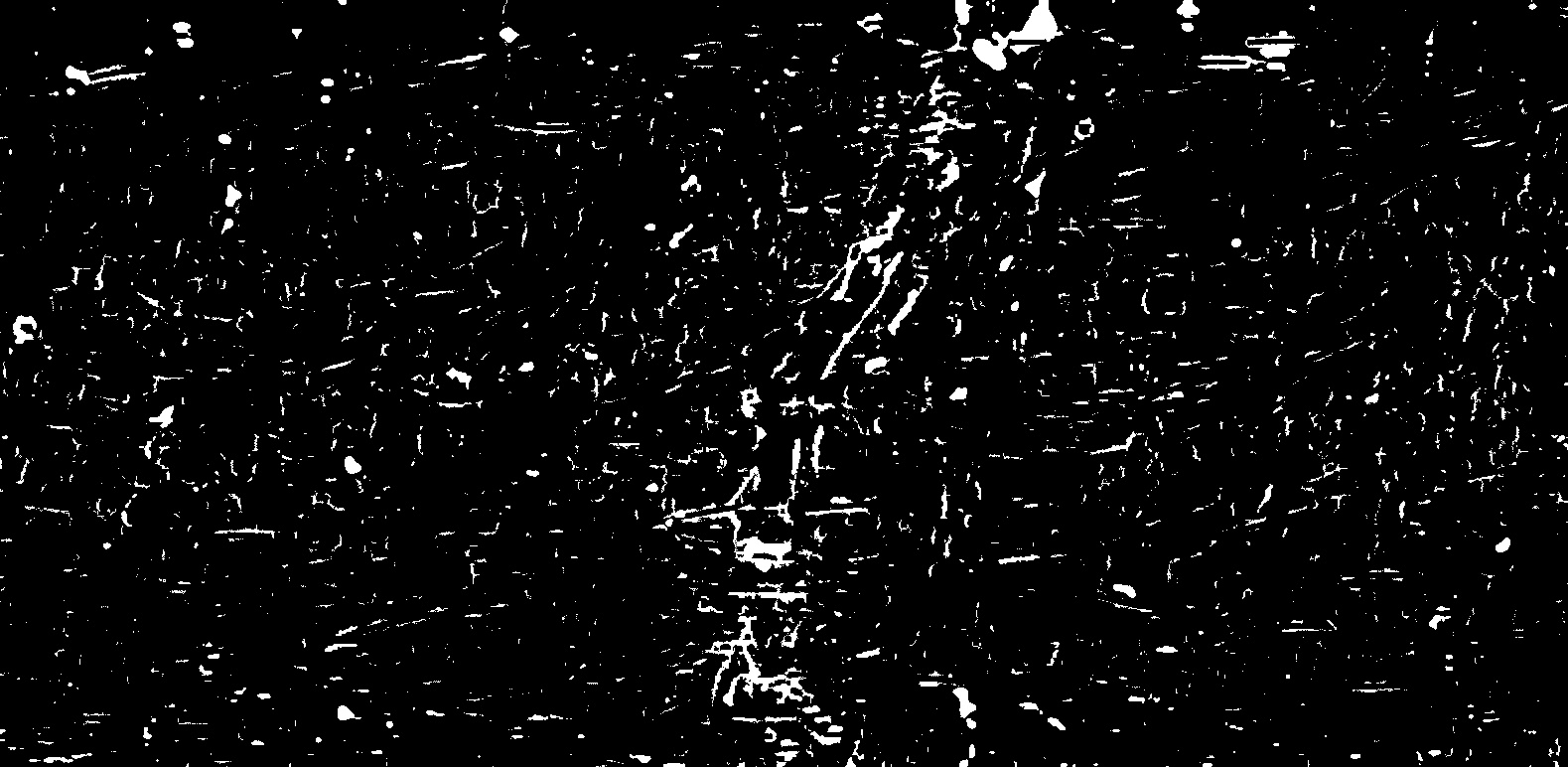}
    \includegraphics[width = 0.32 \textwidth]{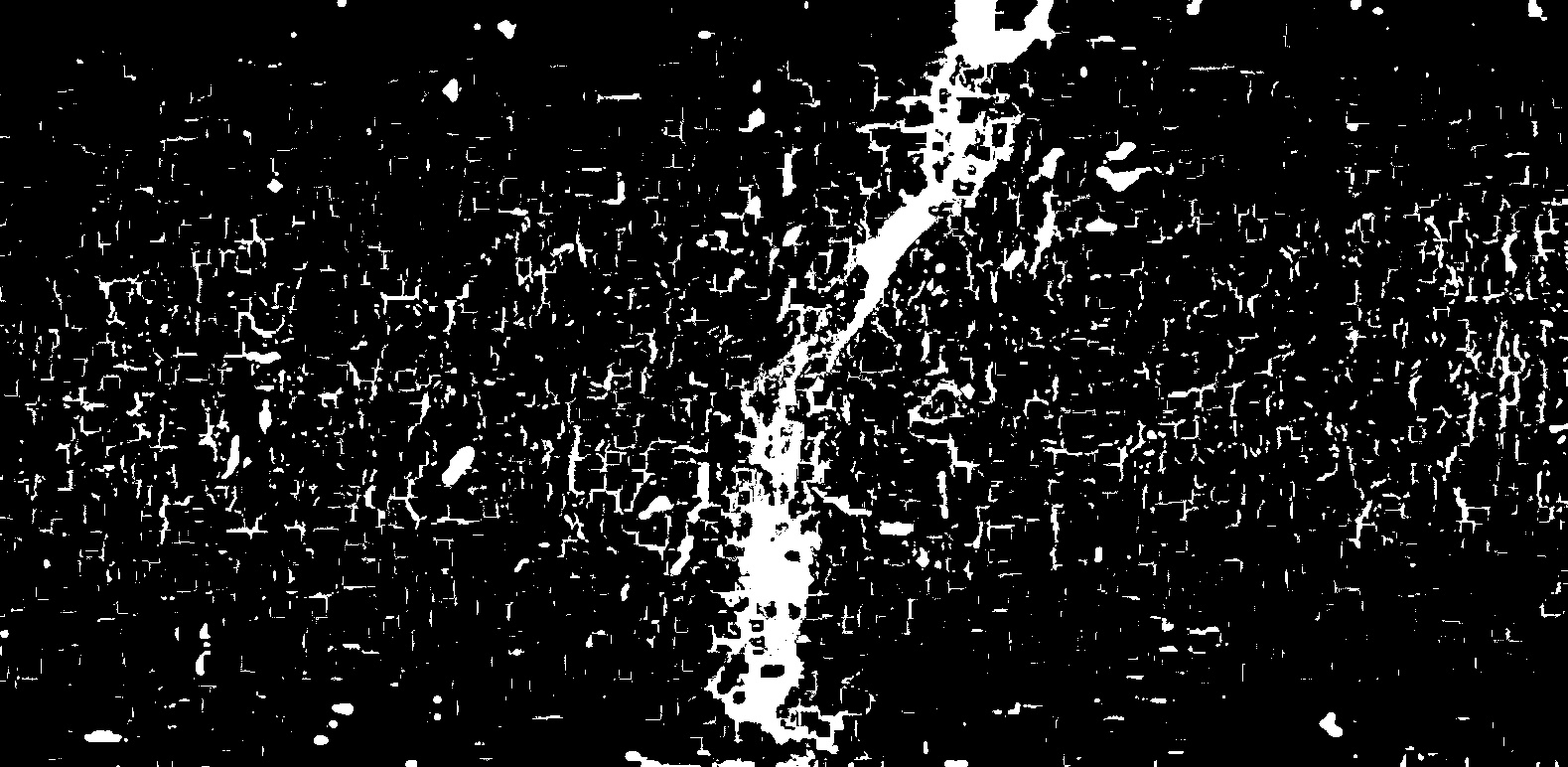}
    \\
    \includegraphics[width = 0.32 \textwidth]{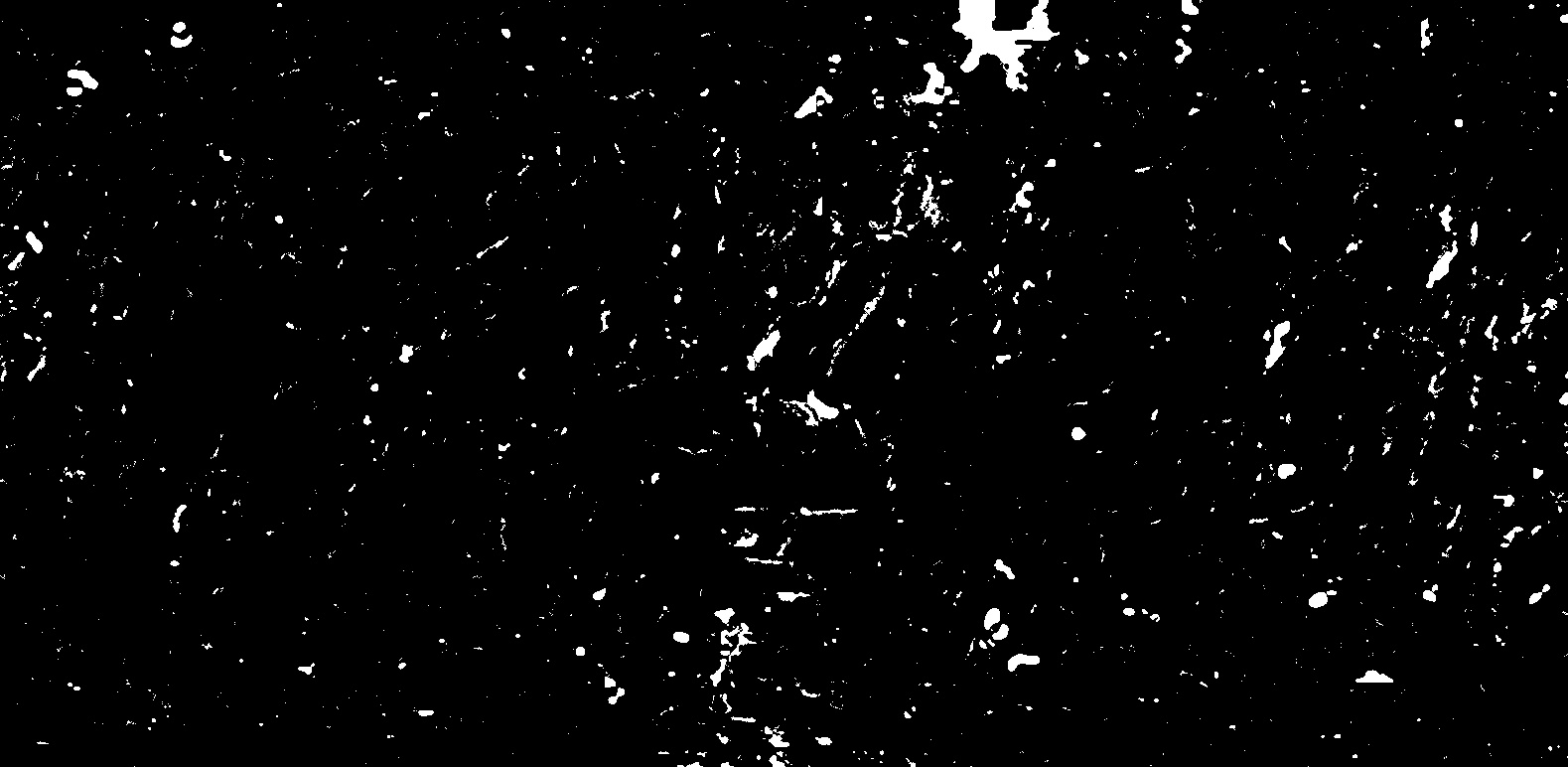}
    \includegraphics[width = 0.32 \textwidth]{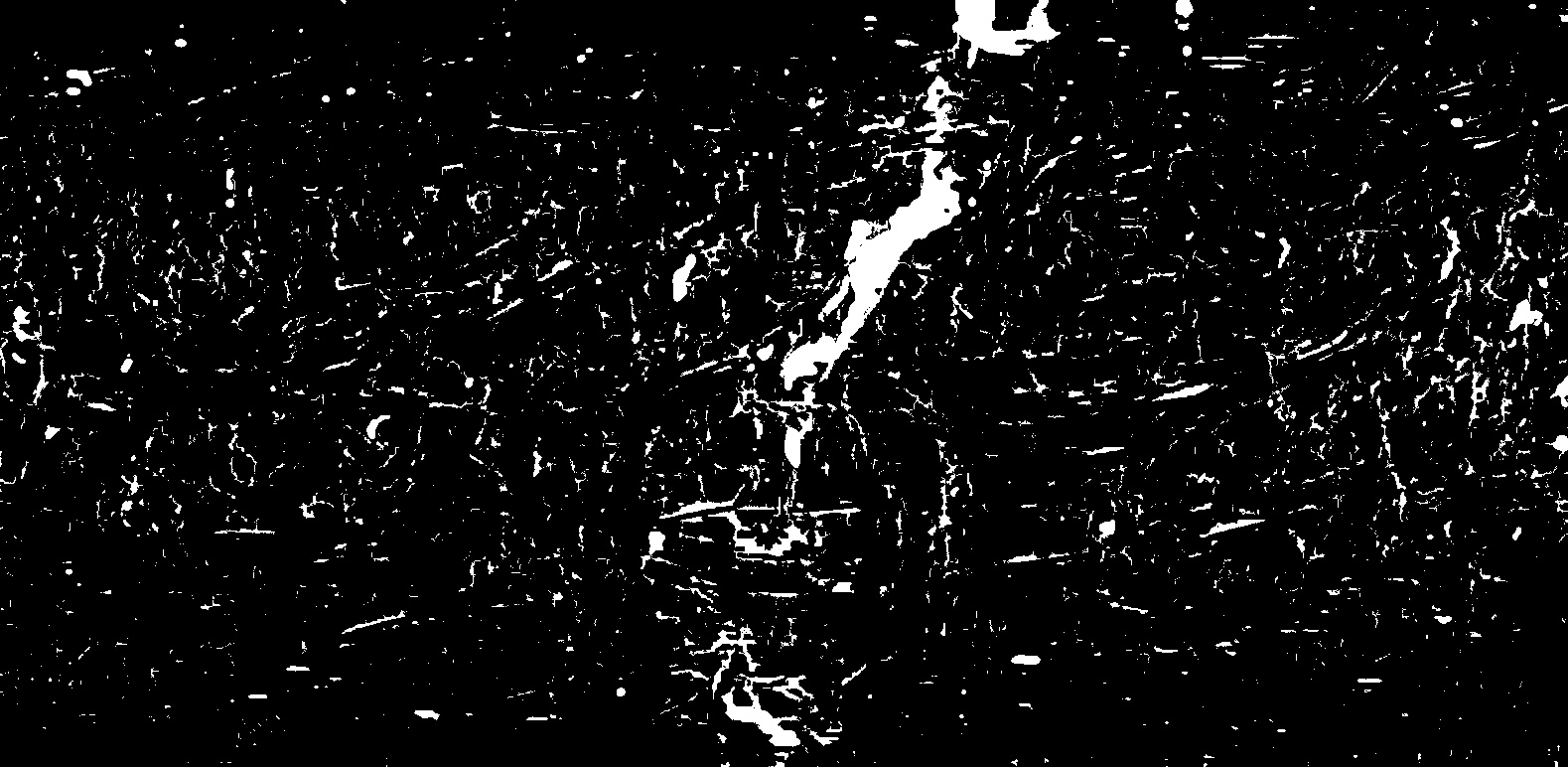}
    \includegraphics[width = 0.32 \textwidth]{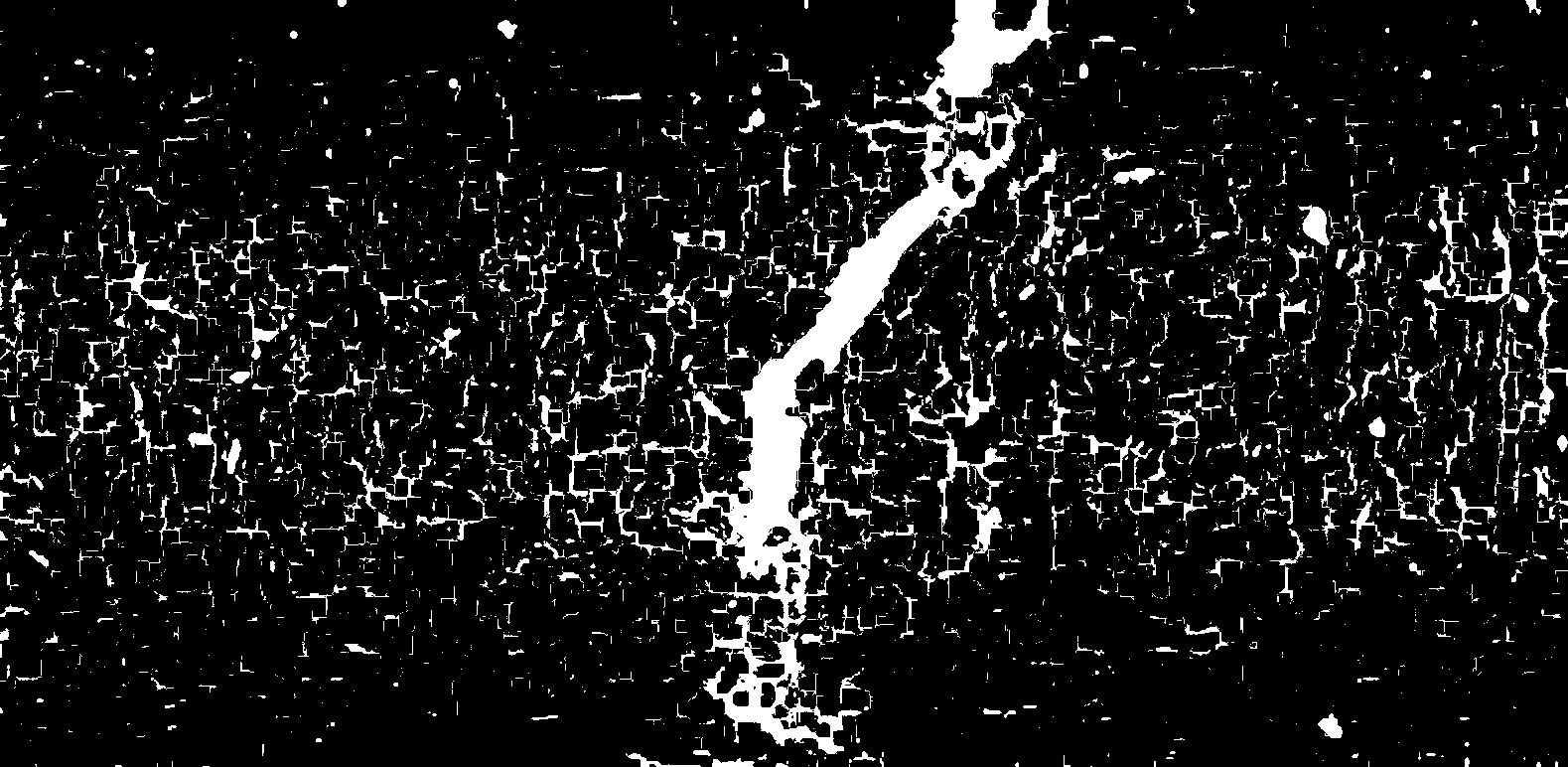}
    \caption{Cracks in samples of ultra high performance  concrete reinforced with steel fibers. Rows (from left to right): input image, segmentation results from the Riesz network, U-net and U-net mix, respectively. Column: original images, images after applying square opening of half-size $2$, and images after applying square closing of half-size $5$. Image size is $1\,579 \times 772$.}
    \label{fig:crack-reinforced3}
\end{figure*}

\section{Method implementation}
The Riesz network has been implemented in PyTorch. Here, we discuss the implementation of the Riesz transform (Algorithm \ref{alg:RT}) and the Riesz layer (Algorithm \ref{alg:RL}).

In Algorithm \ref{alg:RT}, the discrete Fourier transform "FT" denotes the composition of two \textit{torch} functions: first, \textit{torch.fft.fft2} is applied to the input and then followed by \textit{torch.fft.fftshift}. Similary, its inverse "iFT" uses \textit{torch.fft.ifftshift} and \textit{torch.fft.ifft2} but in the reverse order as in "FT". The operator "$\odot$" denotes a pointwise multiplication of two image maps (of the same size).
\begin{algorithm*}
\caption{RieszTransform$(I)$}\label{alg:RT}
\begin{algorithmic}
\Require $I \in \R^{H\times W}$
\Require H, W are even.
\State Discrete version of the Riesz kernel from equation (\ref{eq:FT:RT})
\For{$n_1\in\{1,\cdots,\frac{H}{2}\}$}
\For{$n_2\in\{1,\cdots,\frac{W}{2}\}$}
    \State $R_1(n_1,n_2) = i\frac{n_1-\frac{H}{2}}{\sqrt{(n_1-\frac{H}{2})^2+(n_2-\frac{W}{2})^2}}$ 
    \State $R_2(n_1,n_2) = i\frac{n_2-\frac{W}{2}}{\sqrt{(n_1-\frac{H}{2})^2+(n_2-\frac{W}{2})^2}}$ 
\EndFor
\EndFor
\State Discrete Fourier transform of $I$
\State $\Tilde{I} = FT(I)$
\State First order Riesz transform
\State $I_1 = iFT(\Tilde{I} \odot R_1)$
\State $I_2 = iFT(\Tilde{I} \odot R_2)$
\State Second order Riesz transform
\State $I^{(2,0)} = iFT(\Tilde{I} \odot R_1 . R_1)$
\State $I^{(0,2)} = iFT(\Tilde{I} \odot R_2 . R_2)$
\State $I^{(1,1)} = iFT(\Tilde{I} \odot R_2  . R_1)$
\State \Return $(I_1,I_2, I^{(2,0)},I^{(1,1)},I^{(0,2)}) \in \R^{H\times W \times 5}$
\end{algorithmic}
\end{algorithm*}

Based on equation (\ref{eq:FT:RT}), Algorithm \ref{alg:RL} is the implementation of equations (\ref{base:layer}) and (\ref{full:layer}). Here, a linear combination of Riesz transforms of the input feature maps can be seen as a 1d convolution across the channel dimension. Hence, "conv1d" represents the \textit{torch} function \textit{torch.nn.Conv2d} with $kernel\_size = (1,1)$ where $[5N_{input} , N_{output}]$ specify the number of input and output channels, respectively.
\begin{algorithm*}
\caption{RieszLayer$(L,N_{output})$}\label{alg:RL}
\begin{algorithmic}
\Require $L \in \R^{H\times W \times N_{input}}$
\Require H, W are even.
\Require $N_{output} \in \N$
\State Riesz transform for every channel in $L$
\State Define $L_R \in \R^{H\times W \times (5 N_{input})}$
\For{$i\in\{1,\cdots,N_{input}\}$}
    \State $L_R(:,:, (5(i-1)):(5i)) = \text{RieszTransform(L(:,:,i))}$
\EndFor
\State \Return $conv1d(L_R,[5N_{input}, N_{output}]) \in \R^{H\times W \times N_{output}}$
\end{algorithmic}
\end{algorithm*}

Finally, the Riesz network is built combining batch normalization, Riesz layer (Algorithm \ref{alg:RL}), and ReLU non-linearity.







\ 



\section*{Declarations}
 
\bmhead{Ethical Approval}
Not applicable.

\bmhead{Competing interests} 
The authors have no competing interests to declare that are relevant to the content of this article.
 
\bmhead{Authors' contributions} 
Conceptualization: T.B.; Methodology: T.B.; Formal analysis and investigation:  T.B.; Writing - original draft preparation: T.B.; Writing - review and editing: C.R., K.S.; Funding acquisition: C.R., K.S.; Supervision: C.R., K.S.
 
\bmhead{Funding} 
This work was supported by the German Federal Ministry of Education and Research (BMBF) [grant number 05M2020 (DAnoBi)].
 
\bmhead{Availability of data and materials} 
The datasets generated during and/or analysed during the current study are available from the corresponding author on reasonable request. 
\bibliographystyle{elsarticle-num}
\bibliography{ref}

\end{document}